\documentclass[11pt]{article}

\usepackage{helvet}

\usepackage[top=1in, bottom=1in, left=1in, right=1in]{geometry}
\usepackage{fancyhdr}
\usepackage{amssymb,amsbsy,amsfonts,amsmath,amsthm,mathtools,xspace}
\usepackage{bm}
\usepackage{mathrsfs}

\usepackage{sectsty}
\sectionfont{\fontsize{11}{13}\selectfont}
\subsectionfont{\fontsize{11}{13}\selectfont}
\subsubsectionfont{\fontsize{11}{13}\selectfont}
\usepackage{titlesec}
\titlespacing*{\section}{0pt}{10pt}{5pt}
\titlespacing*{\subsection}{0pt}{6pt}{3pt}

\usepackage{booktabs}
\usepackage{multirow}
\usepackage{makecell}
\usepackage{graphicx}
\graphicspath{{figures/}}
\usepackage{caption}
\usepackage{subcaption}
\usepackage{float}

\usepackage{natbib}
\usepackage{setspace}

\usepackage{color}
\usepackage[colorlinks=true, urlcolor=blue, linkcolor=blue, citecolor=blue]{hyperref}
\usepackage{url}
\usepackage[toc,page]{appendix}
\usepackage{cleveref}

\theoremstyle{plain}
\newtheorem{theorem}{Theorem}
\newtheorem{lemma}{Lemma}
\newtheorem{proposition}{Proposition}
\newtheorem{corollary}{Corollary}
\theoremstyle{definition}

\newtheorem{assumption}{Assumption}
\theoremstyle{remark}
\newtheorem{remark}{Remark}
\theoremstyle{plain}
\newtheorem*{propositionprime}{Proposition~\ref{prop:lb-var}$'$}

\crefname{assumption}{Assumption}{Assumptions}
\Crefname{assumption}{Assumption}{Assumptions}

\DeclareMathOperator{\E}{\mathbb{E}}
\DeclareMathOperator{\Var}{Var}
\DeclareMathOperator{\Cov}{Cov}

\DeclareMathOperator*{\argmax}{arg\,max}
\newcommand{\R}{\mathbb{R}}
\newcommand{\N}{\mathcal{N}}
\newcommand{\given}{\,\vert\,}

\newcommand{\one}{\mathbf{1}}
\newcommand{\Aset}{\mathcal{A}}
\newcommand{\Bset}{\mathcal{B}}
\newcommand{\Naive}{\mathrm{Naive}}
\newcommand{\GLMM}{\mathrm{GLMM}}
\newcommand{\Neff}{N_{\mathrm{eff}}}

\newcommand{\Dhat}{\widehat{\Delta}}
\newcommand{\phat}{\widehat{p}}

\RequirePackage[normalem]{ulem}
\RequirePackage{xcolor}

\crefname{appendix}{Supplement}{Supplements}
\Crefname{appendix}{Supplement}{Supplements}

\begin{document}

\begin{center}
\textbf{\large Trustworthy Method Comparison with AI Judges: Estimation and Design under Order, Batch, and Aggregation Effects
}\\
\bigskip
Tianxi Li \quad Jie Ding\\
\medskip
School of Statistics, University of Minnesota\\
\end{center}

\begin{abstract}
\noindent
Large language models (LLMs) are increasingly used as judges for automated AI evaluation. A common practice is to randomize prompt sequences and average the resulting scores, but its statistical validity remains unclear. We show that LLM evaluation mechanisms can be approximated by a class of Markov generalized linear mixed models (GLMMs), supported by out-of-sample predictions across three major commercial LLMs. Using a first-order Markov GLMM, we study leaderboard ranking and group comparison. For leaderboard ranking, randomize-and-average selection is consistent under a mild separation condition, and a Williams square design can improve efficiency when item qualities are close. For group comparison, naive averaging can yield inconsistent conclusions about differences in group-level quality because of the response model’s nonlinearity. Empirical results further support the validity of the proposed model-based inference beyond the first-order theory, including settings with higher-order sequence memory. We illustrate the approach in an application where AI judges compare two graphical model estimation methods.
\end{abstract}

\section{Introduction}
\label{sec:intro}

Generative large language models (LLMs) are increasingly used as AI judges,
scoring and ranking candidate responses, model versions, and even grant proposals at a
scale that human raters cannot match. For example, a traditional grant review panel reads a slate of
proposals in one review session and makes a sponsoring decision for each. An AI judge can be placed as a screening gate for such a panel by identifying promising proposals before human panels spend effort. For another example, comparing two methods using a benchmark dataset is a common task in scientific publications. When the task involves interpretation and domain knowledge, such as variable selection in genetic studies or graph construction by graphical models, an AI judge with broader background knowledge can take the output of the two methods and help with evaluating whether one gives more scientifically meaningful results. In these examples, a number of items to be evaluated go to the judge together in one shared context, and a decision or ``meaningful'' label is assigned by the AI judge. 

In many of these cases, the items for evaluation are provided together in a batch. This is not only how tasks are commonly posed, since a judgment about one item is made against the standard the rest of the slate sets, but also the token-economical choice to reduce repetitive prompts. In such settings, a judge's ratings may not be independent: what the judge says about the fifth item depends on what it has just
said about the fourth. That dependence is the source of complications for much of the related research. In particular, whether the resulting evaluation can be trusted raises two critical questions.

The first is whether the judge is measuring the right thing. Does the criterion it
applies track the quality a human evaluator ultimately cares about? This is a question of \emph{validity}, and it is where most of the field's attention has gone. MT-Bench and Chatbot Arena calibrate judge ratings against human preference
data \citep{zheng2023judging,chiang2024chatbot}; AlpacaEval's length-controlled
revision corrects a specific, documented misalignment between the judge's criterion and
the human target \citep{dubois2024lengthcontrolled}; and recent surveys and audits
catalog the ways a judge's criterion can drift from the intended one
\citep{gu2024survey,ye2024justice}.

The second question, which is equally important, is whether we can recover what the judge actually concluded. Given
a criterion, aligned or not, the judge is typically queried a finite number of times under a common protocol for stability, and its answers are aggregated into a score, a ranking, or a
conclusion about which of two groups is better. Does that aggregation return the judge's own
answer, or does the procedure introduce error of its own? This is a question of
\emph{measurement} rather than validity.  A perfectly aligned judge tells us nothing useful if the
aggregation step distorts its answer, and, as we show, the distortion can grow with the
evaluation budget, so that the more one spends on collecting more answers, the more confident
one could become in a conclusion the judge never reached. 

In this paper, we focus on solving the second problem (\emph{measurement}). We take the judge's criterion as
given. Every failure we study is therefore a failure of the evaluator's statistics rather than of the rubric it was judged by, and every remedy is a change of estimator or of experimental design rather than of prompt.

The dominant evaluation pipeline used in practice has two steps. First, items are presented to the judge in randomized order to mitigate the well-documented position bias of LLM judges. Such bias has been observed in pairwise judging \citep{zheng2023judging,wang2024fair}, list-wise ranking \citep{shi2024judging}, and multiple-choice option order \citep{pezeshkpour2024order}, and is reviewed together with other judge biases in recent surveys and audits \citep{gu2024survey,ye2024justice}. Proposed remedies include reversing presentation order and averaging scores \citep{wang2024fair}, splitting and re-merging candidate answers \citep{li2024splitmerge}, bootstrapping over reshuffled presentations \citep{hou2024zeroshot}, symmetrizing pairwise comparisons \citep{liusie2024comparative}, and multi-agent debate \citep{khan2024debating}. These remedies are largely heuristic, format-specific, and often computationally costly. Second, each item is scored by its marginal average rating, and these averages are compared or ranked. Variants of this ``random shuffle followed by marginal averaging'' pipeline underlie many contemporary judge-based benchmarks, including MT-Bench, which averages absolute grades \citep{zheng2023judging}, and AlpacaEval, which averages win rates against a fixed baseline \citep{dubois2023alpacafarm,dubois2024lengthcontrolled}. A notable exception is Chatbot Arena, which fits a logit-scale Bradley--Terry model to pairwise votes \citep{chiang2024chatbot} and is therefore insulated from the bias studied below. It is the second step, the \emph{estimator}, that the order-effect literature largely leaves untouched. Yet aggregation itself can introduce systematic error: a biased, order-dependent rating process can distort marginal averages in ways that randomization cannot remove. As we show, this is where the statistical failure arises.

A statistical model is needed to assess the validity of the ``randomize and average'' procedure and to develop improved alternatives. We show that an LLM judge scoring a shuffled batch can be well approximated by a \emph{Markov generalized linear mixed model} (GLMM), in which each item's binary rating follows a Bernoulli distribution on the logit scale, with a serial effect, a cold-start effect, and a batch-level random intercept. We view this as a \emph{working model class}, rather than a representation of the judge's internal mechanism. Its adequacy can be assessed empirically through out-of-sample predictive evaluation, as demonstrated in our experiments.

Under the order-1 Markov GLMM, we study two common evaluation tasks. For \emph{leaderboard ranking}, which seeks to identify the best item in a pool, we show that the naive average is consistent under a mild separation condition. A Williams Latin square design, combined with the same estimator, can further improve selection performance at no additional computational or token cost. In contrast, for \emph{group comparison}, the A/B test of whether one collection of items outperforms another, the naive procedure can fail. We show that it is \emph{structurally biased}, leading to invalid point estimation and hypothesis testing; increasing the evaluation budget does not remove this bias. A unified alternative is to fit the GLMM under an ordinary random shuffle and conduct standard model-based inference. This procedure is valid for both tasks. Although the theory is developed under the order-1 model, we further show that an order-1 approximation can still provide reliable inference when the judge exhibits higher-order memory.

The rest of the paper is organized as follows. \Cref{sec:physics} introduces the Markov-GLMM working model for an AI judge and establishes its empirical adequacy. \Cref{sec:leaderboard} analyzes leaderboard ranking and the properties of the Williams design, while \Cref{sec:convexity} analyzes group comparison and the nonlinearity bias of naive averaging. \Cref{sec:boundary} considers judges with memory beyond first order, a setting outside our design theory. \Cref{sec:data-example} applies the group-comparison method to a real problem, using AI judges to compare two graphical model estimation methods based on the scientific meaningfulness of their discovered associations. \Cref{sec:discussion} concludes with further discussion.
\section{A Working Statistical Model of AI Judges}
\label{sec:physics}

We first examine empirical patterns in AI-judge evaluations and use them to motivate the proposed working model. Our data come from three LLM judges from different providers: Gemini 2.5 Flash, GPT-5.4-mini, and Claude Haiku 4.5. Each judge scored the same $30$ research proposals in $200$ \emph{batches}: each batch presented all $30$ proposals in one shared context, in a randomly permuted order, and the judge assigned a binary accept/reject rating to each proposal sequentially. This yielded $6{,}000$ binary ratings per judge. We held the batch orderings \emph{identical} across the three judges.

Gemini 2.5 Flash provides a representative example. We first examine whether the judge's rating depends on the preceding rating in the same batch. \Cref{fig:context-bias} shows that, pooled over positions $2$--$30$, an item presented immediately after a rejection is accepted $73.0\%$ of the time, compared with $66.6\%$ after an acceptance, a difference of $-6.4\%$. Using $2{,}000$ bootstrap resamples of the batches, the corresponding $95\%$ interval is $[-8.9\%,-3.9\%]$, and all resamples yield a negative gap. The pattern also persists across positions in the batch (\Cref{fig:cb-position}). This observation suggests, though does not prove, a negative serial effect.

The opening slot of a batch has no predecessor, so it is excluded from the preceding comparison. \Cref{fig:context-bias} also shows that the acceptance rate at the first position can be lower than at positions following a rejection.

\begin{figure}[!htbp]
\centering
\begin{subfigure}[b]{0.40\textwidth}
  \includegraphics[width=\linewidth]{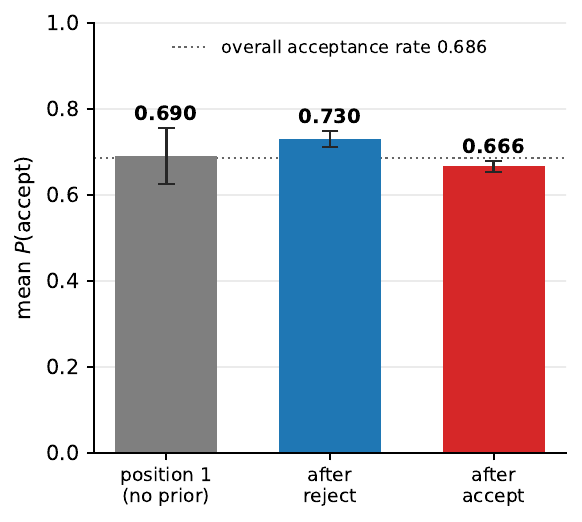}
  \caption{Pooled over positions $2$--$30$.}
  \label{fig:cb-bars}
\end{subfigure}\hfill
\begin{subfigure}[b]{0.56\textwidth}
  \includegraphics[width=\linewidth]{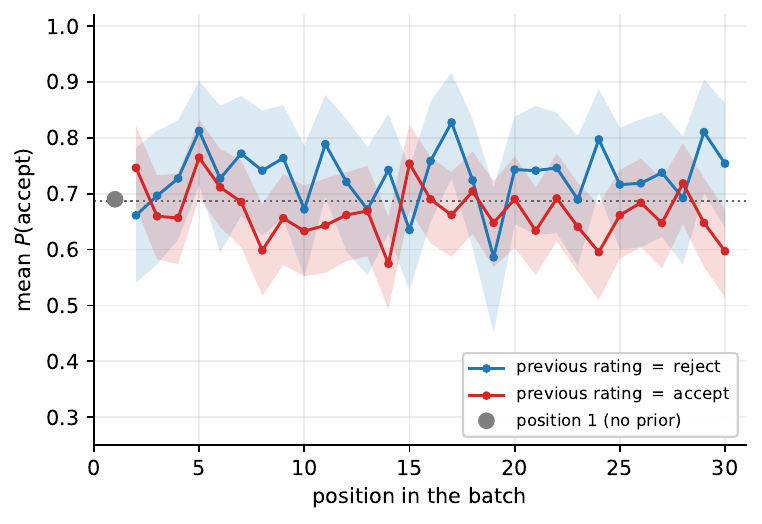}
  \caption{The same split, position by position.}
  \label{fig:cb-position}
\end{subfigure}
\caption{Dependence on the immediately preceding rating in the illustrative judge
(Gemini 2.5 Flash), visible in the raw data before any modeling. }
\label{fig:context-bias}
\end{figure}

Although these results are shown for one judge, qualitatively similar patterns appear for the other two. They motivate a working model in the next subsection.

\subsection{A working model for sequential AI judges}
\label{sec:working-model}

We do \emph{not} seek to recover the internal generative mechanism of an LLM. Instead, we propose a \emph{working model class}, namely a family of Markov generalized linear mixed models. The goal is to approximate the judge's rating process at the scale relevant for evaluation, providing a basis for estimation, design, and inference. Its adequacy is assessed empirically in \Cref{sec:adequacy}.

We consider a fixed pool of $V$ items, such as grant proposals, candidate model responses, or essays, evaluated by an LLM judge. The items are presented sequentially within each \emph{batch}, and the judge assigns a binary rating (accept/reject, valid/invalid, etc.) to each item. This protocol is repeated over $N$ batches. Let $Y_{tj}\in\{0,1\}$ denote the rating at within-batch position $j\in\{1,\dots,V\}$ of batch $t\in\{1,\dots,N\}$, let $V_{tj}\in\{1,\dots,V\}$ denote the item occupying that slot, and let $\theta_v\in\R$ denote the \emph{intrinsic quality} of item $v$.

The \emph{batched} protocol is common in practice. It resembles expert-panel evaluation, where items are assessed within a shared slate, and is also token-efficient for LLM evaluation because instructions and rubrics are shared across items. The resulting context sharing induces the dependence of interest: within a batch, each rating may depend on preceding ratings. The previous subsection identified three features: dependence on previous ratings, a distinct first-position effect, and batch-level heterogeneity.

A natural statistical tool to capture these features is a mixed-effects model with Markov dependence. In particular, we encode the patterns through a contextual shift $C_{tj}$ added to quality on the logit scale, and model the rating at slot $(t,j)$, for item $v=V_{tj}$, as conditionally Bernoulli:
\begin{equation}
  Y_{tj}\given(\theta_v,C_{tj}) \;\sim\;
  \mathrm{Bernoulli}\!\big(\sigma(\theta_v + C_{tj})\big),
  \qquad \sigma(x)=(1+e^{-x})^{-1}.
  \label{eq:bern}
\end{equation}
The \emph{model class} is indexed by memory depth $K\ge 0$, with
\begin{equation}
  C_{tj} \;=\; \sum_{k=1}^{K}\beta_k\,Y_{t,j-k} \;+\; \gamma\,\one_{\{j=1\}} \;+\; u_t,
  \qquad u_t \stackrel{\text{i.i.d.}}{\sim}\N(0,\sigma_b^2),
  \label{eq:shift-K}
\end{equation}
with boundary convention $Y_{t,j}:=0$ for $j\le 0$. Here $\beta_k$ is the serial contrast at lag $k$, $\gamma$ is the cold-start effect on the first item of a batch, and $u_t$ is a batch-level random intercept capturing variation in judge stringency across batches. The model includes as special cases no serial contrast ($\beta_1=\dots=\beta_K=0$), no cold-start effect ($\gamma=0$), and no batch-level variation ($\sigma_b^2=0$). The depth $K$ is not fixed a priori. We select it separately for each judge by standard model-selection criteria; \Cref{sec:fit-four} selects $K=1$ for Gemini 2.5 Flash and $K=3$ for GPT-5.4-mini and Claude Haiku 4.5.

The key structural feature, shared by every member of the class, is that $C_{tj}$ depends only on information preceding slot $(t,j)$: it is a function of previous ratings, the position indicator, and the batch effect, but not of the item identity $V_{tj}$. Combined with a randomized, item-exchangeable allocation (\Cref{ass:design}), this yields a single \emph{generic} contextual shift $C\sim F_C$, defined as the position-averaged distribution of $C_{tj}$ over a batch, unconditional on the item occupying the slot. We then define the \emph{marginal win probability}
\begin{equation}
  \pi(\theta) \;:=\; \E_C\big[\sigma(\theta + C)\big]
  \;=\; \int_{\R}\sigma(\theta+c)\,dF_C(c).
  \label{eq:pi}
\end{equation}
The map $\theta\mapsto\pi(\theta)$ is convex when $\theta+C<0$ almost surely and
concave when $\theta+C>0$ almost surely. 

\begin{assumption}[Exchangeable allocation in independent randomization units]
\label{ass:design}
The allocation $\{V_{tj}\}$ is exchangeable across items, meaning that for every permutation $\varsigma$ of $\{1,\dots,V\}$, the law of $\{V_{tj}\}$ equals that of $\{\varsigma(V_{tj})\}$, and the allocation is independent of $\{u_t\}$ and the ratings' conditional draws. Every batch presents every item exactly once, so $N$ is both the number of batches and the per-item evaluation budget. Moreover, the allocation is generated in \emph{independent randomization units}: either \textnormal{(a)} the $N$ within-batch orders are drawn i.i.d.\ across batches, or \textnormal{(b)} the batches form $N/B$ i.i.d.\ blocks of $B$ consecutive batches, with the within-block orders given by a fixed array of $B$ sequences composed with a single item-to-symbol bijection drawn uniformly and independently for each block.
\end{assumption}

Two allocations satisfy this assumption and are compared in the paper: the \emph{complete random shuffle}, corresponding to (a) with uniformly random orders, and the balanced Williams design of \Cref{sec:leaderboard}, corresponding to (b) with $B=V$ and an independently redrawn item-to-symbol assignment for each block. The independent-unit condition ensures the applicability of the law of large numbers, central limit theorem, and concentration bounds used below, with the batch as the unit under (a) and the block under (b). Exchangeability alone is insufficient, since it permits, for example, a single random order to be reused across all batches, yielding random limits for the empirical win rates. Exchangeability allows both designs to be treated within a common framework and ensures that the reference law $F_C$, and hence the map $\pi(\theta)$ in \eqref{eq:pi}, is the same under either design.

The model \eqref{eq:shift-K} is a \emph{transition} (Markov) model of order $K$. Conditional on $u_t$, the chain rule factorizes the batch joint probability according to the previous $K$ ratings, so the ``lagged response as covariate'' objective maximized by a mixed-model routine is the exact likelihood. Thus, \eqref{eq:bern}--\eqref{eq:shift-K} can be fitted by maximum likelihood using standard mixed-model methods, with likelihood-based inference available directly. Under mild regularity conditions, the GLMM MLE is consistent and asymptotically normal as the number of independent randomization units grows, following standard GLMM inference theory (see Supplement~\ref{supp:glmm-theory}).

One practical caveat is \emph{separation}. For example, if an item is accepted or rejected on nearly all of its appearances, the corresponding $\theta_v$ may have no finite maximum likelihood estimate. A common remedy is to use a weakly informative prior on the fixed effects, such as a Gaussian prior, equivalent to a ridge penalty \citep{chung2013blme}, or a heavier-tailed Cauchy prior \citep{gelman2008weakly}. When penalization is used, inference can be based on the bootstrap. Under the random-shuffle design, the experimental unit is the \emph{batch}, and the bootstrap resamples whole batches with replacement.

\subsection{Fitted models for three AI judges}
\label{sec:fit-four}

\Cref{fig:forest} reports the bootstrap estimates of the memory kernel for all three judges from the proposal evaluation data, using ridge-penalized fitting, with $V=30$ and $N=200$. The lag-1 contrast is negative and supported by the bootstrap for all three judges, indicating that an acceptance lowers the probability of a subsequent acceptance. Memory depth nonetheless varies. Gemini's signal is confined to lag-1; GPT-5.4-mini shows an additional supported lag-3 effect and a borderline cold-start effect; and Claude shows significant negative contrasts at all three lags and a significant cold-start effect. The batch random effect is substantial in all three, with $\sigma_b^2$ equal to $0.446$, $1.173$, and $0.418$, and the overall acceptance rate decreases across the three judges, from $0.686$ to $0.675$ and $0.567$.

We emphasize that this shared structure is specific to these three judges and this task. For other tasks and judges, an appropriate specification should be selected within the proposed model class, or a suitable extension, rather than adopting the fitted structure found here.

\begin{figure}[H]
\centering
\includegraphics[width=0.9\linewidth]{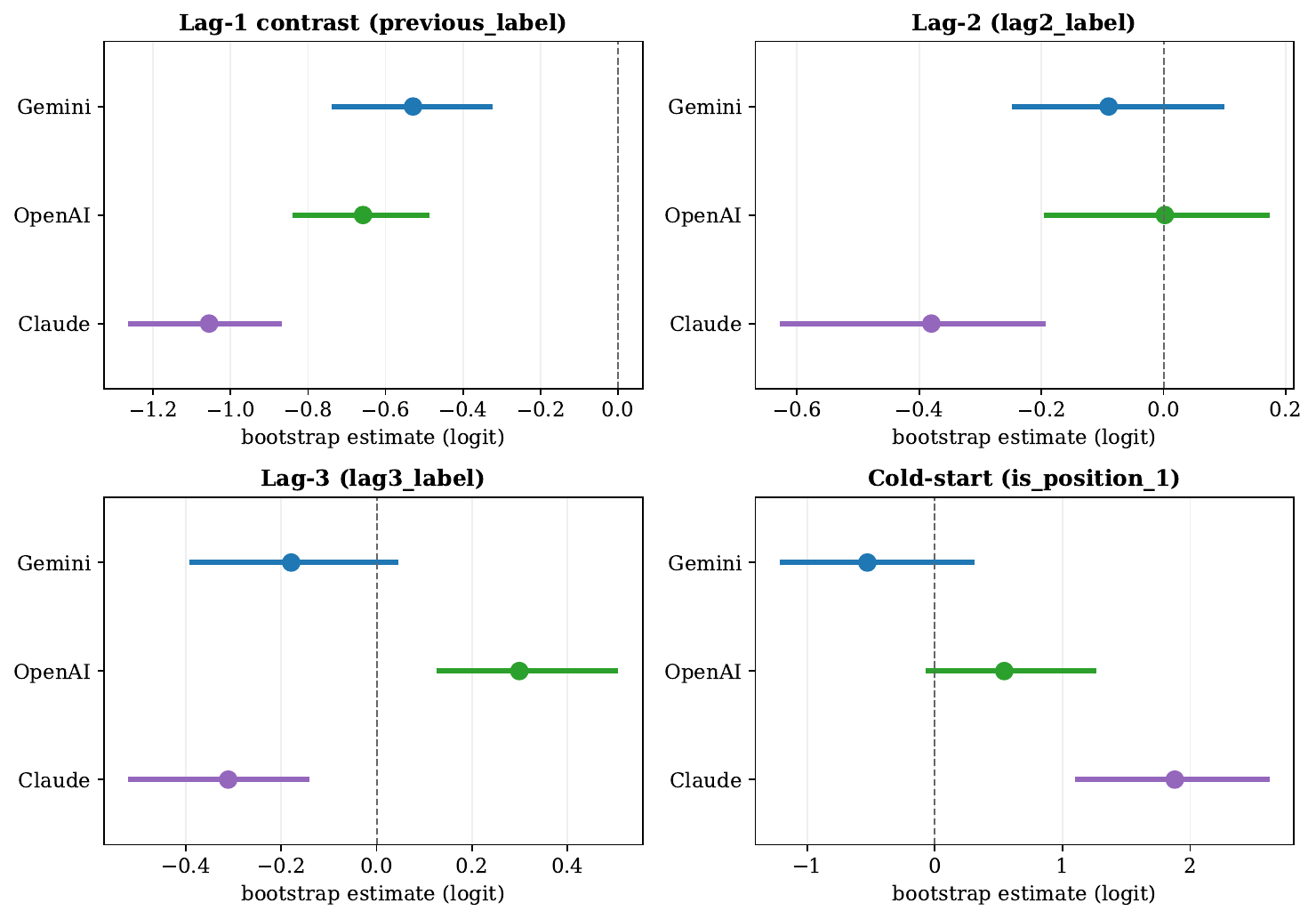}
\caption{Cross-model serial-contrast estimates of coefficients, with
the bootstrap $95\%$ intervals.}
\label{fig:forest}
\end{figure}
\subsection[Model adequacy through out-of-sample predictive checks]{Model adequacy by out-of-sample predictive evaluation}
\label{sec:adequacy}

Establishing that the GLMM is a reasonable summary is not sufficient for our study. \Cref{sec:leaderboard,sec:convexity} build estimators
and tests \emph{on top of} \eqref{eq:bern}--\eqref{eq:shift-K}, so if the working model were missing important
structure, and capturing the judge truly required a neural-network
surrogate, those guarantees would be vacuous. We therefore assess model adequacy through out-of-sample predictive checks of lagged transition probabilities, using newly collected batches that were not used for model fitting or selection.

We collected $100$ new batches using shuffled ordering. Meanwhile, we froze each judge's fitted GLMM and simulated $500$ datasets  from these fitted models under the same new orderings. For any quantity computed from the new test data, we can use the 500 simulated datasets to construct confidence intervals (by taking the corresponding percentiles).
The resulting $95\%$ prediction intervals cover $34$ of $36$ (94.4\%) lagged-transition probabilities (\Cref{fig:ppc}); the two misses are GPT-5.4-mini's lag-2 and lag-3 $0\!\to\!1$ cells. Also, the predictive intervals are narrow, relative to the magnitude of the true values.
The approximately correct and informative coverage supports the working model. Note that the intervals condition on fitted parameters, so this is a predictive diagnostic rather than a formal goodness-of-fit test. We also calculate the predictive CIs for three conditional-entropy values (up to order 3) on the new data and all of these are correctly covered. More details about this experiment can be found in Supplement~\ref{supp:predictive-details}.

\begin{figure}[!htbp]
\centering
\begin{subfigure}[t]{0.32\textwidth}
  \includegraphics[width=\linewidth]{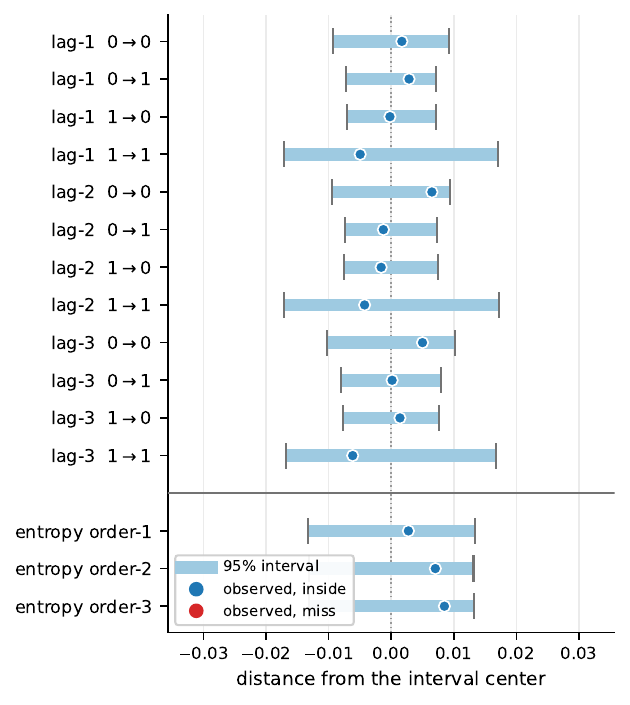}
  \caption{Gemini 2.5 Flash.}
  \label{fig:ppc-gemini}
\end{subfigure}\hfill
\begin{subfigure}[t]{0.32\textwidth}
  \includegraphics[width=\linewidth]{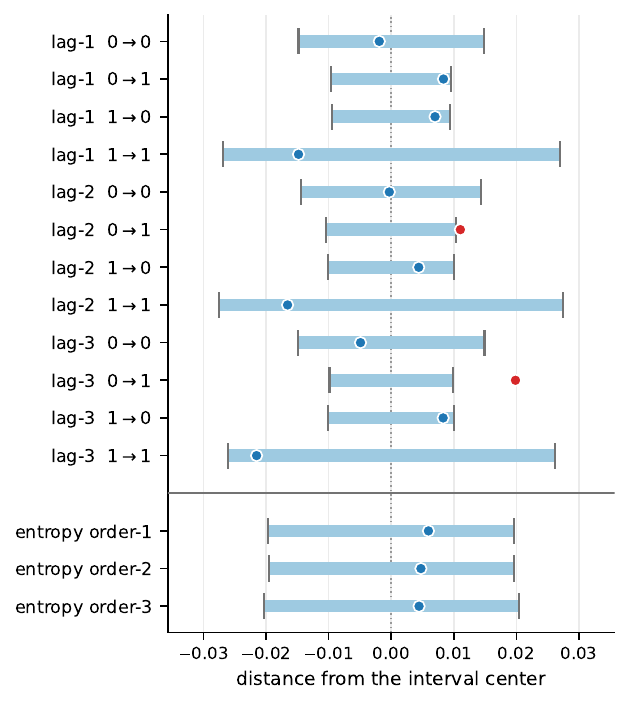}
  \caption{GPT-5.4-mini.}
  \label{fig:ppc-openai}
\end{subfigure}\hfill
\begin{subfigure}[t]{0.32\textwidth}
  \includegraphics[width=\linewidth]{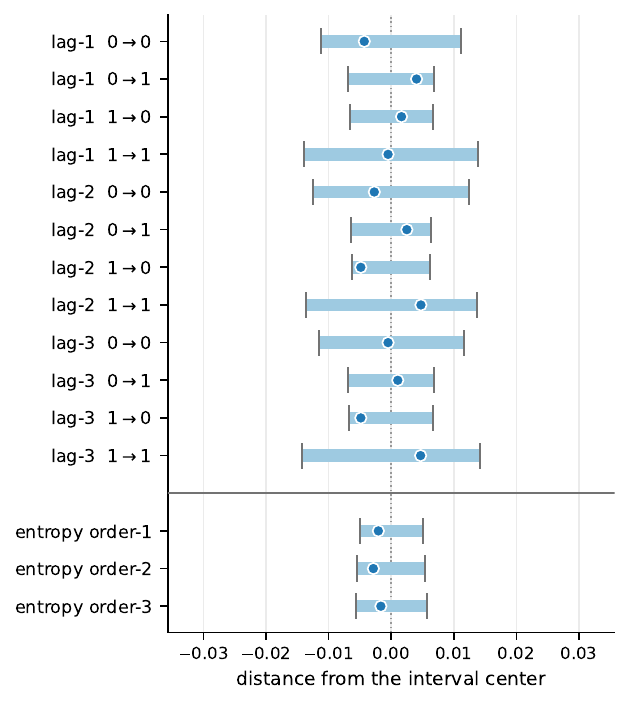}
  \caption{Claude Haiku 4.5.}
  \label{fig:ppc-claude}
\end{subfigure}
\caption{Predictive performance check. Each panel has two
blocks. \emph{Above the rule}: the twelve transition terms. \emph{Below the rule}: the three conditional entropies
$H(Y_t\mid Y_{t-1},\dots,Y_{t-k})$, $k=1,2,3$. For every term we compare the value observed in the $100$ new
batches against the $95\%$ prediction intervals. The intervals are shifted to the common center for visualization purposes.}
\label{fig:ppc}
\end{figure}

In summary, we have shown that the proposed Markov GLMM family can serve as a reasonable approximation of the AI judging mechanism under the batch protocol. In the next two sections, we will use this model to study the ``randomize and average'' evaluation strategy for a deep understanding of its properties.

\section{Task 1: Leaderboard Ranking}
\label{sec:leaderboard}

The design theory of this section and of \Cref{sec:convexity} is developed for the lag-1
working model of \Cref{sec:physics}, the first-order member of the class
\eqref{eq:shift-K}. 
\begin{assumption}
\label{ass:working}
The judge satisfies \eqref{eq:bern} with the $K=1$ member of \eqref{eq:shift-K},
\begin{equation}
  C_{tj} \;=\; \beta\,Y_{t,j-1} \;+\; \gamma\,\one_{\{j=1\}} \;+\; u_t,
  \qquad u_t \stackrel{\text{i.i.d.}}{\sim}\N(0,\sigma_b^2),
  \label{eq:shift}
\end{equation}
where we write $\beta:=\beta_1$, with boundary convention $Y_{t,0}:=0$. Define $\Lambda_1=(1-\tanh(|\beta|/4))^{-2}$. The margin $|\beta|\Lambda_1/V$ in \Cref{prop:lb-rank} below allows for the order-$1/V$ dependence of each item's win rate on the rest of the pool.
\end{assumption}


The first canonical task is the leader selection: from all the items, we want to identify the best one. As the marginal win probability
$\pi$ of \eqref{eq:pi} is strictly \emph{monotone} in quality, the naive win-rate
ranking is expected to be order-consistent, as we will rigorously verify. We will also show that we can further improve the performance via a balanced design. In particular, we will show that a Williams square design gives a variance reduction.

Let the $V$ items have distinct qualities $\theta_1>\theta_2>\cdots>\theta_V$, so
item $1$ is the unique optimum and $\delta:=\theta_1-\theta_2>0$ is the minimal
quality gap. Recall the naive marginal
estimator is $\phat_v=N^{-1}\sum_t Y_{t,P(v,t)}$, where $P(v,t)$ is the position of item $v$ in batch $t$, and the corresponding selection rule is $\hat v=\argmax_v\phat_v$, with ties broken by any fixed rule. 

The performance metric we care about is the \emph{probability of correct selection} $\Pr_{\mathcal D}(\hat v=1)$,
where the subscript records the design $\mathcal D$ under which input orders were
generated. We consider two designs under Assumption~\ref{ass:design} to give the
presentation orders in the batches:
\begin{itemize}
\item \emph{Random Shuffle} $\mathcal D_R$, in which the $N$
within-batch orders are drawn i.i.d.\ uniformly;
\item \emph{Williams Latin square} $\mathcal D_W$, a balanced design that is
simultaneously \emph{position balanced}, so that each item appears equally often in each of
the $V$ positions, and \emph{first-order carryover balanced}, so that each ordered pair
$(a,b)$ appears as consecutive items equally often \citep{williams1949}. Such a square exists for every even $V$, which we assume throughout for this design. Both forms of balance are properties of a complete square of $V$ batches, so we take the budget $N$ to
be a multiple of $V$; the design is then $N/V$ stacked squares, and the item-to-symbol
assignment is redrawn independently for each of them, so every block is exactly balanced.
We call the result the \emph{symbol-randomized} Williams design: the square itself is
deterministic, and the only randomness is the uniform assignment of items to its
symbols, redrawn per block. More details can be found in
\citet{jones2015crossover}.
\end{itemize}
Intuition based on classical experimental design theory suggests the advantage of the Williams Latin square design: it
makes every item's measurement conditions identical, since each item spends the same
number of batches in each position and follows each other item equally often. Two features, nevertheless, separate our setting from the classical one. First, instead of a standard mixed effect model, our current model is Markov with a serial effect on the logit scale. Second, the naive estimator is not an unbiased estimator of the true strength
parameter $\theta_v$; meanwhile, our primary goal is to recover the top item, instead of
estimating the individual parameters. The classical experiment design results are
therefore not directly applicable. However, as we show next, the principled idea of
balanced design still works.

\subsection{Theory for the naive estimator in leaderboard selection}
\label{sec:lb-moments}

We will start with an initial result showing that the naive ``randomize-and-average'' method works for this task. Technical details are included in Supplement~\ref{supp:ranking-details}.

\begin{proposition}
\label{prop:lb-rank}
Under \Cref{ass:working,ass:design}, if the gap between the two top items is sufficiently large such that
$\pi(\theta_1)-\pi(\theta_2)>|\beta|\Lambda_1/V$, naive selection is consistent under either design,
\[
  \Pr_{\mathcal D}(\hat v=1)\;\longrightarrow\;1,
  \qquad \mathcal D\in\{\mathcal D_R,\mathcal D_W\}.
\]
Moreover, $\E[\phat_v(\mathcal D_W)]=\E[\phat_v(\mathcal D_R)]$ for every $v$ and every $N$.
\end{proposition}

Consistency in $N$ holds for any pool whose two
leaders are separable. So the naive ``randomize-and-average'' recipe under both designs is consistent for the task.
Based on the above result, any advantage of one design over the other must therefore live in
the \emph{variance} and that is what we will study next. Writing $\mathbf Z(\mathcal D)$ for the vector
of contrasts $\phat_1-\phat_k$, $k=2,\dots,V$, and
$\bm\mu(\mathcal D):=\E[\mathbf Z(\mathcal D)\mid\mathcal D]$ for its mean once the design
is drawn, the law of total variance decomposes the covariance under \emph{either}
design into the same two pieces,
\begin{equation}
  \Cov(\mathbf Z(\mathcal D))
  \;=\;\underbrace{\E_{\mathcal D}\!\big[\Cov(\mathbf Z\mid\mathcal D)\big]}_{\text{sampling}}
  \;+\;\underbrace{\Cov_{\mathcal D}\!\big(\bm\mu(\mathcal D)\big)}_{\text{design}}.
  \label{eq:lb-ltv}
\end{equation}
The first term is sampling variation with the design fixed; the second is variation from randomized presentation orders. Let $M_v(\pi)$ be the expected rating of item $v$ in order $\pi$, and $\bar m_v(\tau)$ its average expected rating over a Williams block with symbol assignment $\tau$. Define $\nu_{\mathrm R,k}=\Var_{\pi}[M_1(\pi)-M_k(\pi)]$ and $\nu_{\mathrm W,k}=\Var_{\tau}[\bar m_1(\tau)-\bar m_k(\tau)]$, with uniform random orders or assignments. The corresponding design variances are $\nu_{\mathrm R,k}/N$ and $V\nu_{\mathrm W,k}/N$. 
In addition, we write $\bm\Delta_{\mathrm R}$ and $\bm\Delta_{\mathrm W}$ for the second
term, the \emph{design term}, under the random shuffle and under the
symbol-randomized Williams design, respectively. What is being averaged over
differs: for the shuffle the design draw is the $N$ independent orderings, for
Williams it is the per-block item-to-symbol bijections.

For the design variance to be worth removing, it must be present in the first place. Equivalently, we need expected ratings to vary across slots.

\begin{assumption}[Non-degenerate design variance]
\label{ass:nondegen}
$(\beta,\gamma)\ne(0,0)$, and every top contrast carries positive per-batch design
variance: $\nu_{\mathrm R,k}>0$ for every $k=2,\dots,V$.
\end{assumption}

\begin{proposition}[Finite-sample variance dominance]
\label{prop:lb-var}
Fix the item qualities $\bm\theta$ and let \Cref{ass:working,ass:design} hold.
Then for every budget $N$ (a multiple of $V$),
\[
  N\Big([\Cov(\mathbf Z(\mathcal D_R))]_{k-1,k-1}
  -[\Cov(\mathbf Z(\mathcal D_W))]_{k-1,k-1}\Big)
  \;=\;\nu_{\mathrm R,k}\;-\;V\,\nu_{\mathrm W,k},
  \qquad k=2,\dots,V,
\]
and the right-hand side does not depend on $N$. If all qualities are equal then
$\nu_{\mathrm W,k}=0$ exactly, so the right-hand side is $\nu_{\mathrm R,k}$,
positive under \Cref{ass:nondegen}. Both sides are continuous in $\bm\theta$; hence for every
common value $\theta_*$ at which \Cref{ass:nondegen} holds for the flat
configuration $\bm\theta=\theta_*\bm 1$, there is a $\delta_0>0$ such that whenever
$\max_v|\theta_v-\theta_*|\le\delta_0$, every contrast has strictly smaller
variance under the Williams design at \emph{every} finite $N$.
\end{proposition}

To compare selection accuracy when leading items are close, consider $\theta_v=\theta_*+t_v/\sqrt N$. Write $F_v(\bm\theta)=\E[\phat_v]$. At $\bm\theta=\theta_*\bm 1$, let $g_{\mathrm{self}}=\partial F_v/\partial\theta_v$ and $g_{\mathrm{cross}}=\partial F_v/\partial\theta_w$ for $w\ne v$; exchangeability makes these derivatives independent of the chosen indices. 

\begin{theorem}[Mean separation and selection dominance under local alternatives]
\label{thm:lb-pcs}
Consider the local-alternative regime $\theta_v=\theta_*+t_v/\sqrt N$ with bounded
$t_v$ and fixed contrasts $t_1-t_k=d_k>0$, $k=2,\dots,V$, where $V\ge4$ is even, and let
\Cref{ass:working,ass:design} hold.
\begin{enumerate}
\item[\textnormal{(i)}] Under both designs,
\[
  \sqrt N\,\E\big[\phat_1-\phat_k\big]\;\longrightarrow\;
  \big(g_{\mathrm{self}}(\theta_*)-g_{\mathrm{cross}}(\theta_*)\big)\,d_k,
  \qquad k=2,\dots,V,
\]
the background level $\theta_*$ and the bulk coefficients $t_v$, $v\notin\{1,k\}$,
canceling exactly in every contrast.
\item[\textnormal{(ii)}] If, in addition, \Cref{ass:nondegen} holds \emph{at the
flat configuration} $\bm\theta=\theta_*\bm 1$, that is,
$$\nu^*_{\mathrm R,k}:=\Var_{\pi\sim\mathrm{Unif}(S_V)}
[M_1(\pi)-M_k(\pi)]\big|_{\bm\theta=\theta_*\bm 1}>0$$
 for every $k$, and the
self-effect is stronger than the cross-effect,
$g_{\mathrm{self}}(\theta_*)>g_{\mathrm{cross}}(\theta_*)$,
\begin{equation}
  \liminf_{N\to\infty}\big[\Pr_{\mathcal D_W}(\hat v=1)-\Pr_{\mathcal D_R}(\hat v=1)\big]>0 .
  \label{eq:lb-pcs}
\end{equation}
\end{enumerate}
\end{theorem}
The condition
$g_{\mathrm{self}}(\theta_*)>g_{\mathrm{cross}}(\theta_*)$ of part (ii) says that the evaluation of item $v$ depends more on its own quality than on that of another item $w\ne v$, which is intuitively reasonable. 

We want to give a summary of the theoretical study above. Overall, for this leaderboard task, the naive estimator remains valid under both designs. We can expect the Williams design to give some advantage, but generally speaking such a gain need not be substantial. However, given that the Williams design does not introduce extra cost in computation, such a gain can still be beneficial, as demonstrated next in our simulation examples.

\subsection{Simulation evidence}
\label{sec:lb-empirical}

We simulate $V=6$ items under two fitted judges: Gemini 2.5 Flash gives a nearly flat position profile (Setting 1), whereas Claude Haiku 4.5 gives an uneven profile (Setting 2). Each panel has one best item with gap $\delta=2.8/\sqrt N$ and otherwise equal qualities. We compare both estimators under both designs at $N=6,12,18$, using $10{,}000$ replicates per cell and uniformly randomized ties. Configuration details can be found in Supplement~\ref{supp:ranking-simulation}.

\Cref{tab:lb-secret} gives the result. In Setting 2 the Williams design is ahead at every budget, by
$1.2$, $2.0$ and $3.1$ percentage points. In Setting 1 it is not: the three gains are
$-0.5$, $-0.1$ and $-0.5$, scattered about zero. The selection probability itself barely changes down either column, as the local scaling intends. The GLMM model-based estimator is always a valid option, and the comparison between the naive estimator and the GLMM estimator depends on the setup of the problem. However, the Williams design does not help the GLMM estimator --- the model fitting has already fully leveraged the information.

The direction of the effect matches \Cref{prop:lb-var}: the naive estimator
gains under the uneven profile of Setting 2, while in Setting 1 the predicted
design-variance removal is far below what $10{,}000$ replicates can resolve
and no gain is detectable. In contrast, the design does little for the
model-based estimator, which already accounts for position.

The advantage is real but modest, about two percentage points of selection accuracy in the favorable configuration, consistent with removing a tenth of the variance of a contrast. Moreover, it is small for a structural reason:
the profile's departure from flatness is concentrated in the first slot, which is one slot out of $V$, so the design variance available to be removed shrinks as the item pool grows. Balancing is a tool for small pools evaluated on tight budgets, which is the
regime studied here.

\begin{table}[H]
\centering
\caption{Top-1 accuracy (\%) for all four combinations of estimator and design, $V=6$,
$10{,}000$ replicates per cell (Monte Carlo standard error $0.71$ on each gain). The item
panel is $\bm\theta=(\theta^\ast+\delta,\theta^\ast,\dots,\theta^\ast)$ with
$\delta=2.8/\sqrt N$, the local scaling of \Cref{thm:lb-pcs}, so the selection problem
holds constant difficulty as the budget grows. Setting 1 has a nearly flat
position profile, so the removable design variance is nearly zero; Setting 2 has a strongly
uneven profile and a substantial predicted gain, which is what the naive
columns show.}
\label{tab:lb-secret}
\begin{tabular}{llcccc}
\toprule
& & \multicolumn{2}{c}{Naive average (\%)} & \multicolumn{2}{c}{GLMM contrast (\%)} \\
\cmidrule(lr){3-4}\cmidrule(lr){5-6}
Setting & Budget & random & Williams & random & Williams \\
\midrule
\multirow{3}{*}{\makecell[l]{Setting 1 (Flat)}}
  & $N{=}6$  & $44.6$ & $44.1$ & $44.4$ & $44.5$ \\
  & $N{=}12$ & $47.8$ & $47.7$ & $47.8$ & $48.2$ \\
  & $N{=}18$ & $49.8$ & $49.3$ & $49.8$ & $49.5$ \\
\midrule
\multirow{3}{*}{\makecell[l]{Setting 2 (Uneven)}}
  & $N{=}6$  & $46.7$ & $47.9$ & $49.6$ & $49.0$ \\
  & $N{=}12$ & $48.0$ & $49.9$ & $51.3$ & $51.5$ \\
  & $N{=}18$ & $48.1$ & $51.2$ & $50.9$ & $52.3$ \\
\bottomrule
\end{tabular}
\end{table}

In summary, we find that ranking is the forgiving case. The naive estimator is already consistent under a mild separation condition, and a
classical balanced design can further sharpen it at no computational cost, though the magnitude of the improvement varies case by case. 
\section{Task 2: Group-Level Inference}
\label{sec:convexity}

The other canonical evaluation task is group-level comparison: given two
collections of items, such as two prompting strategies,
or two candidate model versions, decide whether the AI judge rates one group higher
than the other on average. This is the $A/B$ test of the evaluation world. The naive method records each
item's empirical win rate $\phat_v$ and aggregates these values to the group level for final comparisons, as used
by many contemporary judge-based benchmarks and leaderboards
\citep{zheng2023judging,dubois2023alpacafarm,dubois2024lengthcontrolled,chiang2024chatbot}. In this section, we show that, unlike the more forgiving leaderboard task, group-level inference is vulnerable to the nonlinearity of the problem, and that the naive method can lead to wrong conclusions about differences in mean latent item quality.

\subsection{The population-level bias}
\label{sec:conv-pop}

Let $\Aset$ and $\Bset$ partition the items into groups of sizes $V_A$ and $V_B$. Under \eqref{eq:bern}, define the group means, within-group variances, and third
central moments, for $G\in\{A,B\}$, by
\[
  \bar\theta_G \;:=\; \frac1{V_G}\sum_{v\in G}\theta_v,
  \qquad
  s_G^2 \;:=\; \frac1{V_G}\sum_{v\in G}\big(\theta_v-\bar\theta_G\big)^2,
  \qquad
  m_{3,G} \;:=\; \frac1{V_G}\sum_{v\in G}\big(\theta_v-\bar\theta_G\big)^3 .
\]
Following the latent-scale formulation of item response models \citep{deboeck2004explanatory}, we define group quality as the mean item effect $\bar\theta_G$ on the logit scale. Our target is therefore $\delta:=\bar\theta_A-\bar\theta_B$, the difference in average item quality after accounting for the modeled sequential and batch effects. As a linear contrast of the GLMM parameters, $\delta$ admits estimation and inference under \Cref{prop:glmm}. In contrast, the naive group contrast is
\begin{equation}
  \Dhat_{\Naive} \;=\; \frac{1}{V_A}\sum_{v\in\Aset}\phat_v
  \;-\; \frac{1}{V_B}\sum_{v\in\Bset}\phat_v .
  \label{eq:naive}
\end{equation}
Intuitively, the trouble is that $\phat_v$ estimates a \emph{nonlinear} transform of quality, namely $\pi(\theta_v)$ up to a correction of order $1/V$ (the last term in \eqref{eq:conv-bias}), and averaging a nonlinear transform is not the
same as transforming the average. Typically, the discrepancy is governed
by the curvature of $\pi$, and curvature does not cancel between two groups unless
they have the same spread. The following result makes this precise. Write
$\bar\mu:=\tfrac12(\bar\theta_A+\bar\theta_B)$ and
$\bar\rho:=\max_{v}|\theta_v-\bar\mu|$; a single third-order
expansion of $\pi$ about $\bar\mu$ then describes the naive limit.

\begin{theorem}
\label{thm:conv-bias}
Under \Cref{ass:working,ass:design}, with $V\ge4$, for every $\delta$ (zero or not), as $N\to\infty$,
\begin{equation}
\begin{aligned}
  \Dhat_{\Naive}
  \;\xrightarrow{p}\;& \pi'(\bar\mu)\,\delta
  \;+\; \tfrac12\,\pi''(\bar\mu)\,\big(s_A^2-s_B^2\big)
  \;+\; \tfrac16\,\pi'''(\bar\mu)\,\big(m_{3,A}-m_{3,B}\big)\\
  &\;+\; O\!\Big(\bar\rho^{\,4}\;+\;|\delta|\,\bar\rho^{\,2}
  \;+\;\frac{|\beta|}{V}\big(|\delta|+\bar\rho^{\,2}\big)\Big),
\end{aligned}
  \label{eq:conv-bias}
\end{equation}
where the implied constant depends only on $\beta$ and on the ratio of the two group sizes, and the slope satisfies $\pi'(\bar\mu)=\E_C[\sigma(\bar\mu+C)(1-\sigma(\bar\mu+C))]\in(0,\tfrac14]$.
\end{theorem}

Consider \eqref{eq:conv-bias} in the special case
\begin{equation}
  H_0^{\mathrm{m}}:\qquad
  \bar\theta_A=\bar\theta_B=\mu
  \qquad\text{and}\qquad
  m_{3,A}=m_{3,B},
  \label{eq:mtm-null}
\end{equation}
which we call the \emph{matched-third-moment null}. The second condition holds,
for instance, whenever the two within-group quality distributions are
symmetric. Under \eqref{eq:mtm-null}, $\delta=0$ and $\bar\mu=\mu$, so the
signal term, the third-moment term, and every $\delta$-part of the remainder
become zero, and the limit of the naive estimator
carries the curvature term
$\tfrac12\pi''(\mu)(s_A^2-s_B^2)$, nonzero whenever the
groups differ in spread ($s_A^2\ne s_B^2$) and $\pi$ has nonzero curvature at
$\mu$. The limit itself is nonzero whenever the curvature term
dominates the remainder
of \eqref{eq:conv-bias}. The bias is a pure artifact of the probability scale that the naive method relies on.
This is a classical phenomenon arising in the new context of AI evaluation. In longitudinal and
clustered-data models it is well understood that averaging a nonlinear link over a
heterogeneous population leads to biased estimators \citep{zeger1988models,neuhaus1991comparison}. The naive aggregation, which the field treats as an unbiased summary, therefore silently
encodes a difference in within-group spread as a difference in quality.

In general, the group quality difference signal enters the
naive estimator's limit as the compressed $\pi'(\bar\mu)\,\delta$, shrunk by a
factor of at least four and, in harsh regimes where $\sigma(\bar\mu+C)$ concentrates
near $0$ or $1$, by an order of magnitude. So the true difference $\delta$, the object we really care about in comparison, can be either masked or distorted in the naive estimator, due to the nonlinearity of the AI models.

\Cref{thm:conv-bias} indicates the issue of the naive comparison via point estimation. Resorting to hypothesis testing for uncertainty quantification does not repair it, and makes matters worse. Note that the bias $b:=\tfrac12\pi''(\mu)(s_A^2-s_B^2)$
is a \emph{constant} in the evaluation budget, whereas the sampling noise of
$\Dhat_{\Naive}$ shrinks at the usual rate. Write $\Delta_\infty$ for the
in-probability limit of $\Dhat_{\Naive}$ in \eqref{eq:conv-bias}. Under \eqref{eq:mtm-null}, $\Delta_\infty\ne0$ whenever the
curvature term dominates the remainder. This indicates that an increasing sampling budget will only lead to a confident mistake for a test based on the naive statistic.

To make this precise, we first fix the test an analyst would actually run.
Under the random-shuffle mechanism of \Cref{ass:design}, every batch judges each
item exactly once, so the naive contrast is the sample mean of $N$ i.i.d.\
per-batch contrasts:
\begin{equation}
  \Dhat_{\Naive}=\frac1N\sum_{t=1}^N c_t,\qquad
  c_t:=\frac1{V_A}\sum_{v\in\Aset}Y_{t,\cdot(v)}
  -\frac1{V_B}\sum_{v\in\Bset}Y_{t,\cdot(v)}\;\in\;[-1,1],
  \label{eq:batch-contrast}
\end{equation}
where $Y_{t,\cdot(v)}$ denotes item $v$'s rating in batch $t$. The contrasts
$c_t$ are i.i.d.\ across batches with $\E[c_t]=\Delta_\infty$ exactly. A studentized level-$\alpha$ test
can be constructed as follows. Estimate the batch-level variance
$\nu_{\mathrm b}:=\Var(c_1)$ by some consistent estimator $\hat\nu$, form
$T_N:=\sqrt N\,\Dhat_{\Naive}/\hat\nu^{1/2}$, and reject $H_0$ when
$|T_N|>z_{1-\alpha/2}$.

\begin{corollary}
\label{cor:conv-student}
Under the random-shuffle design and \Cref{ass:working}, assume \eqref{eq:mtm-null} with
$\Delta_\infty\ne0$ and $\nu_{\mathrm b}>0$. Then for any estimator
$\hat\nu\xrightarrow{p}\nu_{\mathrm b}$, $|T_N|\xrightarrow{p}\infty.$
\end{corollary}
So for every fixed level $\alpha\in(0,1)$ the test rejects the true null with probability
tending to one. The test procedure in this case cannot control the type I error rate.

\subsection{Simulation evidence}
\label{sec:conv-illusion}

We evaluate two settings at $V=30$, split $15$ versus $15$, over a grid
of batch counts $N$:
\begin{enumerate}
\item Setting 1 (Null): equal means, unequal spread. $\bar\theta_A=\bar\theta_B=-0.5$,
  $\mathrm{sd}(\theta_A)=1.27$ against $\mathrm{sd}(\theta_B)=0.11$. So the two groups do not differ in average quality.
\item Setting 2 (Alternative): a true gap, equal spread. Group means $+0.2$ and $-0.2$, with $\delta = 0.4$, matched spreads $\approx0.45$. So the two groups differ in their average quality.
\end{enumerate}
Again, we compare the naive method with the principled GLMM fit. Two aspects are evaluated: the proportion of replications identifying group $A$ as the winner, and the rejection rate of the test of $H_0: \bar\theta_A = \bar\theta_B$ at level $0.05$. For hypothesis testing we use the percentile bootstrap, rejecting when the $95\%$ interval of the batch-resampled contrast excludes zero, for both the naive estimator and the GLMM fit, since the bootstrap is the preferable method in practice for AI judges, with the GLMM being only an approximation. For completeness we also include the Wald test of the GLMM. Each cell uses $2000$ replications for every arm.

\begin{figure}[H]
\centering
\begin{subfigure}[t]{0.48\textwidth}
  \includegraphics[width=\linewidth]{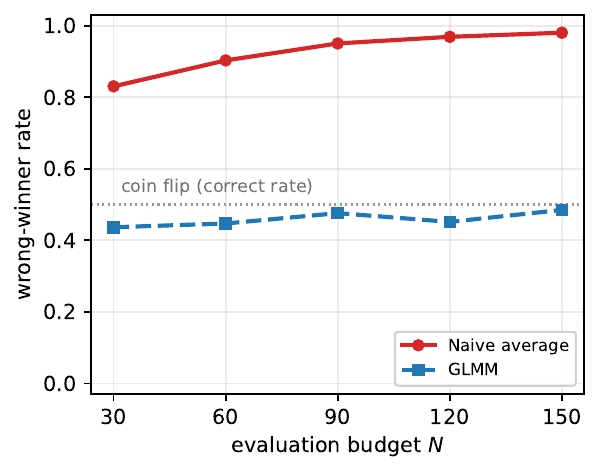}
  \caption{Setting 1, point estimate: rate of naming group $A$.}
  \label{fig:conv-s1-sign}
\end{subfigure}\hfill
\begin{subfigure}[t]{0.48\textwidth}
  \includegraphics[width=\linewidth]{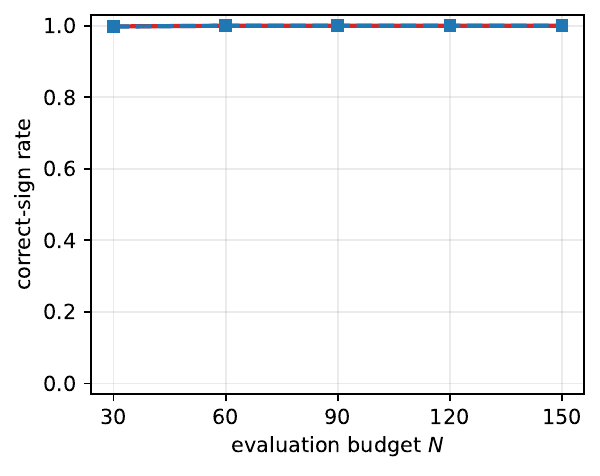}
  \caption{Setting 2, point estimate: correct-sign rate.}
  \label{fig:conv-s2-sign}
\end{subfigure}

\smallskip
\begin{subfigure}[t]{0.48\textwidth}
  \includegraphics[width=\linewidth]{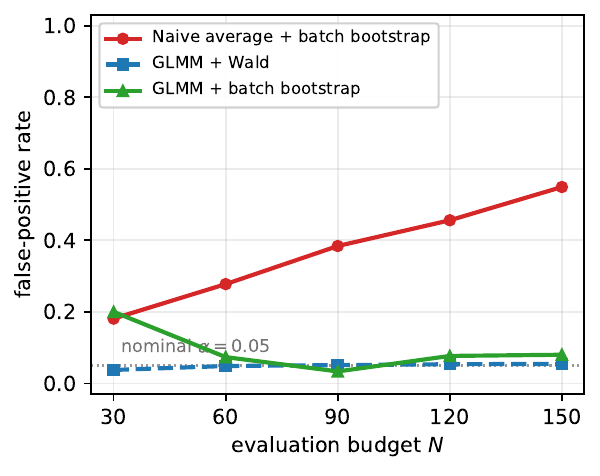}
  \caption{Setting 1, inference: false-positive rate.}
  \label{fig:conv-s1}
\end{subfigure}\hfill
\begin{subfigure}[t]{0.48\textwidth}
  \includegraphics[width=\linewidth]{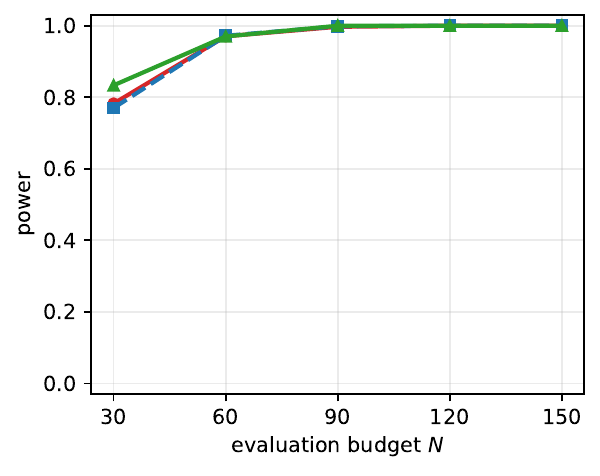}
  \caption{Setting 2, inference: power.}
  \label{fig:conv-s2}
\end{subfigure}
\caption{Group comparison results ($V=30$), for the naive estimator and the GLMM fit.
Top row: the rate at which the sign of each estimator picks group $A$ as the winner.
Bottom row: the false-positive rate and the power of the test.}
\label{fig:conv-three}
\end{figure}
The top row of \Cref{fig:conv-three} gives, for
each estimator, the rate at which its sign declares group $A$ the better one.
In Setting 1, the null case, the naive average names the higher-variance group $A$ in $83.1\%$ of replicates at $N=30$, rising to $98.1\%$ at $N=150$, against a calibrated value of $50\%$, which the GLMM stays close to. In Setting 2, where group $A$ is truly better, both estimators recover the correct sign in at least $99.8\%$ of replicates at every budget.

The bottom row of \Cref{fig:conv-three} reports the rejection rates on the same
replicates, illustrating the effects on hypothesis testing at level $0.05$. In Setting 1, the rejection rate corresponds to the false-positive rate. The rate of the naive method climbs with the budget, from $18.1\%$ at $N=30$ to $54.9\%$ at $N=150$, while the tests based on the GLMM fit do not: the Wald test holds at $3.8$--$5.5\%$ throughout, and the bootstrap at $6.0$--$8.2\%$ once past the smallest budget, where resampling only thirty batches leaves it liberal at $18.2\%$. In Setting 2, the rejection rate corresponds to the power of the test, and the three methods are indistinguishable: $78.3\%$, $77.0\%$ and $75.9\%$ at $N=30$, and all above $99.7\%$ from $N=90$ on.

The same patterns hold for a smaller pool of $V=6$ items as well, whose details can be found in Supplement~\ref{supp:small-pool}.

Lastly, different from \Cref{sec:leaderboard}, the Williams square design does not help much in the current case for two reasons. First, in group comparison the primary difficulty is the bias of \Cref{thm:conv-bias}, not the variance that the balanced design improves. Our experiments also confirm this (see \Cref{app:williams-group}). Second, the bootstrap is less feasible for a Williams square: its batches are rows of a balanced array, and resampling them destroys the very balance the design provides. Resampling blocks is possible, but that only works for very large $N$, which is not practical in this context.

\section{Robustness to Deep Memory}
\label{sec:boundary}

The analysis of \Cref{sec:leaderboard,sec:convexity} is developed entirely within the
order-1 model: the judge's memory is first-order.  This section introduces simulation evaluation using the lag-3 model as the generative model.  One question of particular interest is: do the conclusions of the two
preceding sections survive when the judge's memory runs deeper than the fitted model?

Specifically, the data-generating judge is the $K=3$ member of \eqref{eq:shift-K} while
the fitted GLMM assumes $K=1$, so the working model is deliberately wrong. The results below are therefore
evidence of empirical robustness within the tested settings, not a guarantee.
We compare it against a correctly specified lag-3 GLMM fitting. Each task is examined on the pool relevant to it: ranking at $V=6$, where \Cref{sec:leaderboard} found the design effect, and group comparison at $V=30$, the pool used in \Cref{sec:conv-illusion}. 
\paragraph{Task 1: leaderboard ranking at $V=6$.} The previous ranking patterns persist under lag-3 generation: Williams designs give a modest gain for naive averaging in the uneven-position setting, and the fitted lag-1 and lag-3 GLMMs perform comparably. Supplement~\ref{supp:deep-ranking} provides the simulation details and \Cref{tab:bd-rank}.

\paragraph{Task 2: group comparison at $V=30$.}
We follow \Cref{sec:conv-illusion} on the same $V=30$ pool split $15$ versus $15$, again
changing only the judge to a lag-3 GLMM.  The generating judge is Claude Haiku 4.5, Setting 2 of Task 1, $(\beta_1,\beta_2,\beta_3)=(-1.055,-0.380,-0.311)$ with $\gamma=1.879$---the judge whose position profile is uneven.  \Cref{tab:bd-group30} gives the result. The naive false-positive rate rises from $26.0\%$ to $70.5\%$ across $N$ while both GLMM versions give approximately correct control of the false positive rate, for a sufficiently large $N$. The bias of the naive method drives the sign rate from $86.5\%$ to
$100\%$, and the false-positive rate of the test is likewise uncontrolled.

\begin{table}[!htbp]
\centering
\caption{The group-comparison results under a lag-$3$ model for $V=30$. The bootstrap
($B=200$) is used for inference. The results are under the null case when $\bar\theta_A = \bar\theta_B$.}
\label{tab:bd-group30}
\begin{tabular}{lccccccc}
\toprule
 & \multicolumn{3}{c}{Type-I error, null (\%)} & & \multicolumn{3}{c}{Sign rate (Group A), null (\%)} \\
\cmidrule(lr){2-4}\cmidrule(lr){6-8}
 & Naive & GLMM & GLMM & & Naive & GLMM & GLMM \\
$N$ & average & lag-1 & lag-3 & & average & lag-1 & lag-3 \\
\midrule
$30$  & $26.0$          & $24.5$ & $25.0$ & & $86.5$           & $41.0$ & $40.5$ \\
$60$  & $38.0$          & $10.5$ & $11.5$ & & $95.5$           & $43.5$ & $44.5$ \\
$90$  & $51.0$          & $6.0$  & $4.5$  & & $97.5$           & $46.0$ & $46.5$ \\
$120$ & $66.5$          & $6.5$  & $6.0$  & & $99.0$           & $49.5$ & $51.5$ \\
$150$ & $\mathbf{70.5}$ & $5.5$  & $6.5$  & & $\mathbf{100.0}$ & $44.5$ & $46.5$ \\
\bottomrule
\end{tabular}
\end{table}

In summary, when the judge's memory runs beyond the fitted
lag-1 kernel, the main messages
of the preceding sections still hold. The first-order model fitting is still useful though it is only approximately correct. Such an observation serves as reassuring evidence for using a reasonable GLMM fit for AI judges.

\section{Data Example: Comparing Two Graphical Model Estimation Methods by AI-Judged Edge Quality}
\label{sec:data-example}

\citet{li2020ggmnetwork} introduced a Gaussian graphical model for \emph{network-linked}
observations called ``GNC-lasso''. They applied the method to a corpus of statistics papers from \cite{jijin2016coauthorship}, by leveraging the coauthorship network between authors. The variables in the graphical model are $300$ statistical terms, and each observation is an author's term-frequency vector, linked by a coauthorship
network of $635$ authors. The estimated edges are conditional dependencies between
terms, such as \emph{empirical--likelihood} or \emph{false--discovery}. The benchmark competitor is the ordinary graphical lasso (``glasso''; \citealp{friedman2008glasso,rothman2008spice}),
which estimates a sparse precision matrix by an $\ell_1$-penalized Gaussian likelihood
and ignores the coauthorship network. That study uses the two methods to select $25$ edges for the graphical structure and assesses them pair by pair. The two estimated graphs are shown in Supplement~\ref{supp:term-graphs} (\Cref{fig:de-graphs}). Using human evaluation, they 
conclude that GNC-lasso recovers the more meaningful associations. We take this human
assessment as the reference answer, and test the performance of using AI judges for the same task.

The union of the two lists contains $40$ distinct term-pairs, labeled
by provenance: GNC-only ($\mathcal G$, $15$ pairs), glasso-only ($\mathcal L$, $15$), or
shared ($\mathcal S$, $10$). We exclude the shared pairs, since every judge scores them near-unanimously and they
carry no information about the contrast. This leaves the between-method
difference on the $30$ distinctive pairs. The three judges of \Cref{sec:physics} (Gemini
2.5 Flash, GPT-5.4-mini, Claude Haiku 4.5) score each pair as a meaningful statistical
association ($1$) or not ($0$) over $160$ randomly shuffled batches.

We compare the two aggregation rules: the
naive rule using $\Dhat_{\Naive}$ in \eqref{eq:naive}, and the GLMM-based counterpart of
\Cref{sec:physics}, 
$\Dhat_{\GLMM}$. The
human evaluation in \cite{li2020ggmnetwork} indicates a positive contrast. Because many pairs are judged
unanimously, a case of perfect separation at which the unpenalized item effects diverge, we regularize the model fitting by adding an independent Cauchy prior on the item effects as suggested by \citet{gelman2008weakly}. For a more informative demonstration, we use the Wald-type $p$-value for both rules, with the standard error estimated by resampling whole batches ($B=1000$). Using the quantile-based bootstrap test of \Cref{sec:convexity} instead gives the same rejection decisions, but the comparison is less informative because the smallest $p$-value it can report is about $2/B$. To demonstrate the tradeoff between accuracy and sample size (AI tokens), we check the aggregation results on the first $N$ batches for
$N=20,40,\dots,160$, the sequence an analyst who stopped collecting early would have
seen. \Cref{fig:de-sweep} and \Cref{tab:de-lenient} show how the two methods agree and disagree, depending on the three judges.
\begin{itemize}
\item For \textbf{OpenAI} they agree and both are right. The contrast stays positive and both
$p$-values fall steadily. The difference between the two methods lies in the cost. The GLMM reaches any given
level of confidence in roughly half the batches. It reaches the floor of the plotted scale by $N=40$, while the naive method needs $N\approx100$. So the GLMM inference can save tokens substantially for the same conclusion.
\item For \textbf{Claude} the naive method never delivers a conclusive decision. It is not significant at a single
sample size in the sequence. Its point estimate
crosses zero at $N=60$ and stays negative. The GLMM is significant from $N=20$ onward and
falls to $p\approx2\times10^{-12}$. Thus the GLMM achieves the desired conclusion but the naive method fails to get there.
\item For \textbf{Gemini} the naive rule is confidently wrong. Its contrast is negative and becomes significant from $N=40$ onward. The GLMM never accumulates evidence and is not
significant at the full sample. More detailed examination of the data reveals that Gemini is the judge furthest into saturation, with an accept rate of $0.86$ and $60\%$ of pairs unanimous. So it has stopped discriminating
between the two estimated graphs. The GLMM inference returns an inconclusive result, which is the
correct description of the evidence.
\end{itemize}

\begin{figure}[t]
\centering
\includegraphics[width=\textwidth]{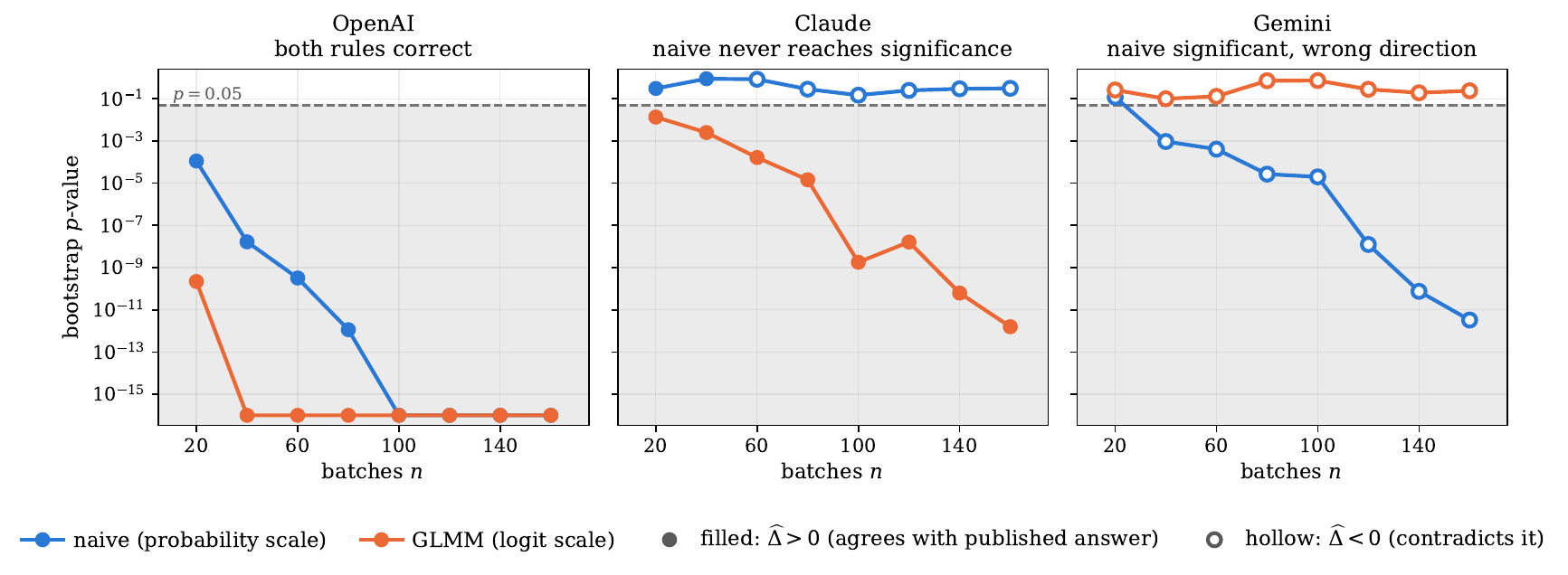}
\caption{Wald $p$-value with batch-bootstrap standard error against the number of batches for the two aggregation rules, on
the $30$ distinctive word pairs, fitted on the first $N$ batches. Shading marks
$p<0.05$. 
Filled markers are contrasts in the direction of the published reference answer
($\Dhat>0$, favoring GNC-lasso); hollow markers contradict it. Values below $10^{-16}$ are displayed at the axis limit.}
\label{fig:de-sweep}
\end{figure}

\begin{table}[t]
\centering
\caption{Full sample ($N=160$): contrasts against the published reference answer, which is $\Dhat>0$; $p$-values are from the Wald statistic with bootstrap standard error. ``Correct'' means significance with the sign matching the human evaluation in \cite{li2020ggmnetwork}; ``inconclusive'' means not significant; ``wrong'' means significant with the
opposite sign.}
\label{tab:de-lenient}
\begin{tabular}{lcccccc}
\toprule
& \multicolumn{3}{c}{naive} & \multicolumn{3}{c}{GLMM} \\
\cmidrule(lr){2-4}\cmidrule(lr){5-7}
Judge & $\Dhat_{\Naive}$ & $p$ & outcome & $\Dhat_{\GLMM}$ & $p$ & outcome \\
\midrule
OpenAI & $+0.092$ & ${<}10^{-12}$ & correct & $+2.30$ & ${<}10^{-12}$ & correct \\
Claude & $-0.007$ & $0.31$ & inconclusive & $+1.16$ & $2\times10^{-12}$ & correct \\
Gemini & $-0.038$ & $3\times10^{-12}$ & wrong & $-0.18$ & $0.24$ & inconclusive \\
\bottomrule
\end{tabular}
\end{table}

\section{Discussion and Practical Guidelines}
\label{sec:discussion}

Large language models increasingly perform tasks once requiring human judgment, often at a scale where statistical properties matter. When a judge scores a sequence of items, its ratings are dependent, and prompt engineering alone cannot determine what their average represents. The adoption of LLM-as-a-judge evaluation therefore requires the corresponding statistical theory to guide the design and interpretation for such procedures.

In this paper, we model the judge through a Markov GLMM \emph{working model}. We treat it as an approximation rather than the true mechanism and assess its adequacy using out-of-sample predictive checks across frontier models. This formulation turns evaluation into inference under an explicit, empirically validated model. The two tasks we study require different strategies. For \emph{leaderboard ranking}, the naive marginal average is consistent under a mild separation condition, while a Williams square design can improve efficiency in many cases at essentially no computational cost. In contrast, for \emph{group comparison}, the marginal average can fail severely because the nonlinear link converts differences in within-group variation into persistent bias. The GLMM estimation provides valid inference for both tasks. The method also remains empirically robust beyond our theoretical setting: when the judge has memory beyond the lag-$1$ kernel, the GLMM achieves approximately correct error control across all regimes considered. We therefore recommend GLMM inference with a random-shuffle batch design for reliable inference using AI judges.

Our conclusions have several limitations. The design theory is developed for a first-order Markov kernel. Although deeper memory is examined empirically and accommodated by the randomized GLMM, optimal designs for higher-order carryover remain open. Two extensions are especially important. First, placing all $V$ items in one batch becomes infeasible for large item pools under bounded context for many AI judges; incomplete block designs may offer a natural direction to overcome this restriction. Second, we focus on binary accept/reject outcomes in the paper, which cover pairwise and accept/reject judging but not scalar scores. The framework could replace the Bernoulli-logit model with a cumulative-link or continuous mixed model. However, both the nonlinear bias and design-optimality arguments depend on link curvature and must be re-derived.

\paragraph{Acknowledgment} Generative AI has been used to assist with language editing, coding, results visualization, and improving and checking proofs. 
T. Li was supported in part by NSF grant DMS-2515367. 
J. Ding was supported in part by NSF grant CAREER-2338506.

\bibliography{refs}
\clearpage
\section*{Supplementary Material}
\appendix
%
%

\setcounter{figure}{0}
\renewcommand{\thefigure}{S\arabic{figure}}
\renewcommand{\theHfigure}{supp.\arabic{figure}}
\setcounter{table}{0}
\renewcommand{\thetable}{S\arabic{table}}
\renewcommand{\theHtable}{supp.\arabic{table}}

\noindent\textbf{Summary.} This supplementary material collects additional analyses and experiments in the order of their appearance in the main text, followed by the proofs. Section~\ref{supp:predictive-details} provides the predictive-check protocol, transition results, interval widths, and conditional-entropy analyses. Section~\ref{supp:ranking-simulation} details the leaderboard simulation settings. Sections~\ref{supp:small-pool} and~\ref{app:williams-group} report group-comparison experiments with a small item pool and with Williams designs, respectively. Section~\ref{supp:deep-ranking} gives the leaderboard-ranking experiment under deeper memory, and Section~\ref{supp:term-graphs} presents the estimated networks used in the data example. Finally, Section~\ref{supp:glmm-theory} states the GLMM estimation result, and Sections~\ref{app:williams} and~\ref{app:convexity} collect the technical definitions, proofs, and supporting theory for leaderboard ranking and group comparison, including design variance, correct selection, nonlinearity bias, and finite-sample concentration.

\section[Details on out-of-sample predictive evaluation]{Details on out-of-sample predictive evaluation}
\label{supp:predictive-details}

We collected $100$ new shuffled batches of the same $30$ proposals, giving $3{,}000$ fresh ratings per judge under new randomized orderings. We
then froze the parameters from the original $200$-batch fits, with no
retraining, and treated each frozen GLMM as generative. Each judge is frozen at the
memory depth its own model selection chose in \Cref{sec:fit-four}: $K=1$ for Gemini
2.5 Flash, and $K=3$ for GPT-5.4-mini and Claude Haiku 4.5, each with the cold-start
term and the batch random intercept. We then simulated $B=500$ synthetic
datasets under the new orderings. We score twelve transition terms for every judge. The terms are the joint $0\!\to\!0$,
$0\!\to\!1$, $1\!\to\!0$, $1\!\to\!1$ cells of the lag-$1$, lag-$2$, and lag-$3$
transition matrices ($4\times3=12$ terms), regardless of the lags that judge's model
carries, so a judge frozen at $K=1$ is scored on transitions two and three steps back
that its model never saw. The $500$ synthetic datasets are then used to construct
$95\%$ prediction intervals for the transition terms. Note that these intervals are
constructed from the frozen GLMMs alone, without using the true judge scores on the
$100$ test batches. We then check whether each real transition probability from the
$100$ test batches falls inside its $95\%$ prediction interval.

Across the three judges we score $36$ terms, and we obtain $34/36$, as
\Cref{fig:ppc} shows. Two judges are covered in full:
\emph{Gemini's lag-$1$ model reproduces all twelve terms} ($12/12$), including the
out-of-order lag-$2$ and lag-$3$ cells it was never fit on, and Claude likewise covers
$12/12$. The two misses are the same failure mode, the $0\!\to\!1$ cell at a
deeper lag where the held-out probability runs slightly above the frozen model's
interval: GPT-5.4-mini at lag-$2$ and lag-$3$ ($10/12$). A miss rate of
$2/36\approx5.6\%$ is close to the $5\%$ a calibrated battery of $95\%$ intervals
should produce; the misses are not scattered but concentrated in the deep $0\!\to\!1$
cell, a small and reproducible higher-order-memory gap for that judge rather than
a global inadequacy. The logistic GLMM is thus sufficient to reproduce the judges'
sequence dynamics, with no drift detectable by this battery; the intervals condition
on the frozen point estimates (parameter-estimation uncertainty is not propagated),
so we read the check as a strong graphical predictive check rather than a formal
goodness-of-fit test. On this evidence no more elaborate surrogate is warranted.

Note that a high coverage rate is necessary but not, by itself, compelling: a $95\%$ rate could
be attained \emph{vacuously} by a severely misspecified model whose predictive
intervals are so wide that they contain almost anything. \Cref{fig:ppc} rules this out,
because it draws every interval at its true width. On the probability scale the transition
intervals span $0.013$ to $0.055$ end to end, and the entropy intervals $0.010$ to $0.041$
bits, so a bar that covers a held-out value is unlikely to have done so by accident. The widths also differ enough among themselves, by a factor of four, for the
comparison to be informative rather than uniform: the $1\!\to\!1$ cells, which predict the
largest probabilities, carry the widest intervals, and the entropies the narrowest. Relative to the quantity each one predicts these widths are small, between $2.5\%$ and $14.8\%$ of the observed value for the transition terms and under $2.5\%$ for the entropies, so the intervals are informative on the scale that matters as well as in absolute terms. We can see that the GLMM's prediction intervals are narrow and structural: they do not cover the held-out transition probabilities merely by being wide, but reproduce the \emph{asymmetric} shifts in those probabilities induced by the negative serial effect. The empirical $0\!\to\!1$ and $1\!\to\!0$
rates range across where the frozen model predicts, not merely staying far inside the bands. Therefore, the check is informative because the intervals are tight and the coverage is nominal.

For completeness, we also include a single
scalar information-theoretic check for global information: the order-$k$ \emph{conditional entropy} of the
binary label sequence, $H(Y_t\mid Y_{t-1},\dots,Y_{t-k})$ in bits, for $k=1,2,3$. As before, for a good model, the prediction interval of the conditional entropy should cover the real conditional entropy on the test data at the nominal level. These three statistics per judge form the lower block of each panel of \Cref{fig:ppc}. All $9$ of them (three judges $\times$ three orders) fall
inside their $95\%$ prediction intervals, which are again narrow ($0.010$--$0.041$ bits
wide). In particular, even Gemini's lag-$1$ model
covers the order-$2$ and order-$3$ entropy of its held-out sequence.

\clearpage
\section[Leaderboard ranking simulation]{Leaderboard ranking simulation}
\label{supp:ranking-simulation}

This section provides the full simulation specification for \Cref{sec:lb-empirical}.

\Cref{ass:nondegen} determines whether balancing can help at all, so we test the prediction where it holds and where it nearly fails. In concrete terms it requires a \emph{non-flat position profile}: in a batch of equal-quality items, the expected rating
must vary across slots. That profile is a property of $(\beta,\gamma)$, and the two
compete: the first slot has no predecessor, so its context is $\gamma$, while at later
slots the context averages $\beta$ times the chance the previous item was accepted. When
those two nearly cancel the profile is flat and there is no design variance to remove.

We therefore simulate two settings, both at $V=6$. Both are taken from the judges fitted in \Cref{sec:physics}, at their estimated coefficients, rather than chosen to produce the effect. The two settings therefore span \Cref{ass:nondegen} because the fitted judges themselves do. Setting 1 is Gemini 2.5 Flash, whose $\beta=-0.529$
and $\gamma=-0.528$ nearly cancel: the expected rating runs $(0.658,0.695,0.691,0.691,
0.691,0.691)$ across the six slots, flat to within $0.04$, so the removable
design variance is nearly zero. Setting 2 is Claude Haiku 4.5, whose
$\beta=-1.055$ and $\gamma=+1.879$ do not: the profile runs
$(0.910,0.419,0.536,0.508,0.515,0.513)$---an
item judged first wins $91\%$ of the time and one judged second only
$42\%$---and the removable design variance is far larger. Each
setting also carries its own judge's fitted batch variance, $\sigma_b^2=0.446$ and
$0.418$, and its own common quality level $\theta^\ast$, obtained by solving $\pi(\theta^\ast)$ against that judge's observed acceptance rate, $0.686$ and $0.567$, rather than by tuning. The two therefore differ in level as well as in sequential structure. In either case, the design variance that balancing can remove is governed by the position profile, which is nearly absent in Setting 1 and substantial in Setting 2.

The item panel is near-flat, $\bm\theta=(\theta^\ast+\delta,\theta^\ast,\dots,
\theta^\ast)$, with one best item and a tied field. We set $\delta=2.8/\sqrt N$, the
local scaling of \Cref{thm:lb-pcs}: the selection problem tightens as fast as the budget grows, so the difficulty is held constant and any advantage that persists reflects the limit rather than one particular sample size. Budgets are $N\in\{6,12,18\}$---one,
two and three Williams blocks---with $10{,}000$ replicates per cell, giving a Monte Carlo
standard error of $0.71$ percentage points on each design comparison. We report all four
combinations of estimator and design; ties among the win rates, which are frequent at
these budgets, are broken uniformly at random.

\clearpage
\section[Group comparison with a small item pool]{Group comparison with a small item pool}
\label{supp:small-pool}

We extend the simulations in \Cref{sec:conv-illusion} to the smaller item pool below.

For completeness, a small pool of $V=6$, split $3$ versus $3$, is also studied. The rest of the setup is unchanged. The corresponding results are shown in \Cref{fig:conv-v6}, and the overall patterns remain the same as in the $V=30$ case.

\begin{figure}[!htbp]
\centering
\begin{subfigure}[t]{0.48\textwidth}
  \includegraphics[width=\linewidth]{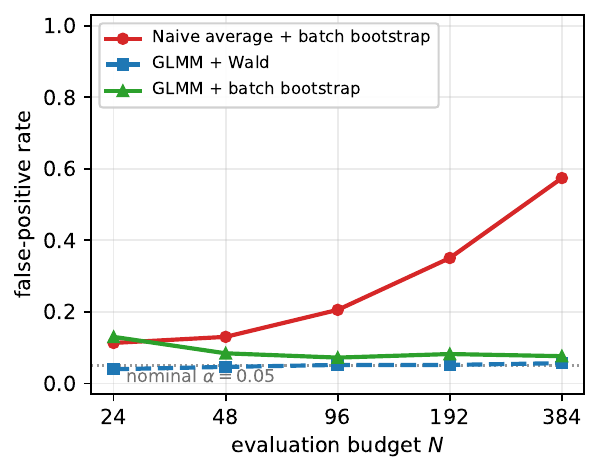}
  \caption{Setting 1: equal means, unequal spread.}
  \label{fig:conv-v6-s1}
\end{subfigure}\hfill
\begin{subfigure}[t]{0.48\textwidth}
  \includegraphics[width=\linewidth]{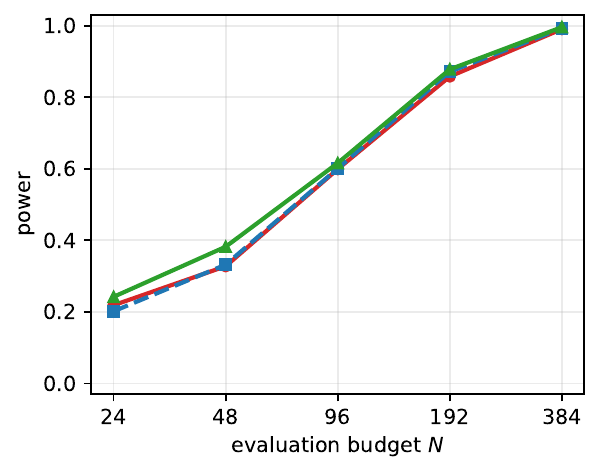}
  \caption{Setting 2: true gap, equal spread.}
  \label{fig:conv-v6-s2}
\end{subfigure}
\caption{The small-pool counterpart of the bottom row of \Cref{fig:conv-three}: the same two
scenarios and the same three procedures at $V=6$, split $3$ versus $3$.}
\label{fig:conv-v6}
\end{figure}

\clearpage
\section{Williams squares in group comparison: an empirical demonstration}
\label{app:williams-group}

\Cref{sec:leaderboard} showed that a Williams square strictly improves the naive
estimator for \emph{ranking} by removing a design-induced variance term. This
section documents the experiment behind the claim in \Cref{sec:convexity} that no
such improvement materializes for group comparison, where the obstacle is the bias
of \Cref{thm:conv-bias} rather than the variance.

\emph{Setup.} We rerun the two estimators of \Cref{sec:conv-illusion}---the naive
marginal average and the GLMM logit-scale contrast---under a Williams Latin square
in place of the random shuffle, with everything else held fixed: the same quality
panels as Settings 1 and 2 of \Cref{sec:conv-illusion} (verified identical to the
random-shuffle study), the same two pool sizes, $V=30$ split $15$ vs.\ $15$ over
$N\in\{30,\dots,150\}$ and $V=6$ split $3$ vs.\ $3$ over $N\in\{24,\dots,384\}$,
and $2000$ independently seeded replicates per cell. The design is the
symbol-randomized Williams design of \Cref{sec:leaderboard} ($V$ is even in both
cases): every budget is a whole multiple of $V$, so each replicate uses $N/V$
complete squares, and the item-to-symbol assignment is redrawn independently for
each square, with exact position and carryover balance verified within every
block. Each replicate records the point estimate and its sign only.

No test is computed under the Williams arms, and the reason is structural. A
bootstrap of batches is licensed by their exchangeability under a random shuffle,
whereas the batches of a Williams square are rows of a balanced array, and
resampling them with replacement destroys the balance the design exists to provide.
The correct resampling unit is then the \emph{complete square}, a block of $V$
batches, and as with block bootstraps for dependent data, consistency requires the
number of blocks, not merely the number of observations, to grow
\citep{kunsch1989,lahiri2003}; with few blocks such tests over-reject substantially
\citep{cameron2008bootstrap}. At $V=30$ a budget of $N=150$ buys five squares, and
reaching even a few dozen blocks would require $N\ge20V$ evaluations, a direct and
largely wasted token cost for judge-based evaluation. The usual escape, a
model-based bootstrap that holds the design fixed and re-simulates responses from a
fitted model \citep[Ch.~6]{davison1997}, is closed off by the premise: the naive
average is attractive precisely because it assumes no model.

\emph{Result.} \Cref{tab:williams-sign} reports the rate at which each arm names
Group~A the winner under the equal-means null Setting 1. Under that null the naive
average names the higher-variance group $98.1\%$ of the time under a random shuffle
at $V=30$ and $98.5\%$ under a Williams square; across all $40$ matched cells the
largest discrepancy between a Williams arm and its random-shuffle counterpart is
$2.9\%$ against a Monte Carlo standard error of $1.6\%$, and no cell reaches a
Bonferroni-corrected threshold. This is what \Cref{thm:conv-bias} predicts: the
nonlinearity bias is a property of the \emph{estimator}, namely of averaging on the
probability scale, and the design enters the theorem only through
\Cref{ass:design}; balancing removes design-induced \emph{variance} and has no
purchase on estimator \emph{bias}.

\begin{table}[!htbp]
\centering
\caption{The balanced design does not repair the naive estimator. Rate at which each
arm names the higher-variance Group~A under the equal-means null Setting 1, in percent; the
groups are equal, so $50$ is the calibrated value and no test is involved. Each
estimator is run under a random shuffle and under a Williams square on independently
seeded replicates ($2000$ per cell). The naive average climbs toward $100\%$ under
\emph{both} designs, and the GLMM is near the coin flip under both. The Setting 2 control, with
a true $0.4$-logit gap and matched spreads, separates all four arms from zero error by
the largest budget (Group~B named in $0.0\%$ of replicates at $V=30$, and $0.0\%$ at
$V=6$, for every arm).}
\label{tab:williams-sign}
\begin{tabular}{llccccc}
\toprule
& & \multicolumn{5}{c}{Group~A named (\%), by budget $N$} \\
\cmidrule(lr){3-7}
\multicolumn{2}{l}{$V=30$ ($15$ vs.\ $15$)} & $30$ & $60$ & $90$ & $120$ & $150$ \\
\midrule
\multirow{2}{*}{Naive} & random shuffle   & $83.1$ & $90.3$ & $95.1$ & $96.9$ & $98.1$ \\
                       & Williams square  & $82.7$ & $91.3$ & $94.6$ & $97.3$ & $98.5$ \\
\multirow{2}{*}{GLMM}  & random shuffle   & $43.7$ & $44.7$ & $47.6$ & $45.2$ & $48.5$ \\
                       & Williams square  & $42.7$ & $47.3$ & $48.0$ & $48.0$ & $47.6$ \\
\midrule
\multicolumn{2}{l}{$V=6$ ($3$ vs.\ $3$)} & $24$ & $48$ & $96$ & $192$ & $384$ \\
\midrule
\multirow{2}{*}{Naive} & random shuffle   & $69.3$ & $75.3$ & $85.8$ & $93.4$ & $98.4$ \\
                       & Williams square  & $68.6$ & $74.7$ & $85.2$ & $92.3$ & $97.9$ \\
\multirow{2}{*}{GLMM}  & random shuffle   & $49.7$ & $49.6$ & $48.9$ & $50.1$ & $50.8$ \\
                       & Williams square  & $49.2$ & $49.4$ & $49.0$ & $49.5$ & $51.0$ \\
\bottomrule
\end{tabular}
\end{table}

\clearpage
\section[Leaderboard ranking under deeper memory]{Leaderboard ranking under deeper memory}
\label{supp:deep-ranking}

This section gives the ranking experiment summarized in \Cref{sec:boundary}, comparing lag-1 and lag-3 fits under lag-3 data generation.

\paragraph{Task 1: leaderboard ranking at $V=6$.}
We repeat \Cref{sec:lb-empirical} exactly---the same two settings and the same local scaling $\delta=2.8/\sqrt N$---but change only the judge, which now follows the lag-3 model.  Each setting keeps the lag-1 coefficient it had in \Cref{sec:leaderboard} and gains the fitted lag-2 and lag-3 terms of the same judge, so it deepens without becoming a different judge: $(\beta_1,\beta_2,\beta_3)=(-0.529,-0.090,-0.179)$ for Gemini 2.5 Flash in Setting 1 and $(-1.055,-0.380,-0.311)$ for Claude Haiku 4.5 in Setting 2, with $\gamma$ and $\sigma_b^2$ carried over unchanged.  \Cref{tab:bd-rank} gives the leaderboard selection accuracy, and it reproduces the patterns in \Cref{sec:leaderboard}. Both the naive method and the GLMM model fitting work for this task. In Setting 1, the Williams square design gains
nothing. In Setting 2, where the position
profile is uneven, the Williams square design gains about $2\%$ improvement on the naive method. On GLMM fitting, the design is not impactful. However, note that the performance of the GLMM fitting based on the misspecified lag-1 model is comparable to that of the lag-3 model fitting, suggesting that using the lag-1 model can still achieve reasonable empirical performance.

\begin{table}[!htbp]
\centering
\caption{Selection accuracy (\%) of \Cref{sec:leaderboard} under a lag-3 GLMM model.}
\label{tab:bd-rank}
\begin{tabular}{llcccccc}
\toprule
 & & \multicolumn{3}{c}{Random shuffle (\%)} & \multicolumn{3}{c}{Williams design (\%)} \\
\cmidrule(lr){3-5}\cmidrule(lr){6-8}
 & & Naive & GLMM & GLMM & Naive & GLMM & GLMM \\
Setting & $N$ & average & lag-1 & lag-3 & average & lag-1 & lag-3 \\
\midrule
\multirow{3}{*}{Setting 1 (Flat)}
  & $6$  & $44.4$ & $44.2$ & $43.3$ & $43.7$ & $43.9$ & $43.2$ \\
  & $12$ & $49.1$ & $49.1$ & $48.6$ & $48.3$ & $48.3$ & $48.1$ \\
  & $18$ & $47.9$ & $48.0$ & $47.8$ & $48.9$ & $49.1$ & $49.0$ \\
\midrule
\multirow{3}{*}{Setting 2 (Uneven)}
  & $6$  & $46.9$ & $49.4$ & $47.1$ & $48.7$ & $49.8$ & $48.8$ \\
  & $12$ & $48.4$ & $51.0$ & $50.2$ & $50.6$ & $51.5$ & $51.1$ \\
  & $18$ & $49.0$ & $51.7$ & $51.3$ & $51.2$ & $51.7$ & $51.7$ \\
\bottomrule
\end{tabular}
\end{table}

\clearpage
\section[Estimated networks for the data example]{Estimated networks for the data example}
\label{supp:term-graphs}

The networks below are the graphical lasso and GNC-lasso estimates described in \Cref{sec:data-example}. The AI-judge comparison uses their 30 distinctive term-pairs.

\Cref{fig:de-graphs} shows the two estimated graphs.

\begin{figure}[p]
\centering
\begin{subfigure}[t]{\textwidth}
  \centering
  \includegraphics[width=\linewidth]{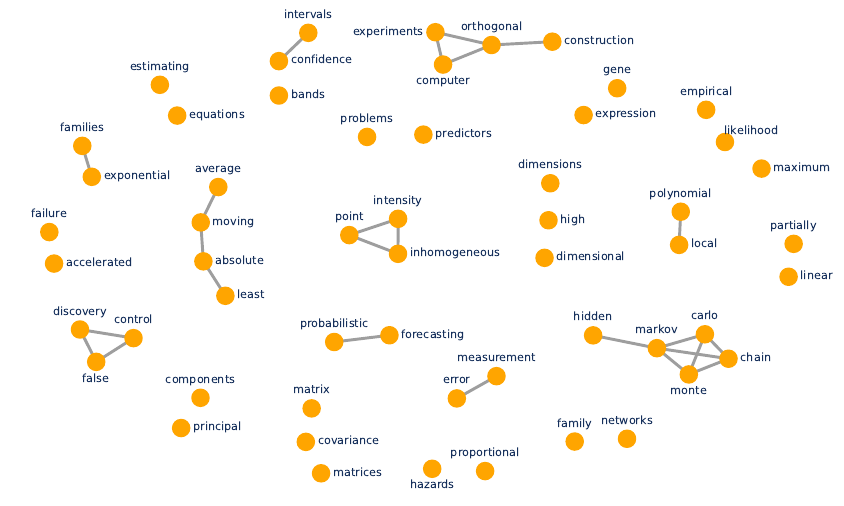}
  \caption{Graphical lasso.}
  \label{fig:de-graph-glasso}
\end{subfigure}

\bigskip
\begin{subfigure}[t]{\textwidth}
  \centering
  \includegraphics[width=\linewidth]{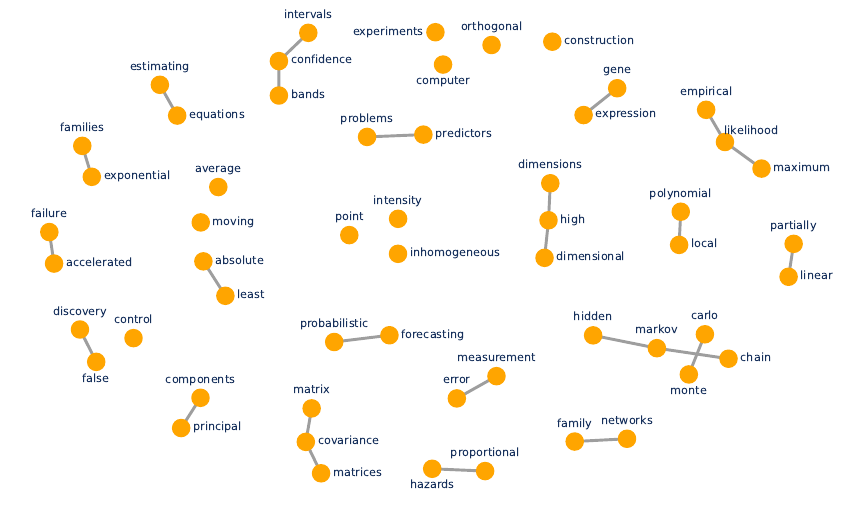}
  \caption{Network-cohesion graphical lasso.}
  \label{fig:de-graph-gnc}
\end{subfigure}
\caption{The two estimated term graphs of \citet{li2020ggmnetwork} by glasso and GNC-lasso.}
\label{fig:de-graphs}
\end{figure}

\clearpage
\section[GLMM estimation theory]{GLMM estimation theory}
\label{supp:glmm-theory}

\begin{proposition}
\label{prop:glmm}
Suppose the judge is evaluated under \eqref{eq:bern}--\eqref{eq:shift-K} at some
memory depth $K<V$ and under \Cref{ass:design}, whose randomization units
(single batches under clause (a); complete blocks of $B=V$ batches under clause
(b)) are independent clusters---conditional on the allocation, the chain rule
factorizes each unit's joint response probability exactly through the lagged
ratings, so the likelihood is a product over units and conditioning on the lag
covariates adds no separate factor. Assume the standard regularity conditions
for mixed-model likelihood asymptotics hold (identifiability of every item
contrast from the marginal likelihood; a positive-definite limiting per-unit
Fisher information $\bm{\mathcal I}_B(\bm\eta_0)$, and write
$\bm{\mathcal I}(\bm\eta_0):=\bm{\mathcal I}_B(\bm\eta_0)/B$ for the
information per batch; and $\bm\eta_0$ interior to the parameter space---in
particular $\sigma_b^2>0$, the boundary case not being covered). Then
the maximum likelihood estimator $\hat{\bm\eta}_N$ is consistent and asymptotically
normal as $N\to\infty$,
\begin{equation}
  \hat{\bm\eta}_N \xrightarrow{p}\bm\eta_0,
  \qquad
  \sqrt{N}\,(\hat{\bm\eta}_N-\bm\eta_0)
  \xrightarrow{d}\N\!\big(\bm 0,\;\bm{\mathcal I}(\bm\eta_0)^{-1}\big).
  \label{eq:can}
\end{equation}
The same conclusion holds for the weakly penalized variant used in our empirical
sections (a fixed weakly-informative prior on the item effects): for a fixed
smooth log-prior with locally bounded gradient, a score expansion around the
MLE, whose likelihood Hessian is of order $N$, gives
$\hat{\bm\eta}_{\textnormal{penalized}}-\hat{\bm\eta}_{\textnormal{MLE}}
=O_p(N^{-1})=o_p(N^{-1/2})$, so \eqref{eq:can} is unchanged.
\end{proposition}

\Cref{prop:glmm} states a standard likelihood result for generalized linear mixed models under cluster asymptotics ($N\to\infty$ with bounded cluster size); see \citet{jiang2021} and \citet{nie2007}.

\clearpage
\section{Proofs for Section~\ref{sec:leaderboard} (Williams designs)}
\label{app:williams}

We use the notation of \Cref{sec:leaderboard}: items $1,\dots,V$ with
$\theta_1>\cdots>\theta_V$, $\delta=\theta_1-\theta_2$, $N$ batches indexed by
$\pi_t\in S_V$, $P(v,t)=\pi_t^{-1}(v)$, naive estimator
$\phat_v=N^{-1}\sum_tY_{t,P(v,t)}$, contrast $\mathbf Z(\mathcal D)=(\phat_1-\phat_2,\dots,\phat_1-\phat_V)^\top$,
and the selection rule $\hat v=\argmax_v\phat_v$, $\Pr_{\mathcal D}(\hat v=1)=\Pr(\hat v=1)$.
Random Shuffle $\mathcal D_R$ draws the $\pi_t$ i.i.d.\ uniform on $S_V$; the Williams
design $\mathcal D_W$ stacks $N/V$ copies of a $V\times V$ Williams square ($V$ even, so that the square exists; this convention is in force for every Williams statement below) (items
assigned to its symbols by independent uniform bijections $\tau_r$, $r=1,\dots,N/V$) and
satisfies \emph{position balance} (W1: each item appears $N/V$ times in each position) and
\emph{first-order carryover balance} (W2: each ordered pair appears $N/V$ times as
consecutive). We assume at least one of $\beta,\gamma$ is nonzero (the case
$\beta=\gamma=0$ is degenerate).

\subsection[Additional definitions and interpretation]{Additional definitions and interpretation}
\label{supp:ranking-details}
The following details accompany the ranking results in \Cref{sec:lb-moments}.

We first note that all explicit constants below involve the serial
effect only through the sharp contraction coefficient
\[
  x\;:=\;\sup_z\big|\sigma(z+\beta)-\sigma(z)\big|\;=\;\tanh(|\beta|/4)\;<\;1
\]
(the supremum is attained at $z=-\beta/2$, where the difference equals
$\sigma(\beta/2)-\sigma(-\beta/2)=\tanh(\beta/4)$), together with the derived
series constants $\Lambda_1:=\sum_{k\ge0}(k+1)x^k=(1-x)^{-2}$ and
$\Lambda:=\sum_{k\ge0}(k+1)^4x^k$, finite for \emph{every} $\beta$ since $x<1$.
Because $\tanh t\le t$, the displayed bounds are stated in the simpler
$|\beta|$-form.

The two terms separate the two sources of randomness. The first is sampling: with
the design held fixed, the ratings themselves are random. The second is the design:
two runs of the same experiment draw different orderings, or different bijections,
so they place items in different slots and after different predecessors, and
therefore have different conditional means $\bm\mu(\mathcal D)$. The comparison
between the designs is a comparison of their design terms. Under the random shuffle
every batch redraws its ordering, and this randomness accumulates at the same rate as the sampling noise: $\bm\Delta_{\mathrm R}=\Theta(1/N)$. Under Williams,
within-block balance holds for \emph{every} bijection, so at equal qualities the
block-level conditional means do not depend on the draw at all and
$\bm\Delta_{\mathrm W}=\mathbf 0$ exactly; under the $N^{-1/2}$ local alternatives
studied below it is $O(1/N^2)$, one order smaller than the shuffle term. At a
fixed configuration with $\nu_{\mathrm W,k}>0$ (defined below) both design terms
are of order $1/N$, and the comparison is then the exact finite-sample identity
of \Cref{prop:lb-var}. The gap at and near equal qualities is the content of the
next result.

The comparison of the two design terms reduces to one scalar constant per design
and per contrast, which we now define. Let $S_V$ denote the set of permutations of
$\{1,\dots,V\}$, and $\mathrm{Unif}(S_V)$ the uniform distribution on $S_V$. For a
batch order $\pi\in S_V$, placing item $\pi(j)$ in slot $j$, let
\[
  M_v(\pi)\;:=\;\E\big[\,Y_{1,\pi^{-1}(v)}\;\big|\;V_{1,j}=\pi(j),\
  j=1,\dots,V\,\big],
  \qquad v=1,\dots,V,
\]
denote the expected rating of item $v$ in a single batch presented in order $\pi$,
the expectation taken over the batch effect and the ratings of the working model
(\Cref{ass:working}), batch $1$ serving as a generic batch. Let
$\rho_1,\dots,\rho_V\in S_V$ denote the rows of the Williams square, so that in a
block with item-to-symbol bijection $\tau$ the $r$-th batch presents the items in
order $\tau^{-1}\circ\rho_r$, and let
\[
  \bar m_v(\tau)\;:=\;\frac1V\sum_{r=1}^{V}M_v\big(\tau^{-1}\circ\rho_r\big)
\]
be the corresponding block mean of item $v$. The independent unit of the random
shuffle is the batch, with a fresh uniform order; that of the Williams design is
the block, with a fresh uniform bijection. The two designs' constants are exact
counterparts, the design variance of the $k$-th contrast over one independent
unit:
\[
  \nu_{\mathrm R,k}\;:=\;\Var_{\pi\sim\mathrm{Unif}(S_V)}
  \big[M_1(\pi)-M_k(\pi)\big],
  \qquad
  \nu_{\mathrm W,k}\;:=\;\Var_{\tau\sim\mathrm{Unif}(S_V)}
  \big[\bar m_1(\tau)-\bar m_k(\tau)\big].
\]
In this notation the diagonals of the design terms in \eqref{eq:lb-ltv} are
exactly $[\bm\Delta_{\mathrm R}]_{k-1,k-1}=\nu_{\mathrm R,k}/N$ over the $N$
batches, and $[\bm\Delta_{\mathrm W}]_{k-1,k-1}=V\nu_{\mathrm W,k}/N$ over the
$N/V$ blocks (see the proof of \Cref{prop:lb-var} below).

\Cref{ass:nondegen} asks only that the presentation order genuinely move the expected
ratings, so that position is worth balancing at all. If $\beta=\gamma=0$ the sequence
carries no information, every ordering is equivalent, and the two designs coincide, so
there is nothing to compare. In concrete terms, the condition says that in a batch of equal-quality items the expected rating is not the same in every slot, that is, a \emph{non-flat position profile}. Moreover, it is mild: \Cref{lem:a-c1null} shows
non-flatness holds automatically when exactly one of $\beta,\gamma$ is nonzero, and
when both are nonzero it fails only on a discrete set of parameters where the
cold-start and carryover effects exactly cancel; it is checkable numerically in any
given instance.

In general, it is intuitively expected that the dominance holds beyond the provable small-gap regime. Every
realization of $\tau$ places each item in each slot exactly once within the
block, so the first-order position effect cancels in
$\bar m_1(\tau)-\bar m_k(\tau)$ for every draw, and $\nu_{\mathrm W,k}$ is
generated only by the second-order interaction between quality differences and
position---quadratically small in the quality spread. The shuffle typically does not enjoy such
cancellation. At a flat configuration, the above result is the counterpart of the classical variance rationale for Latin-square and crossover designs, in which balancing rows and columns removes the period and carryover
nuisance variation that complete randomization leaves in the treatment contrasts
\citep{williams1949,jones2015crossover}.

The natural question is when the variance-reduction advantage of the Williams square design transfers into an advantage in the ultimate quantity of interest for the leaderboard task, the probability of correct selection. Note that Proposition~\ref{prop:lb-rank} already shows that in the standard setting, when the item qualities are fixed with sufficient separation, both designs are asymptotically correct. The regime where one could expect the difference is when the gap between the top items is small, while the sample size is not large. Therefore, next we focus on one local-alternative regime to gain deeper insights in this direction. Specifically, we consider the case $\theta_1-\theta_k=d_k/\sqrt N$, in which the leading candidates are close together. So we are evaluating the two designs at the
resolution where selection is neither trivial nor hopeless---precisely the
finite-sample instances in which the budget is moderate relative to the gaps. This is the standard calibration behind asymptotic relative efficiency \citep{vandervaart1998}.

Under either design and for every $N$, the expected win rate of item $v$ is the smooth function
$$\E[\phat_v]=F_v(\bm\theta):=\E_{\pi\sim\mathrm{Unif}(S_V)}\big[M_v(\pi)\big].$$

At a flat configuration $\bm\theta=\theta_*\bm 1$, exchangeability reduces its gradient to two numbers:
 the
\emph{self-effect} and the \emph{cross-effect} of quality on the win rate,
\[
  g_{\mathrm{self}}(\theta_*)\;:=\;
  \frac{\partial F_v}{\partial\theta_v}\bigg|_{\bm\theta=\theta_*\bm 1},
  \qquad
  g_{\mathrm{cross}}(\theta_*)\;:=\;
  \frac{\partial F_v}{\partial\theta_w}\bigg|_{\bm\theta=\theta_*\bm 1}
  \quad(w\ne v),
\]
neither depending on $v$ or $w$, and both functions of the common level
$\theta_*$.

\subsection{Preliminaries: the mixture CDF and the affine recurrence}

For slot $(t,j)$, $C_{t,j}=\beta Y_{t,j-1}+\gamma\one_{\{j=1\}}+u_t$, with marginal CDF
$F^{\mathcal D}_{t,j}$. By the tower property,
\begin{equation}
  \mu_v(\mathcal D)=\E[\phat_v\mid\mathcal D]=\int_\R\sigma(\theta_v+c)\,d\bar F_v^{\mathcal D}(c),
  \qquad \bar F_v^{\mathcal D}=\tfrac1N\sum_{t=1}^N F^{\mathcal D}_{t,P(v,t)},
  \label{eq:a-muv}
\end{equation}
the item-averaged context CDF. We label each slot of an item by its predecessor category
$a\in\mathcal V\cup\{\varnothing\}$ ($\varnothing$ at position $1$).

\begin{lemma}
\label{lem:a-slotinv}
$F^{\mathcal D}_{t,j}$ depends on $\mathcal D$ and $(t,j)$ only through (i) the predecessor
category $a$ at $(t,j)$, and (ii) the configuration $(\pi_t(1),\dots,\pi_t(j-2))$ of items
strictly before the predecessor.
\end{lemma}

\begin{proof}
By the recursive law \eqref{eq:bern}, $Y_{t,j-1}$ has success probability
$\sigma(\theta_{\pi_t(j-1)}+\beta Y_{t,j-2}+\gamma\one_{\{j-1=1\}}+u_t)$, depending on the
predecessor's identity (i) and the predecessor's predecessor outcome (ii).
\end{proof}

\begin{lemma}
\label{lem:a-multiset}
Under $\mathcal D_W$, every item $v$'s predecessor multiset across its $N$ slots is
\[
  \{\varnothing^{N/V}\}\cup\bigcup_{a\in\mathcal V\setminus\{v\}}\{a^{N/V}\},
\]
where the exponent is the multiplicity; in
particular it omits $v$ itself.
\end{lemma}

\begin{proof}
By (W1) exactly $N/V$ batches place $v$ at position $1$ ($N/V$ copies of $\varnothing$); by
(W2) exactly $N/V$ batches have each $a\ne v$ as the immediate predecessor of $v$. These
account for all $N/V+(V-1)N/V=N$ slots.
\end{proof}

Define the slot-uniform reference CDF $\bar F=\tfrac1{NV}\sum_{t,j}F^{\mathcal D_W}_{t,j}=\tfrac1V\sum_v\bar F_v^{\mathcal D_W}$
and $\bar g(\theta)=\int\sigma(\theta+c)\,d\bar F(c)$. For predecessor category
$a\in\mathcal V$, write $\widehat F_a$ for the \emph{probability-normalized}
average of the $(V-1)N/V$ slot laws whose predecessor is $a$, and
$\widehat F_a^{(v)}$ for the average of the $N/V$ such slots occupied by $v$; define
$\widehat F_\varnothing$ and $\widehat F_\varnothing^{(v)}$ analogously at
position $1$ (all four are probability measures). Counting slots
(\Cref{lem:a-multiset}),
\[
  \bar F_v^{\mathcal D_W}
  =\tfrac1V\widehat F_\varnothing^{(v)}+\tfrac1V\sum_{a\ne v}\widehat F_a^{(v)},
  \qquad
  \bar F=\tfrac1V\widehat F_\varnothing+\tfrac{V-1}{V^2}\sum_{a\in\mathcal V}\widehat F_a .
\]

\begin{lemma}
\label{lem:a-affine}
Fix $\pi$ and $u$; set $B_j=\sigma(\theta_{\pi(j)}+\gamma\one_{\{j=1\}}+u)$ and
$A_j=\sigma(\theta_{\pi(j)}+\beta+\gamma\one_{\{j=1\}}+u)-B_j$. With $Y_0\equiv0$,
$\Pr(Y_j=1\mid Y_{j-1},u,\pi)=B_j+A_jY_{j-1}$, so $p_j:=\Pr(Y_j=1\mid u,\pi)=B_j+A_jp_{j-1}$,
$|A_j|\le|\beta|/4$. Unrolling, $p_{j_v}=\sum_{k=0}^{j_v-1}(\prod_{m=j_v-k+1}^{j_v}A_m)B_{j_v-k}$,
and for $j_v<j_w$, $\Cov(Y_{j_v},Y_{j_w}\mid u,\pi)=(\prod_{m=j_v+1}^{j_w}A_m)\,p_{j_v}(1-p_{j_v})$.
\end{lemma}

\begin{proof}
The transition probability equals $B_j$ if $Y_{j-1}=0$ and $B_j+A_j$ if $Y_{j-1}=1$, i.e.\
$B_j+A_jY_{j-1}$. The recurrence follows by total probability, and $|A_j|\le|\beta|/4$ follows
from the mean value theorem.

We prove the covariance identity by induction on $j_w-j_v$:
\[
  \Cov(Y_{j_v},Y_{j_v+1}\mid u,\pi)=A_{j_v+1}\Var(Y_{j_v}\mid u,\pi)=A_{j_v+1}p_{j_v}(1-p_{j_v}),
\]
and each further step appends one $A$-factor.
\end{proof}

\subsection{Williams preserves expected ranks}

\begin{proposition}
\label{prop:a-tvbound}
For every $v$,
$|\mu_v(\mathcal D_W)-\bar g(\theta_v)|\le\varepsilon_W:=\big(1-V^{-2}\big)\tfrac{|\beta|}4\le\tfrac{|\beta|}4$,
independent of $\sigma_b^2,\bm\theta,N$.
\end{proposition}

\begin{proof}
Subtracting the two slot-count identities above gives
\[
  \bar F_v^{\mathcal D_W}-\bar F
  =\tfrac1V\big(\widehat F_\varnothing^{(v)}-\widehat F_\varnothing\big)
  +\tfrac1V\sum_{a\ne v}\big(\widehat F_a^{(v)}-\widehat F_a\big)
  +\tfrac{V-1}{V^2}\Big(\tfrac1{V-1}\textstyle\sum_{a\ne v}\widehat F_a-\widehat F_v\Big),
\]
where each bracket is a difference of two \emph{probability} measures. We integrate $\sigma(\theta_v+\cdot)$ against each
bracket.

\emph{Position one.} The context at position $1$ is $\gamma+u_t$ regardless of the
allocation, so $\widehat F_\varnothing^{(v)}=\widehat F_\varnothing$ exactly and
the first bracket contributes zero.

\emph{Equal-mass predecessor comparisons.} By \Cref{lem:a-affine}, for any
probability-normalized average $\widehat F$ of slot laws at positions $j\ge2$,
\[
  \int\sigma(\theta_v+c)\,d\widehat F(c)
  =\E_u\big[B^v(u)\big]+\E_u\big[A^v(u)\,\bar Q(u)\big],
\]
with $B^v=\sigma(\theta_v+u)$, $A^v=\sigma(\theta_v+\beta+u)-B^v$,
$|A^v|\le|\beta|/4$, and $\bar Q(u)\in[0,1]$ the corresponding average of
predecessor success probabilities. The $B^v$ term is common to every such
average because the measures have equal mass. It therefore cancels within each bracket,
leaving differences multiplied by $A^v$:
\[
  \Big|\int\sigma(\theta_v+c)\,d\big(\widehat F_a^{(v)}-\widehat F_a\big)\Big|
  \le\E_u|A^v|\le\tfrac{|\beta|}4,
  \qquad
  \Big|\int\sigma(\theta_v+c)\,
  d\Big(\tfrac1{V-1}\textstyle\sum_{a\ne v}\widehat F_a-\widehat F_v\Big)\Big|
  \le\tfrac{|\beta|}4 .
\]
Combining these bounds with the coefficients in the decomposition gives
\[
  |\mu_v(\mathcal D_W)-\bar g(\theta_v)|
  \;\le\;\Big(\tfrac{V-1}V+\tfrac{V-1}{V^2}\Big)\tfrac{|\beta|}4
  \;=\;\big(1-V^{-2}\big)\tfrac{|\beta|}4 .
  \qedhere
\]
\end{proof}

\begin{remark}[What drives $\varepsilon_W$]
\label{rem:a-twoscales}
Every contribution to $\varepsilon_W$ is modulated by the carryover contrast
$A^v$, so the bound vanishes as $|\beta|\to0$ and is uniform in $\sigma_b^2$ and
$\bm\theta$. Two mechanisms appear: the predecessor populations $v$ actually meets
differ from the global ones (the $\tfrac{V-1}V$ term), and $v$'s absence from its
own predecessor pool tilts the reference average (the missing-self term
$\tfrac{V-1}{V^2}$, one order smaller in $V$).
\end{remark}

\begin{remark}[Fixed-square rank preservation]
\label{rem:a-fixedrank}
For a fixed square, \Cref{prop:a-tvbound} also preserves the \emph{conditional}
expected ranks: $\bar g'(\theta)=\int\sigma'(\theta+c)\,d\bar F(c)>0$, so $\bar g$ is
strictly increasing, and
$\mu_1(\mathcal D_W)-\mu_k(\mathcal D_W)\ge\bar g(\theta_1)-\bar g(\theta_k)-2\varepsilon_W
\ge\bar g'_{\min}\delta-2\varepsilon_W>0$ whenever
$\delta>2\varepsilon_W/\bar g'_{\min}$, with
$\bar g'_{\min}=\inf_{[\theta_V,\theta_1]}\bar g'>0$. This conditional statement is not
needed by \Cref{prop:lb-rank}, whose annealed analysis follows.
\end{remark}

\medskip\noindent\textbf{Proof of Proposition~\ref{prop:lb-rank}.}
Strict monotonicity of $\pi$ is part of \Cref{lem:a-pi-deriv}:
$\pi'(\theta)=\E_C[\sigma'(\theta+C)]>0$.

For consistency, we first identify the limit under each design. Under $\mathcal D_R$
this is \Cref{lem:a-lln}: batches are i.i.d., so $\phat_v\xrightarrow{p}\pi_v
=\pi(\theta_v)+e_v$.

Under $\mathcal D_W$ batches are \emph{not} i.i.d., since the $V$
batches of a block share the item-to-symbol bijection $\tau_r$; the i.i.d.\ unit is the
block. Writing $\phat_v$ as the average of the $N/V$ i.i.d.\ block averages, each bounded
in $[0,1]$, the weak law applies at the block level and gives
$\phat_v\xrightarrow{p}\E[\phat_v(\mathcal D_W)]$, which by \Cref{lem:a-uniform}(a)
equals $\E[\phat_v(\mathcal D_R)]=\pi_v$. The two designs therefore share the same limit
as well as the same finite-$N$ mean.

\Cref{lem:a-ms-size} bounds $\max_v|e_v|\le|\beta|\Lambda_1/(2V)=O(|\beta|/V)$, uniformly
in $\bm\theta,\gamma,\sigma_b^2$. If $\pi(\theta_1)-\pi(\theta_2)>2\max_v|e_v|$ then, since
$\pi$ is strictly increasing and $\theta_2$ is the runner-up,
\[
  \begin{aligned}
    \pi_1-\pi_k&=\pi(\theta_1)-\pi(\theta_k)+e_1-e_k\\
    &\ge\pi(\theta_1)-\pi(\theta_2)-2\max_v|e_v|>0
  \end{aligned}
\]
for every $k\ne1$, so the limits are separated by a fixed margin. A union bound over the
$V-1$ contrasts then gives $\Pr(\argmax_v\phat_v=1)\to1$ under either design. \qed

\begin{remark}[Direction of the missing-self perturbation]
\label{rem:a-missingself}
When $\beta<0$ a heuristic favors the top item: item $1$ is the strongest possible
predecessor, and its own predecessor multiset excludes item $1$, so item $1$ ``dodges'' the
harshest carryover penalty that a weaker item still incurs. A rigorous sign theorem is left
open; we use only the magnitude bound.
\end{remark}

\subsection{The design variance of the Random Shuffle}

The following is the matrix-level statement behind \Cref{prop:lb-var}; the
main-text proposition is its diagonal instance, and its concluding display is the
local-alternative limiting form of the comparison, used in the proof of
\Cref{thm:lb-pcs}(ii).

\begin{propositionprime}
Under \Cref{ass:working,ass:design,ass:nondegen}, with the symbol-randomized Williams
design and both covariances taken unconditionally:
\begin{enumerate}
\item[\textnormal{(i)}] the random shuffle carries a strictly positive design term,
$\bm\Delta_{\mathrm R}\succeq\mathbf 0$ with \emph{exact}
diagonal entries $[\bm\Delta_{\mathrm R}]_{k-1,k-1}=\nu_{\mathrm R,k}/N>0$, so its
contrasts fluctuate at the $\Theta(1/N)$ rate for reasons of design alone;
\item[\textnormal{(ii)}] the two designs share every per-batch marginal moment, and
\[
  \Cov(\mathbf Z(\mathcal D_R))-\Cov(\mathbf Z(\mathcal D_W))
  \;=\;\bm\Delta_{\mathrm R}-\bm\Delta_{\mathrm W},
\]
where $\bm\Delta_{\mathrm W}\succeq\mathbf 0$ is the Williams design term, generated
by the symbol randomization, which vanishes
at the null and is $O(1/N^2)=o(1/N)$ under local alternatives.
\end{enumerate}
Consequently, along the local alternatives $\bm\theta_N=\theta_*\bm1+\bm t/\sqrt N$
of \Cref{thm:lb-pcs} and with \Cref{ass:nondegen} holding at the flat
configuration,
$\lim_NN\big([\Cov(\mathbf Z(\mathcal D_R))]_{k-1,k-1}-[\Cov(\mathbf Z(\mathcal D_W))]_{k-1,k-1}\big)=\nu^*_{\mathrm R,k}>0$
for every $k$, with no auxiliary condition on the magnitudes of $(\beta,\gamma)$:
the local-alternative limiting form of \Cref{prop:lb-var}.
\end{propositionprime}

\medskip\noindent\textbf{Proof of Proposition~\ref{prop:lb-var}$'$(i): the $\Theta(1/N)$ design variance.}
Under $\mathcal D_R$,
\[
  \mu_v(\mathcal D_R)=\frac1N\sum_{t=1}^NM_v(\pi_t),
  \qquad
  M_v(\pi)=\int\sigma(\theta_v+c)\,dF^{\mathcal D_R}_{1,\pi^{-1}(v)}(c)
\]
where batch $1$ represents a generic batch, since the context laws are common across
the i.i.d.\ batches. This integral form agrees with the conditional-expectation definition of $M_v$
in the main text. The summands $M_v(\pi_t)$ are i.i.d.\ in $t$ and lie in
$[0,1]$, so
\[
  \Var_{\mathcal D_R}\big[\mu_a(\mathcal D_R)-\mu_b(\mathcal D_R)\big]
  \;=\;\frac1N\,\Var_\pi\big[M_a(\pi)-M_b(\pi)\big].
\]

The map $\pi\mapsto M_a(\pi)-M_b(\pi)$ is non-constant. At the null configuration,
\Cref{lem:a-c1null} establishes this unconditionally when exactly one of
$\beta,\gamma$ is nonzero and outside a discrete exceptional parameter set otherwise.
Non-constancy persists nearby by continuity.

Hence
$\nu_{\mathrm R}(a,b):=\Var_\pi[M_a(\pi)-M_b(\pi)]>0$, matching the main-text
notation via $\nu_{\mathrm R,k}=\nu_{\mathrm R}(1,k)$, and boundedness gives the
upper constant $\overline\nu_{\mathrm R}(a,b)\le1$, i.e.\
$\nu_{\mathrm R}/N\le\Var\le\overline\nu_{\mathrm R}/N$.

For a fixed gap $\delta>0$ this design variance vanishes, in line with the consistency
of \Cref{prop:lb-rank} (whose proof runs through the law of large numbers above). In the
local-alternative regime $\delta=O(N^{-1/2})$, by contrast, the contrast has mean of order $\delta$ and
standard deviation of order $N^{-1/2}$, so by the CLT
\[
  \Pr(\mu_a(\mathcal D_R)<\mu_b(\mathcal D_R))\to\Phi\big(-(g_{\mathrm{self}}(\theta_*)-g_{\mathrm{cross}}(\theta_*))\,d\big/\sqrt{\nu^*_{\mathrm R}(a,b)}\big)>0
\]
when the design variance at the flat configuration is positive;
this is the regime \Cref{thm:lb-pcs} analyzes. \qed

\subsection{Second-moment analysis: the rest of Proposition~\ref{prop:lb-var}$'$}

For a fixed design, conditional independence of batches gives
$\Cov(\mathbf Z\mid\mathcal D)=N^{-2}\sum_tW(\pi_t)$, where
$W(\pi)_{k-1,l-1}=\Cov(Y_{P(1)}-Y_{P(k)},Y_{P(1)}-Y_{P(l)}\mid\pi)$, $|W(\pi)_{k-1,l-1}|\le1$.
The multivariate Law of Total Variance \eqref{eq:lb-ltv} reads, for the random
shuffle,
\begin{equation}
  \Cov(\mathbf Z(\mathcal D_R))=\E_{\mathcal D_R}[\Cov(\mathbf Z\mid\mathcal D_R)]+\bm\Delta_{\mathrm R},\quad
  \bm\Delta_{\mathrm R}:=\Cov_{\mathcal D_R}(\bm\mu(\mathcal D_R)),
  \label{eq:a-ltv}
\end{equation}
and $\E_{\mathcal D_R}[\Cov(\mathbf Z\mid\mathcal D_R)]=\tfrac1N\overline W_{\mathrm{unif}}$
with $\overline W_{\mathrm{unif}}:=\tfrac1{V!}\sum_\pi W(\pi)$.

\begin{lemma}
\label{lem:a-uniform}
Fix any row $\rho\in S_V$ of the Williams square. If $\tau\sim\mathrm{Unif}(S_V)$,
the realized batch permutation $\tau^{-1}\circ\rho$ is uniform on $S_V$: inversion
preserves the uniform law on the group and right-multiplication by $\rho$ is a
bijection. Hence under $\mathcal D_W$ with per-block symbol randomization every
single batch's permutation has exactly the uniform marginal law of a Random-Shuffle
batch (batches within a block are dependent through the shared $\tau_r$), and, for
every finite $N$:
\textnormal{(a)} $\E[\phat_v(\mathcal D_W)]=\E[\phat_v(\mathcal D_R)]
=\E_{\pi\sim\mathrm{Unif}}[M_v(\pi)]$ for every $v$, so
$\E[\mathbf Z(\mathcal D_W)]=\E[\mathbf Z(\mathcal D_R)]$;
\textnormal{(b)} $\E_{\bm\tau}[\Cov(\mathbf Z(\mathcal D_W)\mid\bm\tau)]
=\tfrac1N\overline W_{\mathrm{unif}}
=\E_{\mathcal D_R}[\Cov(\mathbf Z(\mathcal D_R)\mid\mathcal D_R)]$.
\end{lemma}

\begin{proof}
For part (a), the conditional expectation of a rating given the batch's realized
permutation $\pi$ is the same function $M_v(\pi)$ under either design (the
within-batch law given $\pi$ does not depend on how $\pi$ was generated). Taking
expectations batch by batch under the uniform marginal law proves the claim.

For part (b), batches are independent conditional on the design, so
\[
  \Cov(\mathbf Z\mid\bm\tau)\;=\;\frac1{N^2}\sum_{t=1}^NW\big(\pi_t(\bm\tau)\big);
\]
the result follows by taking $\E_{\bm\tau}$ term by term under the uniform marginal law of each
$\pi_t(\bm\tau)$.
\end{proof}

\emph{Part (i).} Positive semidefiniteness of $\bm\Delta_{\mathrm R}$ follows
because it is a covariance matrix, and its $(k{-}1,k{-}1)$ entry is \emph{exactly}
$\Var_{\mathcal D_R}[\mu_1(\mathcal D_R)-\mu_k(\mathcal D_R)]=\nu_{\mathrm R,k}/N$, with
$\nu_{\mathrm R,k}>0$ by \Cref{lem:a-c1null}: both facts were established in the design-variance
computation above. Summing the diagonal entries gives
$\mathrm{tr}(\bm\Delta_{\mathrm R})\ge(V-1)\nu_{\mathrm R}^*/N$,
$\nu_{\mathrm R}^*=\min_k\nu_{\mathrm R,k}$.

\emph{Part (ii) (the Williams design term $\bm\Delta_{\mathrm W}$).} With
$\bm\tau=(\tau_1,\dots,\tau_K)$, write
$\bm\Delta_{\mathrm W}:=\Cov_{\bm\tau}(\E[\mathbf Z(\mathcal D_W)\mid\bm\tau])\succeq\mathbf 0$. By
the block structure,
\[
  \Var_{\bm\tau}\big(\E[\phat_v\mid\bm\tau]\big)
  \;=\;\frac{V^2}{N^2}\sum_{r=1}^{N/V}\Var_{\tau_r}\big(\bar m_v(\tau_r)\big),
  \qquad
  \bar m_v(\tau)\;:=\;\frac1V\sum_{i=1}^{V}M_v\big(\tau^{-1}\circ\rho_i\big),
\]
the average of $N/V$ i.i.d.\ block expectations, where $\rho_1,\dots,\rho_V$ are the
rows of the square and $\tau^{-1}\circ\rho_i$ the block's realized batch orders
(\Cref{lem:a-uniform}).

At the global null, each block expectation equals $\bar g(\theta_*)$ by (W1)--(W2).
Thus,
$\Var_{\tau_r}(\bar m_v)=0$ and $\bm\Delta_{\mathrm W}=\mathbf 0$ there.

Under local alternatives
a Taylor expansion gives $\Var_{\tau_r}(\bar m_v)=O(1/N)$, hence
\[
  \Var_{\bm\tau}\big(\E[\phat_v\mid\bm\tau]\big)
  \;=\;\frac1K\,O(1/N)\;=\;\frac VN\,O(1/N)\;=\;O(1/N^2)\;=\;o(1/N),
\]
and the same holds entry-wise for the contrasts. By contrast, Random Shuffle
contributes
$\Var_{\bm\pi}(\E[\phat_v\mid\bm\pi])=N^{-1}\Var_\pi[M_v(\pi)]=\Theta(1/N)$.

\emph{The concluding display.} Applying the Law of Total Variance to both designs and subtracting,
using \Cref{lem:a-uniform}(b) for the exact cancellation of the average conditional
covariances,
\[
  \Cov(\mathbf Z(\mathcal D_R))-\Cov(\mathbf Z(\mathcal D_W))
  =\bm\Delta_{\mathrm R}-\bm\Delta_{\mathrm W} ,
\]
with no residual term; the diagonal limit
$\lim_NN\big([\cdot]_{k-1,k-1}\big)=\nu_{\mathrm R,k}>0$ follows from parts (i)--(ii); this
is the concluding display of the proposition. \qed

\begin{lemma}
\label{lem:a-c1null}
Fix $V\ge4$. At a global null configuration $\bm\theta=\theta_*\one$,
$M_v(\pi)=m(\pi^{-1}(v))$ depends on $\pi$ only through $v$'s slot, where
$m(j):=\E_u[p_j(u)]$ is the homogeneous-batch slot mean, and
$\nu_{\mathrm R}(a,b)=\Var_\pi[m(j_a)-m(j_b)]>0$ for every $a\ne b$ iff $m$ is non-constant on
$\{1,\dots,V\}$. Moreover: \textnormal{(i)} if $\gamma\ne0$, $\beta=0$:
$m(1)\ne m(2)$ always; \textnormal{(ii)} if $\beta\ne0$, $\gamma=0$:
$m(3)-m(2)=\E_u[A(u)^2p_1(u)]>0$ always, with
$A(u):=\sigma(\theta_*+\beta+u)-\sigma(\theta_*+u)$ of strict sign and
$p_1(u)=\sigma(\theta_*+u)>0$; \textnormal{(iii)} if both are nonzero: $m$ is
non-constant for every $\gamma$ outside a set that is discrete for each fixed
$(\beta,\theta_*,\sigma_b^2)$. Under the local-alternative limit the same
positivity transfers to $\lim_N\nu_{\mathrm R,k}$; it is checkable numerically in any
instance (e.g.\ $\nu_{\mathrm R}\approx1.0\times10^{-3}$ at $\beta=\gamma=-0.5$,
$\sigma_b^2=0.45$, $V=6$, by exact enumeration).
\end{lemma}

\begin{proof}
At the null all slots $j\ge2$ share $B(u):=\sigma(\theta_*+u)$ and $A(u)$ as their
affine-recurrence coefficients (\Cref{lem:a-affine}), so
$p_{j+1}-p_j=A\,(p_j-p_{j-1})$ for $j\ge2$, whence
\[
  m(j+1)-m(j)=\E_u\big[A^{\,j-1}(p_2-p_1)\big]\quad(j\ge2),
  \qquad
  m(2)-m(1)=\E_u[p_2-p_1].
\]

\emph{(i)} With $\beta=0$: $p_1=\sigma(\theta_*+\gamma+u)$,
$p_2=\sigma(\theta_*+u)$, and $\gamma\ne0$ gives $m(1)\ne m(2)$ by strict
monotonicity of $\sigma$.

\emph{(ii)} With $\gamma=0$: $p_1=B$ and $p_2=B+Ap_1$, so $p_2-p_1=Ap_1$ and
\[
  m(3)-m(2)=\E_u\big[A(u)^2\,p_1(u)\big]>0
\]
where the inequality is strict and $V\ge4$ ensures that slot~3 is present.

\emph{(iii)} With both nonzero, $\phi(\gamma):=(m(2)-m(1))^2+(m(3)-m(2))^2$ is
real-analytic in $\gamma$ at fixed $(\beta,\theta_*,\sigma_b^2)$ (Gaussian
integrals of compositions of analytic functions; differentiation under the
integral is justified by the uniform bounds on $\sigma$ and its derivatives), and
$\phi(0)=(\E_u[Ap_1])^2+(\E_u[A^2p_1])^2>0$ by the computation of case (ii); hence
$\phi\not\equiv0$ and its zero set is discrete.

In each case $m$ is non-constant, which gives positive variance since
$(j_a,j_b)$ ranges over all ordered distinct slot pairs with positive probability.
\end{proof}

\subsection{The finite-sample identity: proof of Proposition~\ref{prop:lb-var}}

\begin{proof}[Proof of \Cref{prop:lb-var}]
\emph{The identity.} By \Cref{lem:a-uniform}(b) the sampling terms of
\eqref{eq:lb-ltv} agree exactly at every $N$, so by
Proposition~\ref{prop:lb-var}$'$(ii),
\[
  [\Cov(\mathbf Z(\mathcal D_R))]_{k-1,k-1}
  -[\Cov(\mathbf Z(\mathcal D_W))]_{k-1,k-1}
  \;=\;[\bm\Delta_{\mathrm R}]_{k-1,k-1}-[\bm\Delta_{\mathrm W}]_{k-1,k-1}
\]
for every $N$. The batches of $\mathcal D_R$ are i.i.d., so
$\bm\mu(\mathcal D_R)$ is the average of $N$ i.i.d.\ per-batch mean vectors and
$[\bm\Delta_{\mathrm R}]_{k-1,k-1}=\Var_\pi[M_1(\pi)-M_k(\pi)]/N=\nu_{\mathrm R,k}/N$,
which is Part~(i).

The blocks of $\mathcal D_W$ are i.i.d., so, exactly as in the
display of Part~(ii),
\[
  [\bm\Delta_{\mathrm W}]_{k-1,k-1}=\frac{V}{N}\Var_\tau[\bar m_1(\tau)-\bar m_k(\tau)]
=V\nu_{\mathrm W,k}/N
\]
over the $N/V$ blocks. Subtracting and multiplying by $N$ gives the identity,
whose right-hand side is free of $N$.

\emph{Null vanishing.} At a flat configuration $\bm\theta=\theta_*\bm 1$, every
block mean equals $\bar g(\theta_*)$ for every $\tau$, by the within-block balance
(W1)--(W2)---the same computation as in the null case of Part~(ii)---so
$\bar m_1(\tau)-\bar m_k(\tau)\equiv0$ and $\nu_{\mathrm W,k}=0$.

\emph{Continuity and $\delta_0$.} For each fixed permutation $\pi$, the map
$\bm\theta\mapsto M_v(\pi)$ is continuous: the within-batch chain recursion is
a finite composition of $C^\infty$ functions of $(\bm\theta,u)$ bounded by $1$,
and dominated convergence allows us to pass limits through $\E_u$. Hence $\nu_{\mathrm R,k}$, a
variance over the finite set $S_V$, and $\nu_{\mathrm W,k}$, a variance over the finite set of
bijections, are continuous in $\bm\theta$.

At the flat configuration the
right-hand side equals $\nu_{\mathrm R,k}>0$ by \Cref{lem:a-c1null} (this is
\Cref{ass:nondegen} there); continuity yields $\delta_0>0$ with
$\nu_{\mathrm R,k}-V\nu_{\mathrm W,k}>0$ throughout $\max_v|\theta_v-\theta_*|\le\delta_0$, uniformly
in $k$ since there are finitely many contrasts.
\end{proof}

\subsection{Worst-case bounds for a fixed (unrandomized) square}

For a single fixed square used without symbol randomization,
$\Cov(\mathbf Z\mid\mathcal D_W^{\mathrm{fix}})=\tfrac1N\overline W_{\mathrm{Williams}}$ with
$\overline W_{\mathrm{Williams}}:=\tfrac1V\sum_{r=1}^VW(\rho_r)$, and the comparison
acquires the residual $\bm E^{\mathrm{fix}}:=\tfrac1N(\overline W_{\mathrm{unif}}-\overline
W_{\mathrm{Williams}})$ in place of $-\bm\Delta_{\mathrm W}$. We bound it by worst-case range arguments.
Expanding $W(\pi)_{k-1,l-1}=V_1-C_{1,l}-C_{1,k}+(C_{k,l}\text{ or }V_k)$ with
$V_v=\Var(Y_{P(v)}\mid\pi)$, $C_{v,w}=\Cov(Y_{P(v)},Y_{P(w)}\mid\pi)$, we use range bounds.

\begin{lemma}
\label{lem:a-rangep}
For each $v$, $\sup_\pi p_v(\pi)-\inf_\pi p_v(\pi)\le L(|\beta|+|\gamma|)$, $L=\tfrac14$, and
$\sup_\pi V_v(\pi)-\inf_\pi V_v(\pi)\le L(|\beta|+|\gamma|)$.
\end{lemma}

\begin{proof}
Conditioning on $u$ and the predecessor success $q\in[0,1]$,
\[
  p_v(\pi\mid u)=(1-q)\,\sigma(\theta_v+\gamma\one_{\{j=1\}}+u)
  +q\,\sigma(\theta_v+\beta+\gamma\one_{\{j=1\}}+u)
\]
is a convex combination of two sigmoid values differing by $\le L|\beta|$, with the
cold-start indicator contributing $\le L|\gamma|$. Thus the range is
$\le L(|\beta|+|\gamma|)$, and the bound is preserved under $\E_u$.

Since $p\mapsto p(1-p)$ is $1$-Lipschitz, the same bound applies to $V_v$.
\end{proof}

\begin{lemma}
\label{lem:a-rangeW}
For all $k,l$, $\sup_\pi W(\pi)_{k-1,l-1}-\inf_\pi W(\pi)_{k-1,l-1}\le\kappa_{\mathrm{cov}}:=\tfrac{13}4(|\beta|+|\gamma|)$.
\end{lemma}

\begin{proof}
By \Cref{lem:a-rangep} each variance term has range $\le(|\beta|+|\gamma|)/4$.

For covariances, $C_{v,w}=q_{v,w}-p_vp_w$ with
$q_{v,w}=\E[Y_{P(v)}Y_{P(w)}\mid\pi]$.

Condition on $u$ and let $j_v<j_w$ (the
other case is symmetric): by the chain structure,
$q_{v,w}(\pi\mid u)=p_v(\pi\mid u)\cdot r_{v,w}(\pi\mid u)$ with
$r_{v,w}:=\Pr(Y_{j_w}=1\mid Y_{j_v}=1,u,\pi)\in[0,1]$; the forward chain started
from $Y_{j_v}=1$ obeys the same affine recurrence, so, exactly as for
$p_v(\pi\mid u)$ in the previous display, $r_{v,w}(\pi\mid u)$ has range
$\le L(|\beta|+|\gamma|)$ over $\pi$, uniformly in $u$.

A product of two $[0,1]$
factors, each of range $\le L(|\beta|+|\gamma|)$, has range
$\le2L(|\beta|+|\gamma|)$, and this is preserved by $\E_u$; the same argument
bounds the range of $p_vp_w$. Hence $C_{v,w}$ has range
$\le|\beta|+|\gamma|$.

Combining these bounds with the contrast expansion of $W(\pi)$ gives
\[
  k=l:\quad W=V_1+V_k-2C_{1,k},\qquad
  \text{range}\;\le\;2\cdot\tfrac{|\beta|+|\gamma|}4+2(|\beta|+|\gamma|)
  =\tfrac52(|\beta|+|\gamma|);
\]
\[
  k\ne l:\quad W=V_1-C_{1,k}-C_{1,l}+C_{k,l},\qquad
  \text{range}\;\le\;\tfrac{|\beta|+|\gamma|}4+3(|\beta|+|\gamma|)
  =\tfrac{13}4(|\beta|+|\gamma|).
\]
The maximum is $\kappa_{\mathrm{cov}}$.
\end{proof}

Two averages of a function of entry-wise range $\kappa_{\mathrm{cov}}$ over different measures
on $S_V$ differ by $\le\kappa_{\mathrm{cov}}$, so
$\|\bm E^{\mathrm{fix}}\|_{\max}\le\kappa_{\mathrm{cov}}/N$; the quadratic refinement below
sharpens this to $\widetilde\kappa\,|\beta|(|\beta|+|\gamma|)/N$.

\begin{remark}[These constants cannot certify fixed-square dominance]
\label{rem:a-kappahonest}
Since each $M_v(\pi)$ has range at most $(|\beta|+|\gamma|)/4$ over $\pi$
(\Cref{lem:a-rangep}'s argument applied to the slot expectation), necessarily
$\nu_{\mathrm R}\le(|\beta|+|\gamma|)^2/16$. Hence ``$\nu_{\mathrm R}>\kappa_{\mathrm{cov}}=\tfrac{13}4(|\beta|+|\gamma|)$''
would require $|\beta|+|\gamma|>52$ (impossible), and even
``$\nu_{\mathrm R}>\widetilde\kappa|\beta|(|\beta|+|\gamma|)$'' with the constant actually proven
below fails in essentially every regime of interest (numerically
$\nu_{\mathrm R}\approx10^{-3}$ against bounds of order at least $3$ at $\beta=\gamma=-0.5$,
$V=6$). Fixed-square dominance is therefore numerically checkable---both
$\overline W_{\mathrm{unif}}$ and $\overline W_{\mathrm{Williams}}$ are exactly
computable---but not certified by these worst-case bounds; the symbol-randomized
design of \Cref{prop:lb-var} needs none of this machinery.
\end{remark}

\subsection{Quadratic refinement via the exact binary recurrence}

\begin{lemma}
\label{lem:a-c1quad}
Fix $V\ge4$ and $\bm\theta$ with $\theta_a>\theta_b$. As $(\beta,\gamma)\to(0,0)$,
$\nu_{\mathrm R}(a,b)=\Var_\pi[M_a(\pi)-M_b(\pi)]=Q(\beta,\gamma)+O((\beta^2+\gamma^2)^{3/2})$
for a \emph{positive-definite} quadratic form $Q$, whose entries depend on
$(\bm\theta,\sigma_b^2)$. In particular
$\nu_{\mathrm R}(a,b)=\Theta(\beta^2+\gamma^2)$ as $(\beta,\gamma)\to(0,0)$,
with constants depending on the configuration.
\end{lemma}

\begin{proof}
At $\beta=\gamma=0$, $C_{t,j}=u_t$ and $M_v(\pi)=\E_u[\sigma(\theta_v+u)]$ is
$\pi$-independent, so $\nu_{\mathrm R}=0$. A first-order Taylor expansion gives
\[
  M_v(\pi)=M_v^{(0)}+\beta M_v^{(\beta)}(\pi)
  +\gamma M_v^{(\gamma)}(\pi)+O(\beta^2+\gamma^2),
\]
with
\begin{align*}
  M_v^{(\gamma)}(\pi)&=\E_u[\sigma'(\theta_v+u)]\,\one_{\{\pi^{-1}(v)=1\}},\\
  M_v^{(\beta)}(\pi)&=\E_u\big[\sigma'(\theta_v+u)\,\sigma(\theta_{\pi(\pi^{-1}(v)-1)}+u)\big]\,\one_{\{\pi^{-1}(v)\ge2\}}.
\end{align*}
Thus $\Var_\pi[M_a-M_b]=Q(\beta,\gamma)+O((\beta^2+\gamma^2)^{3/2})$, where
$Q(\beta,\gamma)$ is the variance of the random vector
$\bm X(\pi)=(M_a^{(\beta)}-M_b^{(\beta)},\,M_a^{(\gamma)}-M_b^{(\gamma)})$
contracted with $(\beta,\gamma)$.

$Q(\beta,\gamma)=\Var_\pi[\beta X_1(\pi)+\gamma X_2(\pi)]$ is positive
semidefinite by construction, and it is positive definite if and only if the
support of $(X_1,X_2)$ under $\pi\sim\mathrm{Unif}(S_V)$ is not contained in an
affine line. Write
\[
  I_i:=\E_u[\sigma'(\theta_i+u)]>0,
  \qquad
  J_{ij}:=\E_u[\sigma'(\theta_i+u)\,\sigma(\theta_j+u)];
\]
for fixed $i$, $J_{ij}$ is strictly increasing in $\theta_j$, since $\sigma$ is
strictly increasing and $\sigma'>0$. Pick any third item $c$ ($V\ge4$).

\emph{Case $\theta_c\ne\theta_a$.} Orders beginning $(a,b,c)$ and $(a,c,b)$
give $(X_1,X_2)=(-J_{ba},\,I_a)$ and $(-J_{bc},\,I_a)$: two distinct points
with the same second coordinate. Any order beginning with $b$ has second
coordinate $-I_b<0<I_a$, supplying a third, non-collinear point.

\emph{Case $\theta_c=\theta_a$.} Then $\theta_c\ne\theta_b$, and orders
beginning $(b,a,c)$ and $(b,c,a)$ give the distinct points $(J_{ab},-I_b)$ and
$(J_{ac},-I_b)$, while any order beginning with $a$ has second coordinate
$I_a\ne-I_b$ and supplies the non-collinear third point.

In either case, appending the remaining items in any fixed order produces
permutations of positive probability, so the support of $(X_1,X_2)$ contains
three non-collinear points, no nonzero combination $\beta X_1+\gamma X_2$ is
$\pi$-a.s.\ constant, and $Q$ is positive definite. The upper bound
$O(\beta^2+\gamma^2)$ follows from boundedness of the expansion coefficients.
\end{proof}

\begin{proposition}
\label{prop:a-kappaquad}
For $|\beta|\le2$,
$\|\overline W_{\mathrm{unif}}-\overline W_{\mathrm{Williams}}\|_{\max}\le\widetilde\kappa\,|\beta|(|\beta|+|\gamma|)$
with the explicit, deliberately conservative constant $\widetilde\kappa=7$,
independent of $V,N,\sigma_b^2,\bm\theta$.
\end{proposition}

\begin{proof}
Let $\Delta_\pi[f]=\E_\pi^{\mathrm U}[f]-\E_\pi^{\mathrm W}[f]$. Decompose
\[
  p_{j_v}=T_0+T_1+R,\qquad
  T_0=B_{j_v},\qquad
  T_1=A_{j_v}B_{j_v-1}\one_{\{j_v\ge2\}},
\]
with $R$ collecting the terms with $\ge2$ carryover factors (\Cref{lem:a-affine}).
We bound the three terms using four observations.

\emph{(i)} $T_0$ depends on $\pi$ only through $\one_{\{j_v=1\}}$, matched by (W1):
$\Delta_\pi[T_0]=0$.

\emph{(ii)} $T_1|_{\gamma=0}$ is a function of the ordered pair $(\pi(j_v-1),v)$,
matched by (W2): $\Delta_\pi[T_1|_{\gamma=0}]=0$.

\emph{(iii)} The $\gamma$-correction to $T_1$ is nonzero only at $j_v=2$ and
bounded by $|\beta\gamma|/16$, so
$|\Delta_\pi[T_1-T_1|_{\gamma=0}]|\le|\beta\gamma|/8$.

\emph{(iv)} $|R|\le\sum_{k\ge2}(|\beta|/4)^k\le\beta^2/8$ for $|\beta|\le2$, so
$|\Delta_\pi[R]|\le\beta^2/4$.

Hence
\[
  |\Delta_\pi[p_v]|\;\le\;\frac{|\beta\gamma|}8+\frac{\beta^2}4
  \;\le\;\frac{|\beta|(|\beta|+|\gamma|)}4 .
\]
Write $q:=|\beta|(|\beta|+|\gamma|)$ throughout; recall
$x=\tanh(|\beta|/4)\le|\beta|/4\le\tfrac12$.
We next apply the same decomposition to $\bar p_v^2$, with $\bar p_v:=\E_u[p_{j_v}]$.
The $\bar T_0^2$ term is matched by (W1). The $2\bar T_0\bar T_1|_{\gamma=0}$ term is, at $\gamma=0$, a function of the ordered pair $(\pi(j_v-1),v)$ alone, since the cold-start flag enters only through $\gamma$. It is therefore matched by (W2).

The $\gamma$-corrections total $\le x|\gamma|$, and every remaining
term carries at least two factors of $A$, totaling $\le7x^2$ for $|\beta|\le2$.
Consequently,
\[
  \begin{aligned}
    |\Delta_\pi[V_v]|&=|\Delta_\pi[\bar p_v-\bar p_v^2]|\\
    &\le\tfrac q4+x|\gamma|+7x^2\le2q
  \end{aligned}.
\]

For covariances, decompose by the law of total covariance,
\[
  C_{v,w}(\pi)=\E_u\big[\Cov(Y_{j_v},Y_{j_w}\mid u,\pi)\big]
  +\Cov_u\big(p_{j_v},p_{j_w}\big),
\]
and treat the two pieces separately.

\emph{Conditional piece.} By \Cref{lem:a-affine},
\[
  \Cov(Y_{j_v},Y_{j_w}\mid u,\pi)=(\prod_{m=j_v+1}^{j_w}A_m)\,p_{j_v}(1-p_{j_v}).
\]
For adjacent slots ($j_w=j_v+1$), replace $p_{j_v}$ by $B_{j_v}$: since
$|p_{j_v}-B_{j_v}|\le x$ and $p\mapsto p(1-p)$ is $1$-Lipschitz, the substitution
costs $\le\E_u|A|\,x\le x^2$; the substituted term
$\E_u[A_{j_v+1}B_{j_v}(1-B_{j_v})]$ is, at $\gamma=0$, a function of the ordered
pair $(\pi(j_v),\pi(j_v+1))$, matched by (W2), with $\gamma$-correction (nonzero only at
$j_v=1$) $\le x|\gamma|/4$. Hence the adjacent conditional piece has
$|\Delta_\pi[\cdot]|\le2(x^2+x|\gamma|/4)\le q/4$.

For non-adjacent slots the
product carries at least two factors of $A$, so the piece is $\le x^2/4$ in
absolute value and $|\Delta_\pi[\cdot]|\le x^2/2\le q/32$.

\emph{Batch-effect piece.} Write $p_j=B_j+A_jp_{j-1}$ and expand
$\Cov_u(p_{j_v},p_{j_w})=\Cov_u(B^{(v)},B^{(w)})+\text{cross terms}$, where
$B^{(v)}(u):=\sigma(\theta_v+\gamma\one_{\{j_v=1\}}+u)$ depends on $\pi$ only
through $v$'s cold-start flag. The joint law of the two flags is identical under
the uniform and the square-averaged order (each item is first with probability
$1/V$, and two items are never simultaneously first), so
$\Delta_\pi[\Cov_u(B^{(v)},B^{(w)})]=0$ exactly.

Every cross term contains a factor of $A$. The leading term
$\Cov_u(A_{j_v}B_{j_v-1},B^{(w)})$ is, at $\gamma=0$, a function of the ordered pair $(\pi(j_v-1),v)$ alone, since the flag of $w$ enters $B^{(w)}$ only through $\gamma$, and is therefore matched by (W2).

The $\gamma$-corrections and the terms with two or more $A$-factors are bounded, using
$\mathrm{sd}_u(\cdot)\le\tfrac12$ for every bounded factor and $|\beta|\le2$, by
$x|\gamma|+6x^2\le\tfrac{5q}8$ in total. Hence
$|\Delta_\pi[\Cov_u(p_{j_v},p_{j_w})]|\le\tfrac{5q}8$ and, combining the two pieces,
\[
  |\Delta_\pi[C_{v,w}]|\le\tfrac q4+\tfrac q{32}+\tfrac{5q}8\le\tfrac{3q}2.
\]

\emph{Assembling the entries.} For the diagonal entries,
\[
  |\Delta_\pi[V_1+V_k-2C_{1,k}]|\le2q+2q+2\cdot\tfrac{3q}2=7q.
\]
 For the off-diagonal entries,
\[
  |\Delta_\pi[V_1-C_{1,k}-C_{1,l}+C_{k,l}]|\le2q+3\cdot\tfrac{3q}2\le7q.
\]
 Hence
$\widetilde\kappa=7$ suffices.
\end{proof}

By \Cref{lem:a-c1quad,prop:a-kappaquad}, in the small-bias regime $\nu_{\mathrm R}=\Theta(\beta^2+\gamma^2)$ and
the fixed-square discrepancy bound $\widetilde\kappa|\beta|(|\beta|+|\gamma|)$ are of the same
\emph{order}; \Cref{rem:a-kappahonest} records why the proven constants nonetheless leave
fixed-square dominance uncertified. The refinement is retained for its structural
content---which terms the balance properties (W1)--(W2) match exactly---and because none of
it is needed for the randomized design.

\subsection{Probability of correct selection: proof of Theorem~\ref{thm:lb-pcs}(ii)}

We work in the local-alternative regime $\theta_{1,N}-\theta_{k,N}=d_k/\sqrt N$ ($\min_kd_k=d>0$),
with $(\beta_N,\gamma_N,\sigma_{b,N}^2,\theta_{*,N})\to(\beta,\gamma,\sigma_b^2,\theta_*)$.

\begin{lemma}
\label{lem:a-shift}
With $\xi_v(\bm\theta)=\mu_v(\mathcal D_W)-\bar g(\theta_v\mid\bm\theta)$,
$\sqrt N\,\xi_v(\bm\theta_N)\to\xi_v^{(1)}(\theta_*):=-\sum_{w\ge2}\partial_{\theta_w}\xi_v|_{\theta_*\one}d_w$,
the Taylor remainder being $O(N^{-1})$ uniformly over the bounded local
perturbations, by the uniform second-derivative bounds of \Cref{lem:a-pi-deriv}
applied slotwise and the finiteness of the permutation average.
Here $\mu_v(\mathcal D_W)$ is the conditional mean for a \emph{fixed} square.
Under the symbol randomization the two designs' \emph{means} agree exactly
(\Cref{lem:a-uniform}(a)); the reference $\bar g$ depends on the realized
square through its slot-uniform context law, and averaging it as well gives
$\E_\tau\xi_v=F_v(\bm\theta)-\E_\tau[\bar g_\tau(\theta_v\mid\bm\theta)]
=F_v(\bm\theta)-\pi(\theta_v)=e_v(\bm\theta)$, the missing-self correction,
which is not zero in general. This lemma is used only in the fixed-square
discussion of \Cref{rem:a-fixedshift}.
\end{lemma}

\begin{proof}
At any finite-$N$ null $\bm\theta=\theta_{*,N}\one$ all items are interchangeable,
so by (W1) $\bar F_v^{\mathcal D_W}=\bar F$ there and $\xi_v(\theta_{*,N}\one)=0$.

Taylor expanding $\xi_v$ about $\theta_{*,N}\one$ and substituting
$\theta_{w,N}-\theta_{*,N}=-d_w/\sqrt N$ (exactly, for $w\ge2$; $0$ for $w=1$)
gives
\[
  \xi_v(\bm\theta_N)=-\frac1{\sqrt N}\sum_{w\ge2}
  \partial_{\theta_w}\xi_v\big|_{\theta_{*,N}\one}\,d_w+O(1/N);
\]
multiplying by $\sqrt N$ and letting $N\to\infty$ gives the result, since the derivatives converge by continuity.
\end{proof}

\begin{lemma}
\label{lem:a-BE}
Fix $k$, with $V\ge4$, $(\beta,\gamma)\ne(0,0)$, and the non-degeneracy
$\nu_{\mathrm R,k}>0$ of \Cref{lem:a-c1null}. Set $L_k:=(g_{\mathrm{self}}(\theta_*)-g_{\mathrm{cross}}(\theta_*))d_k$, the limit
in \Cref{thm:lb-pcs}(i), which part (i) below proves; for (iii) assume additionally
$g_{\mathrm{self}}(\theta_*)>g_{\mathrm{cross}}(\theta_*)$ (the hypothesis of
\Cref{thm:lb-pcs}(ii)), so that $L_k>0$. Set
$\sigma_{k,W}^2=\lim_NN\Var(Z_k(\mathcal D_W))$,
$\sigma_{k,R}^2=\lim_NN\Var(Z_k(\mathcal D_R))$ (unconditional). Then:
\textnormal{(i)} $\E[Z_k(\mathcal D_W)]=\E[Z_k(\mathcal D_R)]$ for every $N$, and
$\sqrt N\,\E[Z_k(\mathcal D)]\to L_k$ for both designs;
\textnormal{(ii)} $\sigma_{k,W}^2=\sigma_{k,R}^2-\nu_{\mathrm R,k}<\sigma_{k,R}^2$;
\textnormal{(iii)} $\lim_N\Pr(Z_k(\mathcal D)>0)=\Phi(L_k/\sigma_{k,\mathcal D})$
for both designs, with
$\Phi(L_k/\sigma_{k,W})>\Phi(L_k/\sigma_{k,R})$ strictly.
\end{lemma}

\begin{proof}
\emph{(i) Means: proof of \Cref{thm:lb-pcs}(i).} By \Cref{lem:a-uniform}(a) the
two designs share their means for every $N$, so it suffices to treat the shuffle,
under which $\E[\phat_v]=F_v(\bm\theta):=\E_\pi[M_v(\pi)]$ is a function of
$\bm\theta$ alone: i.i.d.\ batches contribute no $N$-dependence.

Each
$M_v(\pi)$ is $C^\infty$ in $\bm\theta$ (a finite composition of smooth maps
of $(\bm\theta,u)$, with differentiation under $\E_u$ justified by dominated
convergence), so $F_v\in C^1$.

Uniformity of $\pi$ makes the family equivariant
under relabeling: $F_{\varsigma(v)}(\varsigma\cdot\bm\theta)=F_v(\bm\theta)$ for
every permutation $\varsigma$, where
$(\varsigma\cdot\bm\theta)_{\varsigma(w)}=\theta_w$. Differentiating at the flat
configuration $\theta_*\bm 1$, which every $\varsigma$ fixes, forces
$\partial F_v/\partial\theta_v$ to be free of $v$ and
$\partial F_v/\partial\theta_w$ ($w\ne v$) to be free of $(v,w)$: these are
$g_{\mathrm{self}}(\theta_*)$ and $g_{\mathrm{cross}}(\theta_*)$.

With
$\theta_v=\theta_*+t_v/\sqrt N$, a
first-order expansion of the $C^1$ function $F_1-F_k$ at $\theta_*\bm 1$ gives
\[
  \sqrt N\,\E[Z_k(\mathcal D)]
  =\sum_{v=1}^V\Big(\frac{\partial F_1}{\partial\theta_v}
  -\frac{\partial F_k}{\partial\theta_v}\Big)\Big|_{\theta_*\bm 1}\,t_v+o(1)
  =\big(g_{\mathrm{self}}(\theta_*)-g_{\mathrm{cross}}(\theta_*)\big)(t_1-t_k)+o(1)
  \;\longrightarrow\;L_k,
\]
where every $t_v$ with $v\notin\{1,k\}$ has coefficient
$g_{\mathrm{cross}}(\theta_*)-g_{\mathrm{cross}}(\theta_*)=0$, which
also cancels the background level $\theta_*$.

\emph{(ii) Variances.} Proposition~\ref{prop:lb-var}$'$ gives
\[
  N\Var(Z_k(\mathcal D_R))-N\Var(Z_k(\mathcal D_W))
  \;=\;\nu_{\mathrm R,k}-N[\bm\Delta_{\mathrm W}]_{k-1,k-1}\;\to\;\nu_{\mathrm R,k}\;>\;0,
\]
with the limits existing by smoothness of the variance functionals in the base
parameters.

\emph{(iii) Distributional limits.} Under $\mathcal D_R$ the per-batch contrasts
are i.i.d.\ and bounded. Under $\mathcal D_W$ the i.i.d.\ unit is the
\emph{block} of $V$ batches sharing one symbol bijection $\tau_r$, so decompose
$Z_k(\mathcal D_W)=\frac{V}{N}\sum_{r=1}^{N/V}U_r$ into $N/V$ i.i.d.\ bounded block
averages and apply the Berry--Esseen theorem to the block sums (batches within a
block are dependent, so the batch-level i.i.d.\ argument does not
apply to $\mathcal D_W$ directly).

In either case, the standardized contrast is
Gaussian up to $O(N^{-1/2})$, giving the $\Phi$-limits; the strict ordering
follows from $L_k>0$ (the part-(ii) hypothesis of \Cref{thm:lb-pcs}, via (i)),
(ii), and monotonicity of $\Phi$.
\end{proof}

\begin{remark}[Fixed squares: the surviving shift-tolerance condition]
\label{rem:a-fixedshift}
For a fixed square the two designs' means differ by the structural shift of
\Cref{lem:a-shift}, and the comparison instead requires
$L^{\mathrm{fix}}_{k,W}/\sigma_{k,W}>L_k/\sigma_{k,R}$ with
$L^{\mathrm{fix}}_{k,W}:=g_{\mathrm{self}}(\theta_*)d_k+\xi_1^{(1)}(\theta_*)-\xi_k^{(1)}(\theta_*)$.
This is not automatic: $\xi^{(1)}=O(|\beta|+|\gamma|)$ can dominate the
$\Theta(\beta^2+\gamma^2)$ variance gain, and its sign is construction-dependent
(cf.\ \Cref{rem:a-missingself}). Moreover, even when the standardized-mean
comparison holds pairwise, an exact joint selection comparison for a fixed
square would additionally require control of its limiting correlation matrix;
pairwise inequalities alone do not establish it. The case-by-case condition is
confined to the
unrandomized design; the symbol-randomized main line does not need it.
\end{remark}

\begin{theorem}
\label{thm:a-union}
In the local-alternative regime, if the hypotheses of \Cref{lem:a-BE}---the
non-degeneracy $\nu_{\mathrm R,k}>0$ and the selection-consistency condition $L_k>0$---hold
for every $k$, then the union-bound on the probability of incorrect selection,
$\mathrm{PIC}^{\mathrm{UB}}(\mathcal D):=\sum_{k=2}^V\Pr(Z_k(\mathcal D)\le0)$, satisfies
$\liminf_N[\mathrm{PIC}^{\mathrm{UB}}(\mathcal D_R)-\mathrm{PIC}^{\mathrm{UB}}(\mathcal D_W)]>0$.
\end{theorem}

\begin{proof}
By \Cref{lem:a-BE}, 
\[
  \begin{aligned}
    \lim_N\Pr(Z_k(\mathcal D_W)\le0)&=\Phi(-L_k/\sigma_{k,W})\\
    &<\Phi(-L_k/\sigma_{k,R})\\
    &=\lim_N\Pr(Z_k(\mathcal D_R)\le0)
  \end{aligned}
\]
for each $k$; summing the $V-1$ strict inequalities gives the claim.
\end{proof}

The union bound is upgraded to an exact comparison of correct-selection
probabilities by identifying the joint limiting law of the contrasts under
each design, in the triangular-array sense of the local regime, and comparing
the two limits under one common Gaussian vector.

\begin{lemma}
\label{lem:a-jointclt}
In the local-alternative regime, for each design
$\mathcal D\in\{\mathcal D_R,\mathcal D_W\}$:

\textnormal{(i)} $N\,\Cov(\mathbf Z(\mathcal D))=\bm\Omega_{\mathcal D}(\bm\theta)$
for a matrix function free of $N$---under $\mathcal D_R$ the per-batch covariance
of the contrast vector, under $\mathcal D_W$ the per-block covariance times
$V$---with entries continuous in $\bm\theta$; hence
$N\Cov(\mathbf Z(\mathcal D))\to\bm\Sigma_{\mathcal D}:=\bm\Omega_{\mathcal D}(\theta_*\one)$
entrywise.
\textnormal{(ii)} $\bm\Sigma_{\mathcal D}$ is positive definite, with correlation
matrix $\tfrac12(\mathbf I+\mathbf J)$ and diagonal
$\sigma^2_{k,\mathcal D}>0$.
\textnormal{(iii)} Jointly in the $V-1$ coordinates,
$\sqrt N\,\big(\mathbf Z(\mathcal D)-\E\mathbf Z(\mathcal D)\big)
\xrightarrow{d}\N(\bm0,\bm\Sigma_{\mathcal D})$.
\end{lemma}

\begin{proof}
\emph{(i).} Under $\mathcal D_R$, $\mathbf Z$ is the average of $N$ i.i.d.\
bounded per-batch contrast vectors, so $N\Cov(\mathbf Z)$ equals the per-batch
covariance
\[
  \bm\Omega_R(\bm\theta)=\E_\pi[W(\pi)]+\Cov_\pi\big((M_1(\pi)-M_k(\pi))_k\big),
\]
free of $N$; under $\mathcal D_W$, $\mathbf Z$ is the average of $N/V$ i.i.d.\
bounded per-block contrast vectors, so
$N\Cov(\mathbf Z)=V\cdot\Cov(\text{block contrast})=:\bm\Omega_W(\bm\theta)$.

Every entry is a finite combination of expectations of $C^\infty$ functions of
$(\bm\theta,u)$ over finitely many permutations, continuous in $\bm\theta$ by
dominated convergence; the local regime sends $\bm\theta_N\to\theta_*\one$, giving
the entrywise limits.

\emph{(ii).} At $\theta_*\one$ the law of $(\phat_v)_v$ is item-exchangeable under
either design, so the limiting contrast correlation matrix is intraclass,
$\tfrac12(\mathbf I+\mathbf J)$, with eigenvalues $\tfrac12$ (multiplicity $V-2$)
and $\tfrac V2$, both positive (\Cref{cor:a-exch}).

For the diagonal, by
\Cref{lem:a-uniform}(b) the sampling part of $\sigma^2_{k,\mathcal D}$ is
\[
  [\overline W_{\mathrm{unif}}]_{k-1,k-1}
=\E_\pi[\Var(Y_{P(1)}-Y_{P(k)}\mid\pi)];
\]
 for every $\pi$ each conditional success
probability lies strictly in $(0,1)$, so the batch's rating vector has full
support on $\{0,1\}^V$ given $\pi$, the contrast is non-degenerate, and
$[\overline W_{\mathrm{unif}}]_{k-1,k-1}>0$. The design parts are nonnegative,
whence $\sigma^2_{k,\mathcal D}>0$ and $\bm\Sigma_{\mathcal D}\succ\bm0$.

\emph{(iii).} Fix $\bm a\in\R^{V-1}$. Under $\mathcal D_R$ the summands
$\bm a^\top(\text{batch contrast}-\text{its mean})$ form, for each $N$, an i.i.d.\
triangular array bounded by $2\|\bm a\|_1$, with variances
\[
  \bm a^\top\bm\Omega_R(\bm\theta_N)\bm a\to\bm a^\top\bm\Sigma_R\bm a
\]
 by (i);
boundedness verifies the Lindeberg condition, and the Lindeberg--Feller
CLT applies. The Cram\'er--Wold device gives the joint limit. Under
$\mathcal D_W$, we apply the same argument to the $N/V$ i.i.d.\ bounded block
contrast vectors.
\end{proof}

By \Cref{lem:a-jointclt} and $\sqrt N\,\E[Z_k(\mathcal D)]\to L_k$
(\Cref{lem:a-BE}(i), with $L_k>0$ under the part-(ii) hypothesis),
$\sqrt N\,\mathbf Z(\mathcal D)\xrightarrow{d}\N(\bm L,\bm\Sigma_{\mathcal D})$
jointly.

Ties among the win rates have positive probability at any finite $N$,
but with ties broken by any fixed rule,
\[
  \begin{aligned}
    \Pr(\sqrt N\,Z_k>0\ \forall k)&\le\Pr_{\mathcal D}(\hat v=1)\\
    &\le\Pr(\sqrt N\,Z_k\ge0\ \forall k)
  \end{aligned};
\]
 the limiting Gaussian is nondegenerate, so
its orthant boundary carries no mass, both bounds converge to the same limit,
and
\[
  \Pr_{\mathcal D}(\hat v=1)
  \;\longrightarrow\;
  \Pr\big(\tilde G_k>-\tau_k(\mathcal D)\ \forall k\big),
  \qquad
  \tau_k(\mathcal D):=L_k/\sigma_{k,\mathcal D},
\]
with $\tilde{\mathbf G}$ standard Gaussian with correlation matrix
$\mathbf R_\infty(\mathcal D)$.

Two facts drive the comparison:
\textbf{(C1)} \emph{threshold dominance},
$\tau_k(\mathcal D_W)>\tau_k(\mathcal D_R)$ for every $k$, since the numerators
$L_k>0$ coincide and $\sigma_{k,W}<\sigma_{k,R}$ (\Cref{lem:a-BE}(i)--(ii)); and
\textbf{(C2)} \emph{equality of the limiting correlations},
$\mathbf R_\infty(\mathcal D_W)=\mathbf R_\infty(\mathcal D_R)
=\tfrac12(\mathbf I+\mathbf J)$ (\Cref{cor:a-exch}, whose null computation applies
to the local-regime limits by the covariance continuity of
\Cref{lem:a-jointclt}(i)).

\medskip\noindent\textbf{Proof of Theorem~\ref{thm:lb-pcs}(ii).}
Let $A:=\Pr\big(\tilde G_k>-\tau_k(\mathcal D_R)\ \forall k\big)$ under the
common correlation matrix of (C2).

\emph{Step 1 (a common Gaussian).} By (C2) the two limiting standardized laws
coincide: both selection probabilities converge to orthant probabilities of the
\emph{same} Gaussian vector $\tilde{\mathbf G}$, evaluated at the thresholds
$-\tau_k(\mathcal D_R)$ and $-\tau_k(\mathcal D_W)$ respectively. In particular
$\lim_N\Pr_{\mathcal D_R}(\hat v=1)=A$.

\emph{Step 2 (thresholds).} By (C1) the Williams orthant strictly contains the
shuffle orthant, and the difference contains the nonempty open box
$\prod_k\big(-\tau_k(\mathcal D_W),-\tau_k(\mathcal D_R)\big)$, which carries
positive probability because $\tilde{\mathbf G}$ is nondegenerate with full
support (\Cref{lem:a-jointclt}(ii)). Hence
$\lim_N\Pr_{\mathcal D_W}(\hat v=1)>A$.

Combining the two steps,
\[
  \begin{aligned}
    &\liminf_N[\Pr_{\mathcal D_W}(\hat v=1)-\Pr_{\mathcal D_R}(\hat v=1)]\\
    &\quad=\lim_N\Pr_{\mathcal D_W}(\hat v=1)-A>0
  \end{aligned},
\]
 which is \eqref{eq:lb-pcs}. \qed

\begin{corollary}
\label{cor:a-exch}
Under symbol randomization in $\mathcal D_W$, at the global null $\bm\theta=\theta_*\one$ the law of
$(\phat_1,\dots,\phat_V)$ is exchangeable under either design, so the limiting covariance is intraclass,
$\bm\Sigma^{(\infty)}=(a-c)\mathbf I+c\mathbf J$ ($a>c$, design-specific), and the limiting contrast
correlation is $\tfrac12$ for both designs: $\mathbf R_\infty(\mathcal D_W)=\mathbf R_\infty(\mathcal D_R)=\tfrac12(\mathbf I+\mathbf J)$.
Hence \textnormal{(C2)} holds with exact equality in the limit, and \eqref{eq:lb-pcs} is driven purely by
the threshold effect \textnormal{(C1)}.
\end{corollary}

\begin{proof}
At the null all items are interchangeable, so by symmetry of
$\tau_r\sim\mathrm{Unif}(S_V)$ the joint law of $(\phat_1,\dots,\phat_V)$ is
invariant under item permutations and its covariance is intraclass.

For $Z_k=\phat_1-\phat_k$ with $\bm\Sigma=(a-c)\mathbf I+c\mathbf J$,
\[
  \Var(Z_k)=2(a-c),
  \qquad
  \Cov(Z_j,Z_k)=a-c,
\]
so the contrast correlation is $\tfrac12$ for both designs; the limiting
correlation matrices therefore coincide.

The variance-vanishing argument of Proposition~\ref{prop:lb-var}$'$(ii) keeps
$\bm\Sigma_W$ unconditional, so the comparison reduces to (C1).
\end{proof}

\begin{remark}[Why standardization, not PSD-monotonicity]
\label{rem:a-slepian}
Deducing $\Pr_{\mathcal D_W}(\hat v=1)>\Pr_{\mathcal D_R}(\hat v=1)$ from $\Cov(\mathbf Z(\mathcal D_R))\succeq\bm\Sigma_W$
via PSD-monotonicity of orthant probabilities fails in general. The route above
succeeds because it standardizes first: once the limiting correlation matrices
are shown to be \emph{equal} (\Cref{cor:a-exch}), the two limits are orthant
probabilities of one common Gaussian vector at nested thresholds, and the
comparison is driven by the threshold effect (C1) alone---no Gaussian comparison
inequality is needed.
\end{remark}

\clearpage
\section{Proofs for Section~\ref{sec:convexity} (the nonlinearity bias)}
\label{app:convexity}

Throughout, $\sigma(x)=(1+e^{-x})^{-1}$ and $\pi(\theta)=\E_C[\sigma(\theta+C)]$ is the
marginal win probability of \eqref{eq:pi}, $C\sim F_C$ the generic contextual shift
of \Cref{sec:working-model}. Derivatives of $\pi$ are taken with respect to its scalar
argument, with the reference law $F_C$, which depends on the whole quality panel, held fixed. Recall $\sigma'=\sigma(1-\sigma)$,
$\sigma''=\sigma(1-\sigma)(1-2\sigma)$, $\sigma'''=\sigma(1-\sigma)(1-6\sigma+6\sigma^2)$,
and the group quantities $\bar\pi_G=V_G^{-1}\sum_{v\in G}\pi(\theta_v)$,
$m_{k,G}=V_G^{-1}\sum_{v\in G}(\theta_v-\bar\theta_G)^k$, and $\rho_G=\max_{v\in G}|\theta_v-c_0|$ with $c_0$ the expansion point of the result at hand ($c_0=\mu$ under $H_0$, $c_0=\bar\mu$ under $H_1$).

\subsection{Regularity of the marginal win probability}

\begin{lemma}
\label{lem:a-pi-deriv}
$\pi:\R\to(0,1)$ is infinitely differentiable, with
$\pi'(\theta)=\E_C[\sigma'(\theta+C)]\in(0,\tfrac14]$,
$\pi''(\theta)=\E_C[\sigma''(\theta+C)]$, $\pi'''(\theta)=\E_C[\sigma'''(\theta+C)]$,
and $M_3:=\sup_x|\pi'''(x)|\le\tfrac18$, $M_4:=\sup_x|\pi^{(4)}(x)|<0.13$. Moreover
$\sigma''$ is odd: $\sigma''(-x)=-\sigma''(x)$.
\end{lemma}

\begin{proof}
Fix $\theta$ and set
\[
  \Delta_h(c)\;=\;\frac{\sigma(\theta+h+c)-\sigma(\theta+c)}{h}.
\]
By the mean value theorem $|\Delta_h(c)|=|\sigma'(\xi)|\le\tfrac14$ uniformly in $c$
and $h$ near $0$; since $\tfrac14$ is $F_C$-integrable and
$\Delta_h(c)\to\sigma'(\theta+c)$ pointwise, dominated convergence gives
$\pi'(\theta)=\E_C[\sigma'(\theta+C)]$.

Applying the same argument to $\sigma'$ and $\sigma''$ gives the formulas for
$\pi''$ and $\pi'''$, using the sharp bounds
\[
  |\sigma''|\;\le\;\tfrac{\sqrt3}{18}\;<\;0.0963,
  \qquad
  |\sigma'''|\;\le\;\max_{p\in[0,1]}p(1-p)|1-6p+6p^2|\;=\;\tfrac18
\]
(the latter attained at $p=\tfrac12$). Higher derivatives are bounded in the same way,
so $\pi\in C^\infty$ and $M_3\le\tfrac18$; for the fourth derivative,
$|\sigma^{(4)}|=p(1-p)|1-2p|\,|1-12p+12p^2|$ with $p=\sigma(x)$ has maximum $0.1277\ldots<0.13$ over $p\in[0,1]$, so $M_4<0.13$. Positivity and the
$\tfrac14$ bound on $\pi'$ follow from $\sigma'\in(0,\tfrac14]$.

Finally, using $\sigma(-x)=1-\sigma(x)$,
\[
  \sigma''(-x)\;=\;(1-\sigma(x))\,\sigma(x)\,(2\sigma(x)-1)\;=\;-\sigma''(x).
  \qedhere
\]
\end{proof}

\subsection{Consistency under the complete-block shuffle, and the missing-self correction}
\label{app:missing-self}

For each position $j$ let $F_C^{(j)}$ denote the \emph{unconditional} law of the
contextual shift $C_{1,j}$ (integrating over $u_1$, the ordering, and the chain),
and let $F_C:=V^{-1}\sum_{j=1}^VF_C^{(j)}$ be the position-averaged marginal used
in \eqref{eq:pi}. Let $F_C^{(j\mid v)}$ be the law of $C_{1,j}$ conditional on
$V_{1,j}=v$. Under the complete-block shuffle the conditional prefix is a uniform
ordered subset of the \emph{other} $V-1$ items---it can never contain $v$---so the
item-conditional average $\bar F_v:=V^{-1}\sum_jF_C^{(j\mid v)}$ differs from $F_C$
by an $N$-free \emph{missing-self} perturbation. The quantity item $v$'s win rate
actually estimates is
\begin{equation}
  \pi_v \;:=\; \int_{\R}\sigma(\theta_v+c)\,d\bar F_v(c)
  \;=\; \pi(\theta_v)+e_v,
  \qquad
  e_v \;:=\; \int_{\R}\sigma(\theta_v+c)\,d(\bar F_v-F_C)(c).
  \label{eq:a-pi-v}
\end{equation}

\begin{lemma}
\label{lem:a-lln}
Under \Cref{ass:working,ass:design}, $\phat_v\xrightarrow{p}\pi_v$ as $N\to\infty$, for
every $v$.
\end{lemma}

\begin{proof}
Conditional on $(\theta_v,C_{tj})$, $Y_{tj}$ is Bernoulli with mean
$\sigma(\theta_v+C_{tj})$. By the tower property,
\[
  \E[Y_{tj}\mid V_{tj}=v]\;=\;\int\sigma(\theta_v+c)\,dF_C^{(j\mid v)}(c),
\]
and averaging over the uniform position of $v$ within its batch gives per-batch
mean $\pi_v$.

We now average over batches, using the independent-unit clause of
\Cref{ass:design} to apply the law of large numbers. Under clause (a) the
batches are i.i.d.\ blocks of bounded length $V$ (the within-batch lag dependence
is irrelevant to the block sums), and the weak law for i.i.d.\ bounded variables
gives $\phat_v\xrightarrow{p}\pi_v$.

Under clause (b) the i.i.d.\ unit is the
block of $B$ batches: $\phat_v$ is the average of $N/B$ i.i.d.\ block averages,
each bounded in $[0,1]$ with common mean $\pi_v$ (every batch's marginal position
law is uniform, by item-exchangeability), and the same conclusion follows.

Exchangeability without the independent-unit clause would not suffice: a single
random order reused in every batch is exchangeable, yet gives $\phat_v$ a random
limit.
\end{proof}

Throughout this subsection set
$x:=\sup_z|\sigma(z+\beta)-\sigma(z)|=\tanh(|\beta|/4)\in[0,1)$ (the supremum is
attained at $z=-\beta/2$; note $x\le|\beta|/4$ since $\tanh t\le t$),
$\Lambda_1:=\sum_{k\ge0}(k+1)x^k=(1-x)^{-2}$, and
$\Lambda:=\sum_{k\ge0}(k+1)^4x^k=(1+11x+11x^2+x^3)(1-x)^{-5}$; both are finite for
\emph{every} $\beta$ since $x<1$, with $\Lambda_1\le4$ and $\Lambda\le300$ for
$|\beta|\le2$.

Two
structural devices drive all bounds. \emph{(1) The unrolled chain.} Fix a batch, a
prefix ordering, and $u$; with $B_j:=\sigma(\theta_{\text{item at }j}+\gamma\one_{\{j=1\}}+u)$
and $A_j:=\sigma(\theta_{\text{item at }j}+\beta+\gamma\one_{\{j=1\}}+u)-B_j$, the
success probability at slot $j-1$ unrolls \emph{exactly} (each transition being
affine in the binary previous rating) as
\begin{equation}
  q_{j-1}=\sum_{k=0}^{j-2}T_k,\qquad
  T_k:=\Big(\prod_{m=j-k}^{j-1}A_m\Big)B_{j-1-k},\qquad |T_k|\le x^k,
  \label{eq:a-unroll}
\end{equation}
where $T_k$ depends on the prefix only through the items in its \emph{window} of
$k+1$ slots.

\emph{(2) The replacement coupling.} Let $W$ be a uniform ordered
$m$-arrangement of $\{1,\dots,V\}$ and let $W'$ replace the occurrence of $v$ in
$W$ (if any) by a uniform item absent from $W$. A direct count shows $W'$ is
exactly uniform on arrangements avoiding $v$ (each receives probability
$[1+\tfrac m{V-m}]\tfrac{(V-m)!}{V!}=\tfrac{(V-1-m)!}{(V-1)!}$), and
$W=W'$ off $\{v\in W\}$, so for any window of $k+1$ slots,
$\Pr(W,W'\textnormal{ differ on the window})\le(k+1)/V$.

\begin{lemma}
\label{lem:a-ms-size}
Under \Cref{ass:working,ass:design}, $|e_v|\le|\beta|\Lambda_1/(2V)$ for every $v$,
uniformly in $\bm\theta,\gamma,\sigma_b^2$; in particular $|e_v|\le2|\beta|/V$ for
$|\beta|\le2$.
\end{lemma}

\begin{proof}
At $j=1$ the context is $\gamma+u$ regardless of the allocation, so position $1$
contributes nothing to $e_v$. For $j\ge2$, conditioning on $u$ and the prefix and
writing $A^{\theta}(u):=\sigma(\theta+\beta+u)-\sigma(\theta+u)$ (so
$|A^{\theta}(u)|\le x$),
\[
  e_v=\frac1V\sum_{j=2}^{V}\E_u\big[A^{\theta_v}(u)\big(\bar q_j^{(v)}(u)-\bar q_j(u)\big)\big],
\]
with $\bar q_j^{(v)},\bar q_j$ the expectations of $q_{j-1}$ under the $v$-excluded
and unconditional prefix laws.

Under the replacement coupling applied to \eqref{eq:a-unroll}, the term $T_k$ changes
between $W$ and $W'$ only when $v$ lies in its window. Thus, uniformly in $u$ and
$j$,
\[
  \big|\bar q_j^{(v)}(u)-\bar q_j(u)\big|
  \;\le\;\sum_{k\ge0}2x^k\,\frac{k+1}{V}\;=\;\frac{2\Lambda_1}{V}.
\]
Hence
\[
  |e_v|\;\le\;\frac{V-1}{V}\cdot x\cdot\frac{2\Lambda_1}{V}
  \;\le\;\frac{|\beta|\Lambda_1}{2V}
  \qquad(\text{using }x\le|\beta|/4).
  \qedhere
\]
\end{proof}

\begin{lemma}
\label{lem:a-ms-deriv}
View $e_v=E_v(\bm\theta)$ as a function of $\bm\theta\in\R^V$; it is smooth
(each $\bm\theta$-derivative of the finite window sums below is bounded
uniformly in $u$). Fix $u$ and a position $j\ge2$, and for $1\le m\le j-1$
write the window term of length $m$ (the $k=m-1$ term of \eqref{eq:a-unroll})
as $H_{j,m}(W)$, a product of one $B$-factor and $m-1$ $A$-factors over the
ordered window $W$ of the items at slots $j-m,\dots,j-1$, so
$|H_{j,m}|\le x^{m-1}$. With $(n)_m:=n(n-1)\cdots(n-m+1)$,
\[
  \E H_{j,m}=\frac1{(V)_m}\sum_WH_{j,m}(W),
  \qquad
  \E^{-v}H_{j,m}=\frac1{(V-1)_m}\sum_{W:\,v\notin W}H_{j,m}(W),
\]
the sums running over ordered $m$-tuples of distinct items (avoiding $v$ in the
second), and the missing-self correction is \emph{exactly}
\begin{equation}
  e_v(\bm\theta)
  =\frac1V\sum_{j=2}^{V}\E_u\Big[A^{\theta_v}(u)\sum_{m=1}^{j-1}D_{j,m}(u)\Big],
  \qquad
  D_{j,m}:=\E^{-v}H_{j,m}-\E H_{j,m}.
  \label{eq:a-window}
\end{equation}
For all $v$ and all distinct $w,w'\notin\{v\}$, uniformly in
$\bm\theta,\gamma,\sigma_b^2$, for every $V\ge4$:
\begin{gather*}
  |\partial_ve_v|\le\frac{3x\Lambda}{V},\qquad
  |\partial_we_v|\le\frac{3x\Lambda}{V^2},\qquad
  |\partial^2_{vv}e_v|\le\frac{5x\Lambda}{V},\\
  |\partial^2_{vw}e_v|\le\frac{4x\Lambda}{V^2},\qquad
  |\partial^2_{ww}e_v|\le\frac{3x\Lambda}{V^2},\qquad
  |\partial^2_{ww'}e_v|\le\frac{9x\Lambda}{V^3}.
\end{gather*}
\end{lemma}

\begin{proof}
\emph{Step 1 (representation).} Position $1$ has the same context under both
laws and contributes nothing. For $j\ge2$, conditioning on $u$, the current
item's success probability is $\sigma(\theta_v+u)+A^{\theta_v}(u)\,q_{j-1}$, and
\eqref{eq:a-unroll} expands $q_{j-1}$ into window terms; the generic law places
a uniform ordered $m$-tuple from all $V$ items in each window, the
$v$-conditional law a uniform ordered $m$-tuple avoiding $v$. This gives
\eqref{eq:a-window}. The weights $1/(V)_m$ and $1/(V-1)_m$ are free of
$\bm\theta$, so every derivative acts only on the factors of $H_{j,m}$.

\emph{Step 2 (exchange identity).} Splitting the unconditional sum according to whether
$v\in W$ gives
\begin{equation}
  D_{j,m}
  =\frac mV\Big\{\E\big[H_{j,m}\mid v\notin W\big]-\E\big[H_{j,m}\mid v\in W\big]\Big\},
  \label{eq:a-exchange}
\end{equation}
because $\Pr(v\in W)=m/V$ and, conditional on $v\notin W$, the
window law coincides with the $v$-excluded law. All conditional laws are
uniform on their supports, so the inclusion probabilities used below are exact
falling-factorial ratios, and joint inclusion of specified labels is always
computed \emph{jointly}, never as a product of marginals.

\emph{Step 3 (factor derivatives).} For $A(z)=\sigma(z+\beta)-\sigma(z)$, the
oriented-integral representation $A^{(q)}(z)=\int_z^{z+\beta}\sigma^{(q+1)}(t)\,dt$
together with the pointwise bounds $|\sigma''|\le\sigma'$ and
$|\sigma'''|\le\sigma'$ (the latter since
$\sigma'''=\sigma'(1-6\sigma+6\sigma^2)$ and $|1-6p+6p^2|\le1$ on $[0,1]$) gives
\[
  |A^{(q)}(z)|\;\le\;|A(z)|\;\le\;x,\qquad q=0,1,2 .
\]
The $B$-factor and its first two derivatives are bounded by $1$. A label
occupies at most one window slot, so a specified first- or second-order
derivative pattern differentiates specified factors with no product-rule
multiplicity, and every differentiated $H_{j,m}$ remains bounded by $x^{m-1}$.

\emph{Step 4 (window derivative bounds).} Uniformly in $j,u,\bm\theta$, for
distinct $w,w'\ne v$ and $1\le m\le V-1$:
\begin{align*}
  |D_{j,m}|&\le\frac{2m}{V}\,x^{m-1},\\
  |\partial_vD_{j,m}|,\ |\partial^2_{vv}D_{j,m}|&\le\frac mV\,x^{m-1},\\
  |\partial_wD_{j,m}|,\ |\partial^2_{ww}D_{j,m}|&\le\frac{m(2m-1)}{V(V-1)}\,x^{m-1}
  \;\le\;\frac{3m^2}{V^2}\,x^{m-1},\\
  |\partial^2_{vw}D_{j,m}|&\le\frac{(m)_2}{(V)_2}\,x^{m-1}\;\le\;\frac{m^2}{V^2}\,x^{m-1},\\
  |\partial^2_{ww'}D_{j,m}|&\le\frac{2m(m-1)^2}{V(V-1)(V-2)}\,x^{m-1}
  \;\le\;\frac{9m^3}{V^3}\,x^{m-1}.
\end{align*}
The first bound follows from \eqref{eq:a-exchange}, with each conditional mean bounded in absolute value by
$x^{m-1}$. For the second, $\E^{-v}H_{j,m}$ is free of $\theta_v$, so the
derivative acts only on the unconditional average, restricted to windows
containing $v$: probability $m/V$.

For the third, differentiate
\eqref{eq:a-exchange}: the $w$-inclusion probability is $m/(V-1)$ given
$v\notin W$ and $(m-1)/(V-1)$ given $v\in W$, and their sum times $m/V$ is
$m(2m-1)/(V(V-1))$.

For the fourth, $v$ and $w$ must jointly occupy window
slots of the unconditional term: probability $(m)_2/(V)_2\le m^2/V^2$.

For the
fifth, the joint inclusion probabilities of $(w,w')$ under the two
conditionings are $m(m-1)/((V-1)(V-2))$ and $(m-1)(m-2)/((V-1)(V-2))$
(impossible events having probability zero), and their sum times $m/V$ is
$2m(m-1)^2/(V(V-1)(V-2))$.

The final numerical inequalities use
$V/(V-1)\le3/2$ and $V/(V-2)\le3$ for $V\ge4$.

\emph{Step 5 (assembly).} Differentiate \eqref{eq:a-window}; derivatives in
$\theta_v$ also act on the outer factor, with
$|\partial^q_{\theta_v}A^{\theta_v}(u)|\le x$ for $q\le2$ by Step 3. Bounding
every sum $\sum_{m\ge1}m^dx^{m-1}$ ($d\le3$) by
$\Lambda=\sum_{m\ge1}m^4x^{m-1}$ and the outer position average by $1$ gives the following bounds.

For $\partial_v$, the bound is
\[
  x\cdot\tfrac{2\Lambda}V+x\cdot\tfrac\Lambda V=\tfrac{3x\Lambda}V,
\]
where the two terms come from differentiating the outer factor and from $\partial_vD$.
For $\partial_w$, only the window term is differentiated, giving $x\cdot\tfrac{3\Lambda}{V^2}$.

For $\partial^2_{vv}$, the product rule gives the bound
\[
  x\cdot\tfrac{2\Lambda}V+2x\cdot\tfrac\Lambda V+x\cdot\tfrac\Lambda V=\tfrac{5x\Lambda}V.
\]
For $\partial^2_{vw}$, the bound is
\[
  x\cdot\tfrac{3\Lambda}{V^2}+x\cdot\tfrac{\Lambda}{V^2}=\tfrac{4x\Lambda}{V^2}.
\]

Finally, the bound for $\partial^2_{ww}$ is $x\cdot\tfrac{3\Lambda}{V^2}$, and
the bound for $\partial^2_{ww'}$ is $x\cdot\tfrac{9\Lambda}{V^3}$. These uniform bounds also justify
differentiation under the Gaussian expectation.
\end{proof}

\begin{proposition}
\label{prop:a-ms}
Let $\mathcal E(\bm\theta):=\bar e_A-\bar e_B=\sum_vc_ve_v$ with
$c_v=\one_{\{v\in\Aset\}}/V_A-\one_{\{v\in\Bset\}}/V_B$, and let
$r:=V_{\max}/V_{\min}$ with $V_{\max}=\max(V_A,V_B)$,
$V_{\min}=\min(V_A,V_B)$. Then, uniformly in $\bm\theta,\gamma,\sigma_b^2$, for
every $V\ge4$,
\begin{equation}
  |R_{\mathrm{ms}}|=|\mathcal E(\bm\theta)|
  \;\le\;\frac{|\beta|\Lambda}{V}
  \Big[\,|\delta|+\frac{25r}{8}\Big(s_A^2+s_B^2+\frac{\delta^2}{2}\Big)\Big],
  \label{eq:a-msbound}
\end{equation}
with $\delta=\bar\theta_A-\bar\theta_B$; in particular, under the equal-means
null, $|R_{\mathrm{ms}}|\le\tfrac{25r|\beta|\Lambda}{8V}(s_A^2+s_B^2)$. No
smallness of the spreads is required. Combining with \Cref{lem:a-ms-size}
($|\mathcal E|\le\sum_v|c_v||e_v|\le|\beta|\Lambda_1/V$ unconditionally) gives
the certified missing-self tolerance
\begin{equation}
  B_{\mathrm{ms}}(\delta):=\min\Big\{\frac{|\beta|\Lambda_1}{V},\;
  \frac{|\beta|\Lambda}{V}\Big[|\delta|+\frac{25r}{8}\Big(s_A^2+s_B^2+\frac{\delta^2}{2}\Big)\Big]\Big\}
  \;\ge\;|R_{\mathrm{ms}}| .
  \label{eq:a-Bms}
\end{equation}
\end{proposition}

\begin{proof}
\emph{Step 1 (flat point and gradient).} Set $c=\bar\mu$ and
$\bm d=\bm\theta-c\one$. At $\bm\theta=c\one$ all items are exchangeable, so
$e_v(c\one)$ does not depend on $v$ and $\mathcal E(c\one)=0$, since
$\sum_vc_v=0$.

By the same exchangeability $a:=\partial_ve_v(c\one)$ and
$b:=\partial_we_v(c\one)$ ($w\ne v$) do not depend on $(v,w)$, so
\[
  \partial_w\mathcal E(c\one)=c_wa+\big(\sum_{v\ne w}c_v\big)b=c_w(a-b)
\]
 and
\[
  \nabla\mathcal E(c\one)^\top\bm d=(a-b)\sum_wc_wd_w=(a-b)\,\delta,
  \qquad
  |a-b|\le\frac{3x\Lambda}V+\frac{3x\Lambda}{V^2}\le\frac{4x\Lambda}V
  \le\frac{|\beta|\Lambda}V,
\]
using \Cref{lem:a-ms-deriv}, $V\ge4$, and $4x\le|\beta|$.

\emph{Step 2 (uniform Hessian bound).} With $|c_v|\le1/V_{\min}$,
$\sum_v|c_v|=2$, and the second-order bounds of \Cref{lem:a-ms-deriv},
\[
  |\partial^2_{ww}\mathcal E|
  \le\frac1{V_{\min}}\cdot\frac{5x\Lambda}{V}+2\cdot\frac{3x\Lambda}{V^2},
  \qquad
  |\partial^2_{ww'}\mathcal E|
  \le\frac2{V_{\min}}\cdot\frac{4x\Lambda}{V^2}+2\cdot\frac{9x\Lambda}{V^3},
\]
separating the summands $v\in\{w,w'\}$ from the rest. The symmetric Hessian's
operator norm is at most its largest absolute row sum, so, uniformly over
$\bm\theta\in\R^V$,
\[
  \|\nabla^2\mathcal E(\bm\theta)\|_{\mathrm{op}}
  \;\le\;\frac{5x\Lambda}{VV_{\min}}+\frac{6x\Lambda}{V^2}
  +\frac{8x\Lambda}{VV_{\min}}+\frac{18x\Lambda}{V^2}
  \;\le\;\frac{13x\Lambda}{VV_{\min}}+\frac{24x\Lambda}{V^2}
  \;\le\;\frac{25x\Lambda}{VV_{\min}},
\]
where the last step uses $V_{\min}\le V/2$, which gives $1/V\le1/(2V_{\min})$ and hence
$24x\Lambda/V^2\le12x\Lambda/(VV_{\min})$.

\emph{Step 3 (Taylor along the whole segment).} Taylor's theorem with integral
remainder gives
\[
  \mathcal E(\bm\theta)=(a-b)\,\delta
  +\int_0^1(1-t)\,\bm d^\top\nabla^2\mathcal E(c\one+t\bm d)\,\bm d\,dt,
  \qquad
  |\mathcal E(\bm\theta)|\le\frac{|\beta|\Lambda}{V}\,|\delta|
  +\frac{25x\Lambda}{2VV_{\min}}\,\|\bm d\|_2^2 .
\]
Since
$\|\bm d\|_2^2=V_A(s_A^2+\delta^2/4)+V_B(s_B^2+\delta^2/4)
\le V_{\max}\big(s_A^2+s_B^2+\delta^2/2\big)$ and
$V_{\max}/V_{\min}=r$, the bound $25x/2\le25|\beta|/8$ (from $x\le|\beta|/4$)
gives \eqref{eq:a-msbound}. The Hessian bound holds along the entire segment,
so the argument requires neither a third-derivative bound nor a smallness condition on $\bm d$.
\end{proof}

\begin{remark}[Sharp scalings, conservative constants]
\label{rem:a-ms-numeric}
Exact computation (Gauss--Hermite over $u$, full enumeration of orderings) at
$\beta=\gamma=-0.5$, $\sigma_b^2=0.45$ gives
$|R_{\mathrm{ms}}|\approx0.012\cdot|\beta|(s_A^2+s_B^2)/V$ across $V\in\{6,8\}$ and
a fourfold spread change: the $1/V$ and spread-squared scalings of
\eqref{eq:a-msbound} are the correct laws, while the explicit constant
$25\Lambda/8$ overshoots by roughly three orders of magnitude. Note
$\sum_ve_v\ne0$ in general, so no exact
symmetrization is available. The companion fixed-design bound
$\varepsilon_W\le|\beta|/4$ of \Cref{app:williams} is the worst case over
deterministic prefix configurations; annealing over the shuffle is what improves it
to $O(|\beta|/V)$ here.
\end{remark}

\subsection{Proof of Theorem~\ref{thm:conv-bias} (attenuated signal plus curvature bias)}

By \Cref{lem:a-lln} and the continuous mapping theorem,
\[
  \Dhat_{\Naive}
  \;\xrightarrow{p}\;(\bar\pi_A-\bar\pi_B)+(\bar e_A-\bar e_B);
\]
the second bracket, $R_{\mathrm{ms}}:=\bar e_A-\bar e_B$, is the missing-self
contribution, bounded by \eqref{eq:a-msbound} of \Cref{prop:a-ms}; at
$\delta=0$ that bound is $\tfrac{25r|\beta|\Lambda}{8V}(s_A^2+s_B^2)$.

To handle the first bracket, we fix $v$ and expand $\pi$ about
$\bar\mu=\tfrac12(\bar\theta_A+\bar\theta_B)$ to order three with Lagrange
remainder:
\[
  \pi(\theta_v)=\pi(\bar\mu)+\pi'(\bar\mu)(\theta_v-\bar\mu)+\tfrac12\pi''(\bar\mu)(\theta_v-\bar\mu)^2
  +\tfrac16\pi'''(\bar\mu)(\theta_v-\bar\mu)^3+r_v,\quad |r_v|\le\tfrac1{24}M_4(\theta_v-\bar\mu)^4.
\]
Average over $v\in G$ and write $\bar r_G:=V_G^{-1}\sum_{v\in G}r_v$. With
$a_G:=\bar\theta_G-\bar\mu$, so that $a_A=\delta/2$ and $a_B=-\delta/2$, the
parallel-axis identities
\begin{align*}
  \frac1{V_G}\sum_{v\in G}(\theta_v-\bar\mu)^2 &= s_G^2+a_G^2,\\
  \frac1{V_G}\sum_{v\in G}(\theta_v-\bar\mu)^3 &= m_{3,G}+3s_G^2\,a_G+a_G^3
\end{align*}
give
\[
  \bar\pi_G=\pi(\bar\mu)+\pi'(\bar\mu)\,a_G+\tfrac12\pi''(\bar\mu)\big(s_G^2+a_G^2\big)
  +\tfrac16\pi'''(\bar\mu)\big(m_{3,G}+3s_G^2a_G+a_G^3\big)+\bar r_G .
\]
Subtracting the two group averages cancels the $\pi(\bar\mu)$ and $a_G^2=\delta^2/4$ terms, giving
\[
  \bar\pi_A-\bar\pi_B=\pi'(\bar\mu)\,\delta+\tfrac12\pi''(\bar\mu)(s_A^2-s_B^2)
  +\tfrac16\pi'''(\bar\mu)(m_{3,A}-m_{3,B})+I+(\bar r_A-\bar r_B),
\]
where the third-order gap--spread interaction is
\[
  I\;:=\;\tfrac14\pi'''(\bar\mu)\,\delta\,\big(s_A^2+s_B^2\big)
  \;+\;\tfrac1{24}\pi'''(\bar\mu)\,\delta^3,
  \qquad
  |I|\;\le\;\tfrac23M_3\,|\delta|\,\bar\rho^{\,2},
\]
using $s_G^2\le\bar\rho^{\,2}$, $\delta^2\le4\bar\rho^{\,2}$, and
$M_3:=\sup_x|\pi'''(x)|$; and
$|\bar r_A-\bar r_B|\le\tfrac1{24}M_4(\rho_A^4+\rho_B^4)$, where
$\rho_G:=\max_{v\in G}|\theta_v-\bar\mu|$ and $M_4:=\sup_x|\pi^{(4)}(x)|\le0.13$.

For the slope,
$\pi'(\bar\mu)=\E_C[\sigma(\bar\mu+C)(1-\sigma(\bar\mu+C))]\le\tfrac14$ since
$p(1-p)\le\tfrac14$, with equality only if $\sigma(\bar\mu+C)=\tfrac12$ a.s.;
positivity follows from $\sigma\in(0,1)$.

Combining these terms, the remainder of \eqref{eq:conv-bias} is $R+I+R_{\mathrm{ms}}$ with
$R:=\bar r_A-\bar r_B$, and, for the balanced split $V_A=V_B$ (unbalanced
splits multiply the quadratic part by $r=V_{\max}/V_{\min}$, \Cref{prop:a-ms}),
\[
  |R+I+R_{\mathrm{ms}}|
  \;\le\;\tfrac1{24}M_4\big(\rho_A^4+\rho_B^4\big)
  \;+\;\tfrac23M_3\,|\delta|\,\bar\rho^{\,2}
  \;+\;\tfrac{|\beta|\Lambda}{V}\Big[\,|\delta|
  +\tfrac{25}8\big(s_A^2+s_B^2+\delta^2/2\big)\Big];
\]
since $\rho_G\le\bar\rho$, $s_G^2\le\bar\rho^{\,2}$, and $|\delta|\le2\bar\rho$,
this is the
$O\big(\bar\rho^{\,4}+|\delta|\bar\rho^{\,2}+\tfrac{|\beta|}V(|\delta|+\bar\rho^{\,2})\big)$
asserted, with constants depending only on $\beta$ (through $\Lambda$) and the
split ratio.

At $\delta=0$ (so $\bar\mu=\mu$), $I=0$, the $\delta$-parts of the
$R_{\mathrm{ms}}$ bound vanish, and the display and remainder reduce to the null
form: curvature term, third-moment term, and
$O(\bar\rho^{\,4}+|\beta|\bar\rho^{\,2}/V)$. Matched third moments
remove the displayed $\pi'''$ term; for $\delta\ne0$ the interaction $I$ carries
a residual $\pi'''$ contribution inside the $|\delta|\bar\rho^{\,2}$ part of the
remainder. \qed

\begin{remark}[When is the curvature positive?]
\label{rem:a-signpos}
By \Cref{lem:a-pi-deriv}, $\pi''(\mu)=\E_C[\sigma''(\mu+C)]$, with $\sigma''(x)>0$
for $x<0$, $\sigma''(x)<0$ for $x>0$, and $\sigma''(-y)=-\sigma''(y)$. Two simple
sufficient conditions for $\pi''(\mu)>0$: \textnormal{(a)} $\mu+C\le0$ almost
surely with positive probability of strict inequality (splitting the expectation
on $\{\mu+C<0\}$; attainable only for bounded-support contextual laws, hence only
at $\sigma_b^2=0$ under \Cref{ass:working}); \textnormal{(b)} the density $f$ of
$\mu+C$ is left-lopsided, $f(y)\le f(-y)$ for all $y>0$ with strict inequality on
a positive-measure set, since then, by the odd symmetry of $\sigma''$,
\[
  \pi''(\mu)\;=\;\int_0^\infty\sigma''(y)\,\big[f(y)-f(-y)\big]\,dy\;>\;0 .
\]
In practice neither condition is needed: $\pi''(\mu)$ is computed directly at the
fitted parameters.
\end{remark}

\begin{corollary}[Catastrophic Type-I inflation]
\label{cor:a-typeI}
Under \Cref{ass:working,ass:design} and $H_0$ ($\delta=0$), with matched third
moments, $\pi''(\mu)>0$, and
$\tfrac12\pi''(\mu)(s_A^2-s_B^2)>|R|+|R_{\mathrm{ms}}|$ with $s_A^2>s_B^2$, the
in-probability
limit of $\Dhat_{\Naive}$ is strictly positive though $H_0$ holds.
Consequently any test rejecting when $\Dhat_{\Naive}$ exceeds a fixed threshold
below that limit has rejection probability tending to $1$ as $N\to\infty$.
\end{corollary}

\begin{remark}[Interaction of the two biases]
\label{rem:a-interaction}
The curvature term $\tfrac12\pi''(\bar\mu)(s_A^2-s_B^2)$ appears under both hypotheses:
under $H_0$ it is the entire (spurious) signal, while under $H_1$ it is an additive
contaminant on top of the attenuated $\pi'(\bar\mu)\delta$. Depending on the signs of
$\pi''(\bar\mu)$, $s_A^2-s_B^2$ and $\delta$ it reinforces or partially cancels the true
signal---a defect no threshold can repair.
\end{remark}

\subsection{Finite-sample concentration}
\label{app:concentration}

Recall the per-batch contrast decomposition \eqref{eq:batch-contrast}: under
the complete-block shuffle each batch contributes exactly one rating per item,
so $\Dhat_{\Naive}=N^{-1}\sum_{t=1}^Nc_t$ is an average of i.i.d.\ per-batch
contrasts $c_t\in[-1,1]$, and
$\E[c_t]=(\bar\pi_A-\bar\pi_B)+(\bar e_A-\bar e_B)=:\Delta_\infty$ \emph{exactly},
for every $t$ and $N$: the centering needs no asymptotics, and no within-batch
independence is assumed---the dependence created by the shared $u_t$ and the lag
chain lives entirely inside $c_t$. Write $B_{\mathrm{ms}}(\delta)$ for the
certified missing-self tolerance \eqref{eq:a-Bms} of \Cref{prop:a-ms}, and
$\varrho:=R^*+B_{\mathrm{ms}}(0)$,
$\varrho':=R^{**}+B_{\mathrm{ms}}(\delta)$ with $R^*:=M_4(\rho_A^4+\rho_B^4)/24$,
$R^{**}:=M_3(\rho_A^3+\rho_B^3)/6$; here $R^*$ bounds the third-order Taylor
remainder of the proof of \Cref{thm:conv-bias} at $\delta=0$, while $R^{**}$
bounds, for any $\delta$, everything beyond the first two displayed terms of
\eqref{eq:conv-bias} (a second-order Lagrange expansion about $\bar\mu$; the
group averages of $|\theta_v-\bar\mu|^3$ are at most $\rho_G^3$).

\begin{lemma}
\label{lem:a-conc}
Under \Cref{ass:working} and the random-shuffle clause (a) of \Cref{ass:design},
for every $\epsilon>0$,
$\Pr[|\Dhat_{\Naive}-\Delta_\infty|>\epsilon]\le2\exp(-N\epsilon^2/2)$, and
$\Var(\Dhat_{\Naive})=\nu_{\mathrm b}/N$ with $\nu_{\mathrm b}:=\Var(c_1)\le1$.
Under the block clause (b) the same argument at the block level gives
$\Pr[|\Dhat_{\Naive}-\Delta_\infty|>\epsilon]\le2\exp(-(N/B)\epsilon^2/2)$.
\end{lemma}

\begin{proof}
Under clause (a) the $c_t$ are i.i.d.\ with range contained in $[-1,1]$ (each is a
difference of two
averages of indicators), so Hoeffding's inequality for bounded i.i.d.\ variables
applies; the variance identity follows from independence across batches.

Under clause (b) the block-averaged contrasts are $N/B$ i.i.d.\ variables with
range in $[-1,1]$ and the same mean, and Hoeffding applies to them.
\end{proof}

\begin{remark}[Stateless evaluation: sharper constants, comparison only]
\label{rem:a-stateless}
If every evaluation were instead an independent stateless call---$C_{v,k}\stackrel{\text{i.i.d.}}{\sim}F_C$,
$Y_{v,k}\mid C_{v,k}\sim\mathrm{Bernoulli}(\sigma(\theta_v+C_{v,k}))$, the family
mutually independent---then $e_v\equiv0$ and the weighted-Hoeffding computation
with weights $a_{v,k}=\pm1/(V_{G(v)}N)$,
$\sum_{v,k}a_{v,k}^2=N^{-1}(V_A^{-1}+V_B^{-1})=1/\Neff$ with
$\Neff:=V_AV_BN/(V_A+V_B)$, sharpens the tail of
\Cref{lem:a-conc} to $2\exp(-2\Neff\epsilon^2)$ and gives the exact variance
$\Var(\Dhat_{\Naive})=\bar\nu/\Neff$ with
$\bar\nu=(V_B\bar\nu_A+V_A\bar\nu_B)/(V_A+V_B)$,
$\bar\nu_G=V_G^{-1}\sum_{v\in G}\pi(\theta_v)(1-\pi(\theta_v))$. The batched bound
loses a factor of order $V$ in the exponent---the price of context sharing---but no
result in this paper assumes stateless evaluation; this remark only calibrates the
constants.
\end{remark}

\begin{proposition}
\label{prop:a-type1conc}
Under \Cref{ass:working}, the random-shuffle clause (a) of \Cref{ass:design}, and
$H_0$, for every $\epsilon>0$,
\[
  \Pr\!\Big[\big|\Dhat_{\Naive}-\tfrac12\pi''(\mu)(s_A^2-s_B^2)-\tfrac16\pi'''(\mu)(m_{3,A}-m_{3,B})\big|>\epsilon+\varrho\Big]\le2\exp(-N\epsilon^2/2).
\]
\end{proposition}

\begin{proof}
Write the bracketed quantity as $S+R_0$, where $S:=\Dhat_{\Naive}-\Delta_\infty$ is
stochastic and
\[
  R_0:=\Delta_\infty-\tfrac12\pi''(\mu)(s_A^2-s_B^2)-\tfrac16\pi'''(\mu)(m_{3,A}-m_{3,B})
\]
is deterministic, with $|R_0|\le\varrho$ by the $\delta=0$ case of the proof of
\Cref{thm:conv-bias} together with \Cref{prop:a-ms}.

Then $\{|S+R_0|>\epsilon+\varrho\}\subseteq\{|S|>\epsilon\}$, and
\Cref{lem:a-conc} applies.
\end{proof}

\begin{proposition}
\label{prop:a-type2conc}
Under \Cref{ass:working}, the random-shuffle clause (a) of \Cref{ass:design}, and
$H_1:\bar\theta_A-\bar\theta_B=\delta$, for
every $\epsilon>0$,
\[
  \Pr\!\Big[\big|\Dhat_{\Naive}-\pi'(\bar\mu)\delta-\tfrac12\pi''(\bar\mu)(s_A^2-s_B^2)\big|>\epsilon+\varrho'\Big]\le2\exp(-N\epsilon^2/2);
\]
in particular, when $s_A^2=s_B^2$, the centering is $\pi'(\bar\mu)\delta$.
\end{proposition}

\begin{proof}
We use the decomposition in the proof of \Cref{prop:a-type1conc}. The second-order Lagrange
bound $R^{**}$ above and \Cref{prop:a-ms} bound the deterministic part by
$\varrho'$, while \Cref{lem:a-conc} controls $S$.
\end{proof}

\begin{remark}[Bernstein sharpening]
\label{rem:a-bernstein}
Bernstein's inequality applied to the i.i.d.\ batch contrasts replaces the exponent
$N\epsilon^2/2$ by $N\epsilon^2/(2\nu_{\mathrm b}+4\epsilon/3)$, improving the
bound whenever $\nu_{\mathrm b}+2\epsilon/3<1$; the stateless analogue uses
$\bar\nu/\Neff$.
\end{remark}

\medskip\noindent\textbf{Proof of Corollary~\ref{cor:conv-student}.}
Under clause (a) the batch contrasts $c_t$ are i.i.d., bounded by $1$ in absolute
value, with mean $\Delta_\infty\ne0$ and variance $\nu_{\mathrm b}>0$. By the CLT,
\[
  \sqrt N(\Dhat_{\Naive}-\Delta_\infty)\xrightarrow{d}\N(0,\nu_{\mathrm b}),
\]
 so
\[
  \sqrt N\,\Dhat_{\Naive}=\sqrt N\,\Delta_\infty+O_p(1)\to\pm\infty
\]
 in probability
(with the sign of $\Delta_\infty$).

The empirical variance of the $c_t$ converges
in probability to $\nu_{\mathrm b}$ by the weak law for i.i.d.\ bounded variables,
and for any $\hat\nu\xrightarrow{p}\nu_{\mathrm b}>0$ Slutsky's theorem gives
\[
  |T_N|=|\sqrt N\,\Dhat_{\Naive}|/\hat\nu^{1/2}\xrightarrow{p}\infty;
\]
 hence
$\Pr(|T_N|>z_{1-\alpha/2})\to1$ for every fixed $\alpha\in(0,1)$. \qed

\begin{proposition}[Calibrated local efficiency comparison in the stateless submodel]
\label{prop:a-are}
Consider the independent stateless submodel of \Cref{rem:a-stateless}
($C_{v,k}\stackrel{\text{i.i.d.}}{\sim}F_C$, $Y_{v,k}\mid C_{v,k}\sim
\mathrm{Bernoulli}(\sigma(\theta_v+C_{v,k}))$, the family mutually independent,
the context law held fixed). Fix a baseline panel $\bm\theta^0$ with equal group
means, and consider the local path $\theta_{v,N}=\theta^0_v+h/(2\sqrt N)$ on
$\Aset$, $\theta_{v,N}=\theta^0_v-h/(2\sqrt N)$ on $\Bset$, for fixed $h>0$.
Write $w_v:=\pi(\theta^0_v)(1-\pi(\theta^0_v))$, $\pi'_v:=\pi'(\theta^0_v)$,
$b_0:=V_A^{-1}\sum_{\Aset}\pi(\theta^0_v)-V_B^{-1}\sum_{\Bset}\pi(\theta^0_v)$,
and
\[
  \ell:=\tfrac12(\bar\pi'_A+\bar\pi'_B),\quad
  S:=V_A^{-2}\sum_{\Aset}w_v+V_B^{-2}\sum_{\Bset}w_v,\quad
  T:=V_A^{-2}\sum_{\Aset}\mathcal I_v^{-1}+V_B^{-2}\sum_{\Bset}\mathcal I_v^{-1},
\]
where $\mathcal I_v>0$ is the asymptotic Fisher information per evaluation about
$\theta_v$ achieved by the model-based fit, \emph{assumed} diagonal across
items---as holds when the context law is known; estimating common context-law
parameters would introduce cross-item information and require the full
covariance matrix in place of $T$. Assume the fitted item estimates are
regular and jointly asymptotically normal about the local parameters with
$N\Var(\Dhat_{\GLMM})\to T$. Then, with the naive statistic
\emph{oracle-centered} at $b_0$,
\[
  \frac{\sqrt N\,(\Dhat_{\Naive}-b_0)}{\sqrt S}
  \xrightarrow{d}\N\Big(\frac{h\ell}{\sqrt S},\,1\Big),
  \qquad
  \frac{\sqrt N\,\Dhat_{\GLMM}}{\sqrt T}
  \xrightarrow{d}\N\Big(\frac{h}{\sqrt T},\,1\Big),
\]
and the asymptotic relative efficiency, defined as the ratio of squared
limiting noncentrality parameters (model relative to naive), is
\begin{equation}
  \mathrm{ARE}
  =\frac{S}{\ell^2\,T}
  =\frac{V_A^{-1}\bar w_A+V_B^{-1}\bar w_B}
  {\big(\tfrac12\bar\pi'_A+\tfrac12\bar\pi'_B\big)^2
  \big[V_A^{-1}\overline{(1/\mathcal I)}_A+V_B^{-1}\overline{(1/\mathcal I)}_B\big]},
  \label{eq:a-are}
\end{equation}
with $\bar w_G,\bar\pi'_G,\overline{(1/\mathcal I)}_G$ within-group means. This
is a comparison at the specified baseline and path, not a validity claim for
the naive statistic as a test of the composite equal-group-means null: $b_0$
may be nonzero at baselines with unequal spreads (\Cref{thm:conv-bias}), so the
centering is essential, and estimating $b_0$ would add variance and covariance
contributions that the display omits. With
equal group sizes and pooled means
$\E^*[\cdot]:=\tfrac12(\textnormal{mean}_{\Aset}+\textnormal{mean}_{\Bset})$:
\textnormal{(i)} in the \emph{context-observed} regime
($\mathcal I_v=\pi'_v$: the fit conditions on the realized context, a distinct
observation experiment from the marginal one below),
\begin{equation}
  \mathrm{ARE}
  =\underbrace{\frac{\E^*[w]}{\E^*[\pi']}}_{=:\,\bar\kappa\,\ge\,1}
  \times\underbrace{\frac{(\E^*[1/\pi'])^{-1}}{\E^*[\pi']}}_{\le\,1},
  \label{eq:a-are-decomp}
\end{equation}
a \emph{context gain} ($w_v\ge\pi'_v$ pointwise by Jensen on the concave
$x(1-x)$, strict iff $\Var(\sigma(\theta_v+C))>0$) times a \emph{heterogeneity
tax} (harmonic-to-arithmetic mean ratio of $\{\pi'_v\}$, $=1$ iff $\pi'_v$
constant); \textnormal{(ii)} in the \emph{context-marginal} regime
($\mathcal I_v=(\pi'_v)^2/w_v$: only each item's marginal Bernoulli law is used),
$\mathrm{ARE}=\E^*[w]/(\E^*[\pi']^2\E^*[w/(\pi')^2])$, which tends to $1$ under
slow variation of $\pi'$ (the marginal MLE is a smooth transform of $\phat_v$) and
reduces to $1/(\E^*[w]\E^*[1/w])\le1$ when the context is degenerate
($\pi'_v=w_v$)---a pure variance tax growing with item heterogeneity.
\end{proposition}

\begin{proof}
Along the path, $\E[\Dhat_{\Naive}]=b_0+\ell\,h/\sqrt N+o(N^{-1/2})$ by a
first-order expansion of $\pi$ at $\bm\theta^0$, and
$N\Var(\Dhat_{\Naive})\to S$ (exactly $S$ at each $N$, evaluated along the
path, by independence); the CLT for the oracle-centered statistic follows.

The model-based contrast estimates the local gap $h/\sqrt N$ directly with
$N\Var(\Dhat_{\GLMM})\to T$ by hypothesis. The ratio of squared noncentrality
parameters is 
\[
  (h/\sqrt T)^2/(h\ell/\sqrt S)^2=S/(\ell^2T),
\]
 which is
\eqref{eq:a-are}.

For (i), substitute $\mathcal I_v=\pi'_v$ and factor
\[
  \frac{\E^*[w]}{\E^*[\pi']^2\,\E^*[1/\pi']}
  \;=\;\frac{\E^*[w]}{\E^*[\pi']}\cdot\frac1{\E^*[\pi']\,\E^*[1/\pi']}\,;
\]
the pointwise inequality $w_v=f(\E_C[\sigma])\ge\E_C[f(\sigma)]=\pi'_v$ follows from Jensen's inequality
for $f(x)=x(1-x)$, and $\E^*[\pi']\,\E^*[1/\pi']\ge1$ is the arithmetic--harmonic
inequality.

For (ii), substituting $\mathcal I_v=(\pi'_v)^2/w_v$ gives the stated expression and its two limiting cases.
Finally, $\E^*[w]\,\E^*[1/w]\ge1$ follows from the Cauchy--Schwarz inequality.
\end{proof}

\begin{remark}[Which regime is real]
\label{rem:a-are-interp}
\Cref{prop:a-are} is exact only in the stateless submodel: under the batched
Markov design both estimators acquire within-batch covariance and
nuisance-parameter ($\beta,\gamma,\sigma_b^2$) corrections that the display
ignores, and we use the factorization only as an interpretable calibration---the
batched case is probed directly by the simulations of \Cref{sec:conv-illusion}.
With that scope understood: the ``efficiency factor'' $\kappa(\bar\mu)=\bar p(1-\bar p)/\pi'(\bar\mu)\ge1$ is regime (i) under slow variation---all context gain, no tax; the
opposite-signed ``variance tax'' $1/\sqrt{\bar\nu\bar\omega}\le1$ (with
$\bar\omega:=\E^*[1/w]$) is regime (ii) with degenerate context---all tax, no gain,
and setting $\pi'_v=w_v$ contradicts the very contextual noise this paper models.
The batched GLMM sits between the regimes: the lag and cold-start covariates
are observed exactly, while the batch effect is a genuinely latent component;
in our fits at $V\ge30$ the estimated per-item information tracks the
regime-(i) value, and we use regime (i) as a heuristic calibration rather than
a limit claim. The heterogeneity tax in \eqref{eq:a-are-decomp} is real and
grows with the item-panel spread. Neither test dominates in general: in
matched-spread simulations the factors nearly cancel, while on saturated panels the
context gain dominates---the pattern of \Cref{sec:data-example}.
\end{remark}

\bigskip

\end{document}